\documentclass{article}
\usepackage[a4paper]{geometry}
\usepackage{natbib}
\usepackage[utf8]{inputenc}
\usepackage[T1]{fontenc}
\usepackage{times,pifont}

\usepackage{hyperref}
\usepackage{url}
\hypersetup{
    colorlinks,
    linkcolor={red!80!white},
    citecolor={green!60!black},
    urlcolor={black}
}

\usepackage{booktabs}
\usepackage{microtype}

\usepackage{tabularx}

\usepackage{empheq}
\usepackage[usenames,dvipsnames]{xcolor}

\usepackage{graphicx}
\graphicspath{{./figures/}}
\usepackage{subcaption}

\usepackage{amsmath,amsthm,amsfonts,nicefrac,amssymb,bm}

\usepackage[capitalize]{cleveref}

\newtheorem{theorem}{Theorem}
\newtheorem{lemma}[theorem]{Lemma}

\newtheorem{proposition}[theorem]{Proposition}

\newtheorem{remark}{Remark}[theorem]

\DeclareMathOperator*{\diag}{\mathrm{diag}}

\newcommand{\fim}{\mathcal{I}}
\newcommand{\efima}{\hat{\mathcal{I}}_{1}}
\newcommand{\efimb}{\hat{\mathcal{I}}_{2}}
\newcommand{\efimz}{\hat{\mathcal{I}}_{z}}
\newcommand{\efimc}{\hat{\mathcal{I}}_{\alpha}}

\newcommand{\prob}{\mathrm{Pr}}
\newcommand{\var}{\mathrm{Var}}
\newcommand{\cov}{\mathrm{Cov}}
\newcommand{\kurt}{\mathcal{K}}
\newcommand{\trace}{\mathrm{tr}}
\newcommand{\relu}{\mathrm{ReLU}}
\newcommand{\vectorise}{\mathrm{vec}}
\newcommand{\grad}{\bigtriangledown}
\newcommand{\defeq}{\vcentcolon=}
\newcommand{\dx}{\, \mathrm{d}\bm{x}}
\newcommand{\dy}{\, \mathrm{d}\bm{y}}
\newcommand{\dmeas}[1]{\, \mathrm{d}#1}

\newcommand{\residual}{\mathcal{R}}

\newcommand{\expect}{\mathbb{E}}

\newcommand{\indicator}[1]{{\bm 1}_{#1}}

\newcommand{\T}{\top}

\usepackage{xspace}
\makeatletter
\DeclareRobustCommand\onedot{\futurelet\@let@token\bmv@onedotaux}
\def\bmv@onedotaux{\ifx\@let@token.\else.\null\fi\xspace}
\def\eg{\emph{e.g}\onedot} 
\def\ie{\emph{i.e}\onedot} 
 
\def\etc{\emph{etc}\onedot} 
\def\wrt{w.r.t\onedot} 
\def\wlogt{w.l.o.g\onedot}

\def\iid{i.i.d\onedot}
\def\psd{p.s.d\onedot}

\def\bigoh{\mathcal{O}}
\makeatother

\newcolumntype{B}{>{\hsize=1.6\hsize}X}
\newcolumntype{M}{>{\hsize=1.1\hsize}X}
\newcolumntype{S}{>{\hsize=0.6\hsize}X}
\newcolumntype{L}{>{\hsize=0.3\hsize}X}
\newcolumntype{D}{>{\hsize=0.15\hsize}X}
\newcolumntype{N}{>{\hsize=0.85\hsize}X}

\usepackage[ruled,algo2e]{algorithm2e}

\usepackage{soul}
\usepackage{float}

\usepackage{orcidlink}
\usepackage[tight]{minitoc}

\usepackage[toc,page,header]{appendix}

\title{On the Variance of the Fisher Information\\for Deep Learning\thanks{This article is a digital reprint of \cite{varfim} with a different format and an appendix.}}

\hypersetup{colorlinks=false} \author{Alexander Soen~\orcidlink{0000-0002-2440-4814}\\
The Australian National University\\
Canberra, Australia\\
\texttt{alexander.soen@anu.edu.au}
\and
Ke Sun~\orcidlink{0000-0001-6263-7355}\\
CSIRO's Data61, Sydney, Australia\\
The Australian National University\\
\texttt{sunk@ieee.org}
}
\hypersetup{colorlinks=true, citecolor=MidnightBlue, urlcolor=blue}

\makeatletter
\begin{document}
\doparttoc
\faketableofcontents

\maketitle

\begin{abstract}
In the realm of deep learning, the Fisher information matrix (FIM) gives novel
insights and useful tools to characterize the loss landscape, perform
second-order optimization, and build geometric learning theories.
The exact FIM is either unavailable in closed form or too expensive to compute.
In practice, it is almost always estimated based on empirical samples.  We
investigate two such estimators based on two equivalent representations of the
FIM --- both unbiased and consistent. Their estimation quality is naturally
gauged by their variance given in closed form. We analyze how the parametric
structure of a deep neural network can affect the variance. The meaning of this
variance measure and its upper bounds are then discussed in the context of deep learning.
\end{abstract}

\section{Introduction}\label{sec:intro}

The Fisher information is one of the most fundamental concepts in
statistical machine learning.
Intuitively, it measures the amount of information carried by a single
random observation when the underlying model varies along certain directions
in the parameter space:
if such a variation does not change the underlying model,
then a corresponding observation contains zero (Fisher) information and is non-informative
regarding the varied parameter. Parameter estimation is impossible in this case.
Otherwise, if the variation significantly changes the model and has large information,
then an observation is informative and the parameter estimation
can be more efficient as compared to parameters with small Fisher information.
In machine learning, this basic concept is useful for defining intrinsic structures
of the parameter space, measuring model complexity, and performing gradient-based optimization.

Given a statistical model that is specified by a parametric form
$p(\bm{z}\,\bm\vert\,\bm\theta)$ and a continuous domain $\bm\theta\in\mathcal{M}$,
the Fisher information matrix (FIM) is a 2D tensor varying with $\bm\theta\in\mathcal{M}$,
given by
\begin{equation}
    \label{eq:fim}
\fim(\bm\theta) =
\expect_{p(\bm{z}\,\vert\,\bm\theta)}
\left( \frac{\partial\ell}{\partial\bm\theta} \frac{\partial\ell}{\partial\bm\theta^\T} \right),
\end{equation}
where $\expect_{p(\bm{z}\,\vert\,\bm\theta)}(\cdot)$, or simply $\expect_p(\cdot)$ if the model $p$ is clear from the context,
denotes the expectation \wrt $p(\bm{z} \,\vert\,\bm\theta)$,
and $\ell\defeq\log{p}(\bm{z}\,\vert\,\bm\theta)$ is the log-likelihood function.
All vectors are column vectors throughout this paper.
Under weak conditions (see Lemma 5.3 in \citet{thpe} for the univariate case),
the FIM has the equivalent expression
$\fim(\bm\theta) = \expect_{p(\bm{z}\,\vert\,\bm\theta)} \left( -{\partial^2\ell}/{\partial\bm\theta\partial\bm\theta^\T} \right)$.
Given $N$ \iid observations $\bm{z}_1,\ldots,\bm{z}_N$, these two equivalent expressions of the FIM
lead to two different estimators
\begin{equation}\label{eq:estimators}
\efima(\bm\theta) = \frac{1}{N}\sum_{i=1}^N
\left( \frac{\partial\ell_i}{\partial\bm\theta} \frac{\partial\ell_i}{\partial\bm\theta^\T} \right)
\quad \text{and} \quad
\efimb(\bm\theta) = \frac{1}{N}\sum_{i=1}^N
\left( -\frac{\partial^2\ell_i}{\partial\bm\theta\partial\bm\theta^\T} \right),
\end{equation}
where $\ell_i\defeq\log p(\bm{z}_i\,\vert\,\bm\theta)$ is the log-likelihood of the $i$'th observation $\bm{z}_i$.
The notations $\efima(\bm\theta)$ and $\efimb(\bm\theta)$ are abused for
simplicity as they depend on both $\bm\theta$ and the random observations
$\bm{z}_i$.

These estimators are universal and independent to the
parametric form $p(\bm{z}\,\vert\,\bm\theta)$. They are expressed in terms
of the 1st- or 2nd-order derivatives of the log-likelihood. Usually, we already
have these derivatives to perform gradient-based learning. Therefore, we can
save computational cost and reuse these derivatives to estimate the Fisher information,
which in turn can be useful, \eg, to perform natural gradient optimization~\citep{aIGI,pbRNG}.
Estimating the FIM is especially meaningful for deep learning, where
the computational overhead of the exact FIM can be significant.

It is straightforward from the law of large numbers and the central limit theorem that
both estimators in \cref{eq:estimators} are unbiased and consistent.
This is formally stated as follows.
\begin{proposition} \label{prop:unbiased_consistent}
\begin{align*}
&\expect_{p(\bm{z}\,\vert\,\bm\theta)}\left( \efima(\bm\theta) \right)
= \expect_{p(\bm{z}\,\vert\,\bm\theta)}\left( \efimb(\bm\theta) \right)
= \fim(\bm\theta).\\
\forall{\epsilon}>0,\;
&\lim_{N\to\infty}
\mathrm{Prob}\left(
\left\Vert \efima(\bm\theta) - \fim(\bm\theta) \right\Vert_F
+
\left\Vert \efimb(\bm\theta) - \fim(\bm\theta) \right\Vert_F
> \epsilon \right)
= 0,
\end{align*}
where $\mathrm{Prob}(\cdot)$ denotes the probability of the parameter statement being true
and
$\Vert\cdot\Vert_{F}$ is the Frobenius norm of a tensor
    (with \( \Vert \cdot \Vert_{2} \) as the regular vector \( L_{2} \)-norm)
\end{proposition}
The Fisher information can be zero for non-regular models or
infinite~\citep{InfiniteFIMMixture-2009}. However, these properties may not be
preserved by the empirical estimators.

How far can $\efima(\bm\theta)$ and $\efimb(\bm\theta)$ deviate from the ``true
FIM'' $\fim(\bm\theta)$, and how fast can they converge to $\fim(\bm\theta)$ as
the number of observations increases? To answer these questions, it is natural
to think of the variance of $\efima(\bm\theta)$ and $\efimb(\bm\theta)$. For
example, an estimator with a large variance means the estimation does not
accurately reflect $\fim(\bm\theta)$; and any procedure depending on the FIM consequently suffers from
the estimation error. Through studying the variance, we can control the estimation
quality
and reliably perform subsequent measurements or algorithms based on the FIM.

Towards this direction,
we made the following contributions that will unfold in the following
\cref{sec:ffexp,sec:estimators,sec:dl_structure}:
\begin{itemize}
\item We review and rediscover two equivalent expression of the FIM in the context
of deep feed-forward networks (\cref{sec:ffexp});
\item We give in closed form the variance (extending to meaningful upper bounds)
and discuss the convergence rate of the estimators
$\efima(\bm\theta)$ and $\efimb(\bm\theta)$ (\cref{sec:estimators});
\item We analyze how the 1st- and 2nd-order derivatives of the neural network
can affect the estimation of the FIM (\cref{sec:dl_structure}).
\end{itemize}
We discuss related work in \cref{sec:related} and conclude in \cref{sec:con}.

\section{Feed-forward Networks with Exponential Family Output}\label{sec:ffexp}

This section realizes the concept of Fisher information in a
feed-forward network with exponentially family output and explains
why its estimators are useful in theory and practice.

Consider supervised learning with a neural network.
The underlying statistical model is
$p(\bm{z}\,\vert\,\bm\theta) = p(\bm{x})p(\bm{y}\,\vert\,\bm{x}, \bm\theta)$,
where
\( \bm z = (\bm x, \bm y) \),
the random variable \( \bm x \) represents features,
and \( \bm y \) is the target variable.
The marginal distribution
$p(\bm{x})$ is parameter-free, usually fixed as the empirical distribution $p(\bm{x})=\frac{1}{M}\sum_{i=1}^M \delta(\bm{x}-\bm{x}_i)$ \wrt a set of observations $\{\bm{x}_i\}_{i=1}^M$, where $\delta(\cdot)$ is the Dirac delta.
In this paper, we consider \wlogt \( M = 1 \) as the FIM \wrt observations \( \{\bm{x}_i\}_{i=1}^M \) is simply the average over FIMs of each individual observation. All results generalize to multiple observations by taking the empirical average.
The predictor $p(\bm{y}\,\vert\,\bm{x}, \bm\theta)$
is a neural network with parameters $\bm\theta = \{\bm W_{l-1}\}_{l=1}^{L}$ and exponential family output units,
given by
\begin{align}\label{eq:exp}
p(\bm{y}\,\vert\,\bm{x})
&= \exp\left( \bm{t}^\T(\bm{y})\bm{h}_L - F(\bm{h}_L) \right),\nonumber\\
\bm{h}_L
&= \bm{W}_{L-1}\bar{\bm{h}}_{L-1},\nonumber\\
\bm{h}_l
&= \sigma( \bm{W}_{l-1}\bar{\bm{h}}_{l-1}), \quad{}(l=1,\ldots,L-1)\nonumber\\
\bar{\bm{h}}_l &= (\bm{h}_l^\T,1)^\T,\nonumber\\
\bm{h}_0 &= \bm{x},
\end{align}
where $\bm{t}(\bm{y})$ is the sufficient statistics of the prediction model,
$F(\bm{h})=\log\int\exp(\bm{t}^\T(\bm{y})\bm{h})d\bm{y}$ is the log-partition function,
and $\sigma:\;\Re\to\Re$ is an element-wise non-linear activation function.
Moreover, $\bm{W}_l$ is a $n_{l+1}\times{}(n_l+1)$ matrix, representing the
neural network parameters (weights and biases) in the $l$'th layer,
where \( n_{l} \defeq \dim(\bm h_{l}) \) denotes the size of layer \( l \).
We use $\bm{W}_l^{-}$ for the $n_{l+1}\times{n}_l$ weight matrix without the
bias terms, obtained by removing the last column of $\bm{W}_l$.
$\bm{h}_l$ is a learned representation of the input $\bm{x}$. All
intermediate variables $\bm{h}_l$ are extended to include a constant scalar
1 in $\bar{\bm{h}}_l$, so that a linear layer can simply be expressed as
$\bm{W}_{l}\bar{\bm{h}}_{l}$.
The last layer's output \( \bm h_{L} \) with dimensionality $n_L$
specifies the natural parameter of the exponential family.

We need the following Lemma which gives the FIM \wrt $\bm{h}_L$,
which is a $n_L\times{n}_L$ matrix in simple closed form for
commonly used probability distributions.
\begin{lemma}\label{thm:exp}
For the neural network model specified in \cref{eq:exp},
\begin{equation*}
\fim( \bm{h}_L )
= \cov( \bm{t}(\bm{y}) )
= \frac{\partial\bm\eta}{\partial\bm{h}_L},
\end{equation*}
where $\cov(\cdot)$ denotes the covariance matrix \wrt \( p(\bm y \, \vert \, \bm x, \bm \theta )\),
$\bm\eta \defeq \bm\eta(\bm{h}_L) \defeq \partial{}F/\partial\bm{h}_L$ is the expectation parameters,
and the vector-vector-derivative ${\partial\bm\eta}/{\partial\bm{h}_L}$ denotes the Jacobian matrix
of the mapping $\bm{h}_L\to\bm\eta$.
\end{lemma}

The derivatives of the log-likelihood
$\ell(\bm\theta)\defeq\log{p}(\bm{x},\bm{y}\,\vert\,\bm\theta)$
characterize its landscape and are essential to compute the FIM.
By \cref{eq:exp}, the score function (gradient of $\ell$) is
\begin{equation}\label{eq:gradl}
\frac{\partial\ell}{\partial\bm\theta}
=
\left(\frac{\partial\bm{h}_L}{\partial\bm\theta}\right)^\T
( \bm{t}(\bm{y}) - \bm\eta(\bm{h}_L) )
=
\frac{\partial\bm{h}_L^a}{\partial\bm\theta}(\bm{t}_a - \bm\eta_a).
\end{equation}
In this paper, we mix the usual $\Sigma$-notation of summation with the Einstein notation:
in the same term, an index appearing in both upper- and lower-positions indicates a sum over this index.
For example, $\bm{t}_a\bm{h}^a = \sum_{a} \bm{t}_a \bm{h}_a$.
Hence, in our equations, upper- and lower-indexes have the same meaning:
both $\bm{h}^a$ and $\bm{h}_a$ mean the $a$'th element of $\bm{h}$.
For convenience and consistency, we take quantities \wrt \( \bm \theta \) as upper indexed
and other quantities as lower indexed, \ie, \( \fim^{ij}(\bm \theta) \) versus \( \fim_{ij}(\bm h_{L}) \).
This mixed representation of sums helps to simplify our expressions without causing confusion.
From \cref{eq:gradl} and \cref{thm:exp}, the Hessian of $\ell$ is given by
\begin{equation}\label{eq:hessianl}
\frac{\partial^2\ell}{\partial\bm\theta\partial\bm\theta^\T}
=
(\bm{t}_a - \bm\eta_a)
\frac{\partial^2\bm{h}_L^a}{\partial\bm\theta\partial\bm\theta^\T}
-
\frac{\partial\bm{h}_L^a}{\partial\bm\theta}
\frac{\partial\bm\eta_a}{\partial\bm\theta^\T}
=
(\bm{t}_a - \bm\eta_a)
\frac{\partial^2\bm{h}_L^a}{\partial\bm\theta\partial\bm\theta^\T}
-
\frac{\partial\bm{h}_L^a}{\partial\bm\theta}
\fim_{ab}(\bm{h}_L) \frac{\partial\bm{h}_L^b}{\partial\bm\theta^\T}.
\end{equation}

Similar to the case of a general statistical model, the FIM is equivalent
to the expectation of the Hessian of $-\ell$ as long as the activation function is smooth enough.
\begin{theorem}\label{thm:fim2}
Consider the neural network model in \cref{eq:exp}.
For any activation function $\sigma\in{}C^2(\Re)$ (both $\sigma'(z)$ and $\sigma''(z)$ exist and are continuous), we have $\fim(\bm\theta) = \expect_p\left( -\frac{\partial^2\ell}{\partial\bm\theta\partial\bm\theta^{\T}}\right)$.
\end{theorem}
\begin{remark}
ReLU networks do not have this equivalent expression as $\relu(z)$ is not differentiable at $z=0$.
\end{remark}

Through the definition of the FIM, or alternatively its equivalent formula in \cref{thm:fim2},
we arrive at the same expression
\begin{equation}\label{eq:fimexp}
\fim(\bm\theta) =
\left( \frac{\partial\bm{h}_L}{\partial\bm\theta} \right)^\T
\fim(\bm{h}_L)
\frac{\partial\bm{h}_L}{\partial\bm\theta^\T}
=
\frac{\partial\bm{h}_L^a}{\partial\bm\theta}
\fim_{ab}(\bm{h}_L) \frac{\partial\bm{h}_L^b}{\partial\bm\theta^\T}.
\end{equation}
\Cref{eq:fimexp} takes the form of a generalized Gauss-Newton matrix~\citep{martens}.
This general expression of the FIM has been known in the literature~\citep{park,pbRNG}.
Under weak conditions, $\fim(\bm\theta)$ is a pullback metric~\citep{pull} of $\fim(\bm{h}_L)$
in \cref{thm:exp} associated with the mapping $\bm\theta\to\bm{h}_L$.
To compute $\fim(\bm\theta)$ in closed form, one need to first compute the Jacobian matrix
of size ${n}_L\times\dim(\bm\theta)$
then perform the matrix multiplication in \cref{eq:fimexp}.
The naive algorithm to evaluate \cref{eq:fimexp} has a computational complexity of
$\bigoh(n_L^2\dim(\bm\theta) + n_L\dim^2(\bm\theta))$,
where the term $\bigoh(n_L\dim^2(\bm\theta))$ is dominant as
$\dim(\bm\theta)\gg{n}_L$ in deep architectures.
Once the parameter $\bm\theta$ is updated, the FIM has to be recomputed.
This is infeasible in practice for large networks where
$\dim(\bm\theta)$ can be millions or billions.

The two estimators $\efima(\bm \theta)$ and $\efimb(\bm \theta)$ in \cref{eq:estimators} provide
a computationally inexpensive way to estimate the FIM. Given $\bm\theta$ and $\bm{x}$, one can draw \iid samples
$\bm{y}_1,\ldots,\bm{y}_N\sim{p}(\bm{y}\,\vert\,\bm{x},\bm\theta)$.
Both $\partial\ell_i/\partial\bm\theta$ and
$\partial^2\ell_i/\partial\bm\theta\partial\bm\theta^\T$
can be evaluated directly through auto-differentiation (AD) that is highly optimized for modern GPUs.
For $\efima(\bm\theta)$,
we already have $\partial\ell_i/\partial\bm\theta$ to perform gradient descent.
For $\efimb(\bm\theta)$,
efficient methods to compute the Hessian are implemented in AD frameworks such as PyTorch~\citep{torch}.
Using these derivatives, the computational cost only scales with the number $N$
of samples but does not scale with $n_L$.

We rarely need the full FIM of size $\dim(\bm\theta)\times\dim(\bm\theta)$.
Most of the time, only its diagonal blocks are needed, where
each block corresponds to a subset of parameters, \eg, the neural network weights
of a particular layer. Therefore the computation of both estimators
can be further reduced.

If $p(\bm{y}\,\vert\,\bm{x},\bm\theta)$ has the parametric form in \cref{eq:exp},
from \cref{eq:gradl,eq:hessianl}, the FIM estimators become
\begin{align}
\efima(\bm \theta) &=
    \frac{\partial\bm{h}_L^a}{\partial\bm\theta}
    \cdot
    \frac{1}{N} \sum_{i=1}^{N}
    (\bm{t}_a(\bm{y}_i) - \bm\eta_a)
    (\bm{t}_b(\bm{y}_i) - \bm\eta_b)
    \cdot
    \frac{\partial\bm{h}_L^b}{\partial\bm\theta^\T}, \label{eq:efima} \\
\efimb(\bm \theta) &=
    \left( \bm\eta_a - \frac{1}{N}\sum_{i=1}^N\bm{t}_a(\bm{y}_i) \right)
    \frac{\partial^2\bm{h}_L^a}{\partial\bm\theta\partial\bm\theta^\T}
    +
    \frac{\partial\bm{h}_L^a}{\partial\bm\theta}
    \fim_{ab}(\bm{h}_L)
\frac{\partial\bm{h}_L^b}{\partial\bm\theta^\T}.\label{eq:efimb}
\end{align}
Recall that the notation of $\efima(\bm \theta)$ and $\efimb(\bm \theta)$ is abused as they depend on $\bm{x}$ and $\bm{y}_1\cdots\bm{y}_N$.
Notably, in \cref{eq:efima}, $\efima(\bm \theta)$ is expressed in terms of the Jacobian matrix of the mapping $\bm\theta\to\bm{h}_L$
and the empirical variance of the minimal sufficient statistic $\bm{t}(\bm{y}_i)$
of the output exponential family.
In \cref{eq:efimb},
$\efimb(\bm \theta)$ depends on both the Jacobian and Hessian of $\bm\theta\to\bm{h}_L$
and the empirical average of $\bm{t}(\bm{y}_i)$.
The second term on the right-hand side (RHS) of \cref{eq:efimb} is exactly the FIM,
and therefore the first term serves as a bias term.
\cref{eq:efima,eq:efimb} are only for the case with exponential family output.
If the output units belong to non-exponential families, \eg, a statistical mixture model,
one falls back to the general formulae, \ie, \cref{eq:estimators} for the FIM.

As an application of the Fisher information, the Cram\'er-Rao lower bound (CRLB)
states that any unbiased estimator $\hat{\bm \theta}$
of the parameters $\bm \theta$ satisfies $\cov(\hat{\bm \theta})\succeq[\fim(\bm \theta)]^{-1}$.
For example, in \cref{thm:exp}, the FIM is \wrt the output of the neural network.
As such, \( \fim(\bm h_{L}) \) can be used to study
the estimation covariance of $\bm{h}_L$ based on
random samples $\bm{y}_1\cdots\bm{y}_N$ drawn from $p(\bm{y}\,\vert\,\bm{x},\bm\theta)$.
Similarly for \cref{eq:fimexp}, we can consider unbiased estimators of
the weights of the neural network.
In any case, to apply the CRLB, one needs an accurate estimation of $\fim(\bm \theta)$.
If the scale of $\fim(\bm \theta)$ is relatively small when compared to its covariance,
its estimation $\hat{\fim}(\bm \theta)$
is more likely to be a small positive value (or even worse, zero or negative).
The empirical computation of the CRLB is not meaningful in this case.

\section{The Variance of the FIM Estimators}\label{sec:estimators}

Based on the deep learning architecture specified in \cref{eq:exp},
we measure the quality of the two estimators $\efima(\bm\theta)$
and $\efimb(\bm\theta)$ given by their variances. Given the same sample size $N$,
a smaller variance is preferred as the estimator is more accurate and likely
to be closer to the true FIM $\fim(\bm\theta)$.
We study how the structure of the exponential family has an impact on the variance.

\subsection{Variance in closed form}
\label{subsec:var_in_closed_form}

We first consider $\efima(\bm\theta)$ and $\efimb(\bm\theta)$
in \cref{eq:estimators} as real matrices of dimension ${\dim(\bm\theta)\times\dim(\bm\theta)}$.
As $\efima(\bm\theta)$ is a square matrix,
the corresponding covariance is a 4D tensor
$\left[ \cov \left( \efima(\bm \theta) \right) \right]^{ijkl}$
of dimension ${\dim(\bm\theta)\times\dim(\bm\theta)\times\dim(\bm\theta)\times\dim(\bm\theta)}$,
representing the covariance between the two elements
$\efima^{ij}(\bm \theta)$ and $\efima^{kl}(\bm \theta)$.
The element-wise variance of
$\efima(\bm\theta)$ is a matrix with the same size of
$\efima(\bm\theta)$, which we denote as
$\var(\efima(\bm\theta))$.
Thus,
\begin{equation}\label{eq:varcov}
\var(\efima(\bm\theta))^{ij} =
\left[ \cov \left( \efima(\bm \theta) \right) \right]^{ijij}.
\end{equation}
Similarly, the covariance and element-wise variance of $\efimb(\bm\theta)$
are denoted as
$\left[ \cov \left( \efimb(\bm \theta) \right) \right]^{ijkl}$
and
$\var(\efimb(\bm\theta))^{ij}$, respectively.

As the samples $\bm{y}_1,\ldots,\bm{y}_N$ are \iid, we have
\begin{equation} \label{eq:varfims}
\cov(\efima(\bm\theta)) = \frac{1}{N}
\cov \left( \frac{\partial\ell}{\partial\bm\theta} \frac{\partial\ell}{\partial\bm\theta^{\T}} \right)
\quad \text{and} \quad
\cov(\efimb(\bm\theta)) = \frac{1}{N}
\cov \left(  \frac{\partial^{2}\ell}{\partial\bm\theta \partial\bm\theta^{\T}}  \right).
\end{equation}
Both $\cov(\efima(\bm\theta))$ and $\cov(\efimb(\bm\theta))$ have an order of $\bigoh(1/N)$.
For the neural network model in \cref{eq:exp}, we further have those covariance tensors
in closed form.

\begin{theorem}\label{thm:varefima}
    \begin{align*}
\left[
        \cov
        \left(
            \efima(\bm \theta)
        \right)
        \right]^{ijkl}
        &=
        \frac{1}{N}
        \cdot
        \cov
        \left(
            \frac{\partial\ell}{\partial\bm\theta_{i}} \frac{\partial\ell}{\partial\bm\theta_{j}}
            ,
            \frac{\partial\ell}{\partial\bm\theta_{k}} \frac{\partial\ell}{\partial\bm\theta_{l}}
        \right)\\
        &=
        \frac{1}{N}
        \cdot
        \partial_{i} {\bm h}_{L}^{a}(\bm x) \partial_{j} {\bm h}_{L}^{b}(\bm x) \partial_{k} {\bm h}_{L}^{c}(\bm x) \partial_{l} {\bm h}_{L}^{d}(\bm x)
        \cdot
        \left( \kurt_{abcd}(\bm{t}) - \fim_{ab}(\bm{h}_L) \cdot \fim_{cd}(\bm{h}_L) \right),
    \end{align*}
where the 4D tensor
\begin{equation*}
\kurt_{abcd}( \bm{t} )
\defeq
\expect\left[
(\bm{t}_a - \bm{\eta}_a(\bm{h}_{L}(\bm{x})))
(\bm{t}_b - \bm{\eta}_b(\bm{h}_{L}(\bm{x})))
(\bm{t}_c - \bm{\eta}_c(\bm{h}_{L}(\bm{x})))
(\bm{t}_d - \bm{\eta}_d(\bm{h}_{L}(\bm{x})))
\right]
\end{equation*}
is the 4th (unscaled) central moment
\footnote{The kurtosis of a random variable is defined by its 4th
standardized (both centered and normalized) moment.
Here, $\kurt(\cdot)$ denotes the 4th central moment but \emph{not} the kurtosis.} of $\bm{t}(\bm{y})$
and $\partial_i\bm{h}_L(\bm{x})\defeq\partial\bm{h}_L(\bm{x})/\partial\bm\theta_i$.
\end{theorem}
\begin{remark}\label{remark:kur}
The 4D tensor
$\left( \kurt_{abcd}(\bm{t}) - \fim_{ab}(\bm{h}_L) \cdot \fim_{cd}(\bm{h}_L) \right)$
is the covariance of the random matrix
\begin{equation*}
\frac{\partial\ell}{\partial\bm{h}_L} \frac{\partial\ell}{\partial\bm{h}_L^\T}
=
(\bm{t}(\bm{y}) - \bm\eta) (\bm{t}(\bm{y}) - \bm\eta)^\T,
\end{equation*}
where $\bm{y}\sim{}p(\bm{y}\,\vert\,\bm{h}_L)$.
This random matrix is an estimator of $\fim(\bm{h}_L)$, \ie the FIM \wrt the natural parameters $\bm{h}_L$.
\Cref{thm:varefima} describes how the covariance tensor adapts \wrt
the coordinate transformation $\bm{h}_L\to\bm\theta$.
\end{remark}

Notably, as \( \bm t(\bm y) \) is the sufficient statistics of an exponential family, the derivatives of the log-partition function \( F(\bm h) \) \wrt the natural parameters \( \bm{h} \) are equivalent to the \emph{cumulants} of \( \bm t (\bm y) \). The cumulants correspond to the coefficients of the Taylor expansion of the logarithm of the moment generating function~\citep{mccullagh2018tensor}.
Importantly, the cumulants of order 3 and below are equivalent to the central moments (see \eg \cref{thm:exp}).
However, this is not the case for the 4th central moment which must be expressed as a combination
of the 2nd and 4th cumulants, as stated in the following Lemma.
\begin{lemma} \label{lem:kurt}
\begin{equation*}
        \kurt_{abcd}( \bm{t} )
        = \kappa_{abcd} + \fim_{ab}(\bm h_{L}) \cdot \fim_{cd}(\bm h_{L}) + \fim_{ac}(\bm h_{L}) \cdot \fim_{bd}(\bm h_{L}) + \fim_{ad}(\bm h_{L}) \cdot \fim_{bc}(\bm h_{L}),
    \end{equation*}
    where
\begin{equation*}
        \kappa_{abcd}
        \defeq \left.\frac{\partial^{4} F(\bm h)}{\partial \bm h_{a} \partial \bm h_{b} \partial \bm h_{c} \partial \bm h_{d}}\right\vert_{\bm{h}=\bm{h}_{L} (\bm x)}.
    \end{equation*}
\end{lemma}

\begin{remark}
In the 1D case, the 4th central moment simplifies to
$\kurt(\bm t) = F''''(\bm h_{L}) + 3(F''(\bm h_{L}))^{2}$.
\end{remark}

For the second estimator \( \efimb(\bm \theta) \), the covariance only depends on the 2nd central moment
of $\bm{t}(\bm{y})$.

\begin{theorem}
    \label{thm:varefimb}
    \begin{align*}
\left[
        \cov
        \left(
            \efimb(\bm \theta)
        \right)
        \right]^{ijkl}
        =
        \frac{1}{N}
        \cdot
        \cov
        \left(
            -\frac{\partial^{2}\ell}{\partial\bm\theta_{i} \partial\bm\theta_{j}},
            -\frac{\partial^{2}\ell}{\partial\bm\theta_{k} \partial\bm\theta_{l}}
        \right)
        &=
        \frac{1}{N}
        \cdot
        \partial^{2}_{ij}{\bm h}_{L}^{\alpha}({\bm x})
        \partial^{2}_{kl}{\bm h}_{L}^{\beta}({\bm x})
        \fim_{\alpha\beta}( \bm{h}_L ),
    \end{align*}
where
$\partial_{ij}^2\bm{h}_L(\bm{x}) \defeq
{\partial^2\bm{h}_L(\bm{x})}
/
{\partial\bm\theta_{i}\partial\bm\theta_j}$\footnote{In this paper,
the derivatives are by default taken \wrt $\bm\theta$. Therefore,
$\partial_i\defeq\frac{\partial}{\partial\bm\theta_i}$ and
$\partial_{ij}^2\defeq\frac{\partial^2}{\partial\bm\theta_i\partial\bm\theta_j}$.}.
\end{theorem}

\begin{remark}\label{thm:remark}
By \cref{thm:exp},
the matrix $\fim_{\alpha\beta}(\bm{h}_L)$
is the covariance of the sufficient statistic $\bm{t}(\bm{y})$.
Hence, the covariance of $\efimb(\bm \theta)$
scales with the covariance of
$\bm{t}(\bm{y})$. If $\bm{t}(\bm{y})$ tends to be deterministic, then the
covariance of $\efimb(\bm\theta)$ shrinks towards 0 and its
estimation of the FIM becomes accurate.
\end{remark}

The covariance in \cref{thm:varefima,thm:varefimb} has two different components:
\ding{192} the derivatives of the deep neural network;
and
\ding{193} the central (unscaled) moments of \( \bm t (\bm y) \).
The 4D tensor $\kurt_{abcd}(\bm t)$ and the 2D FIM $\fim_{\alpha\beta}(\bm h_{L})$ correspond to
the 4th and 2nd central moments of $\bm{t}(\bm y)$, respectively.
Intuitively, the larger the scale of the Jacobian or the Hessian of the neural
network mapping $\bm\theta\to\bm{h}_L$
and/or the larger the central moments of the exponential family,
the lower the accuracy when estimating the FIM.

Additionally, in \cref{sup:changecoordinates} we show that under reparametrization of the neural network weights \( \efima(\bm \theta) \) is a covariant tensor, just like the FIM. Contrarily, \( \efimb(\bm \theta) \) does not have this property.

\subsection{Variance Bounds}
\label{subsec:variance_bounds}

We aim to derive meaningful upper bounds of the covariances presented in \cref{thm:varefima,thm:varefimb}.
Using the Cauchy-Schwarz inequality, we can
``decouple'' the derivatives of the neural network mapping
and the central moments of the exponential family
into different terms. This provides various bounds on the scale of covariance quantities.

\begin{theorem}\label{thm:varfima11}
    \begin{align*}
        \left\Vert \cov \left( \efima(\bm \theta) \right) \right\Vert_F
        & \le
        \frac{1}{N}
        \cdot
        \left\Vert\frac{\partial\bm{h}_L}{\partial\bm\theta}\right\Vert_F^4
        \cdot
        \Vert \kurt(\bm{t}) - \fim(\bm{h}_L) \otimes \fim(\bm{h}_L) \Vert_{F},
    \end{align*}
    where  $\otimes$ is the tensor-product:
    $\left( \fim(\bm{h}_L) \otimes \fim(\bm{h}_L) \right)_{abcd} \defeq \fim_{ab}(\bm{h}_L) \cdot \fim_{cd}(\bm{h}_L)$.
\end{theorem}
The scale $\left\Vert \cov \left( \efima(\bm \theta) \right) \right\Vert_F$
measures how much the estimator $\efima(\bm\theta)$ deviates from
$\fim(\bm\theta)$. \Cref{thm:varfima11} says that this deviation
is bounded by the scale of the Jacobian matrix
${\partial\bm{h}_L}/{\partial\bm\theta}$
as well as the scale of $(\kurt(\bm{t}) - \fim(\bm{h}_L) \otimes \fim(\bm{h}_L))$.
Recall from \cref{remark:kur} the latter measures the variance
when estimating the FIM $\fim(\bm{h}_L)$ of the exponential family.
\Cref{thm:varfima11} allows us to study these two different factors
separately. Similarly, we have an upper bound on the scale of the covariance of \( \efimb(\bm\theta) \).

\begin{theorem}\label{thm:varfimb11}
    \begin{align*}
        \left\Vert \cov \left( \efimb(\bm \theta) \right) \right\Vert_F
        & \le
        \frac{1}{N} \cdot
        \left\Vert
        \frac{\partial^2\bm{h}_L(\bm{x})}{\partial\bm\theta\partial\bm\theta^\T}
        \right\Vert_{F}^2
        \cdot
        \Vert \fim(\bm{h}_L) \Vert_{F}.
    \end{align*}
\end{theorem}
On the RHS, the Hessian ${\partial^2\bm{h}_L(\bm{x})}/{\partial\bm\theta\partial\bm\theta^\T}$
is a 3D tensor of shape $n_L\times{}\dim(\bm\theta)\times\dim(\bm\theta)$.
Therefore, the variance of $\efimb(\bm \theta)$
is bounded by the scale of the Hessian, as well as the scale of the FIM $\fim(\bm{h}_L)$
of the output exponential family.

We consider an upper bound to further simplify related terms in \cref{thm:varfima11,thm:varfimb11}.
\begin{lemma}\label{thm:elementwisemoment}
\begin{align*}
\left\Vert \kurt(\bm{t}) - \fim(\bm{h}_L) \otimes \fim(\bm{h}_L) \right\Vert_{F}
&\le
\sqrt{2}
\left( \sum_{a=1}^{n_{L}}
\left( \sqrt{\kurt_{aaaa}(\bm{t})} + \fim_{aa}(\bm{h}_L) \right) \right)^2,\\
\left\Vert \fim(\bm{h}_L) \right\Vert_{F}
&\le
\sum_{a=1}^{n_{L}} \fim_{aa}(\bm{h}_L).
\end{align*}
\end{lemma}
\begin{remark}
Using \cref{thm:elementwisemoment}, it is straightforward to bound the scale
of the covariance tensors with the size of the Jacobian/Hessian, as well as the
central moments $\kurt_{aaaa}(\bm{t})$ and $\fim_{aa}(\bm{h}_L)$. These bounds
are meaningful but omitted for brevity.
\end{remark}
\begin{remark}
By \cref{thm:elementwisemoment},
$\Vert \kurt(\bm{t}) - \fim(\bm{h}_L) \otimes \fim(\bm{h}_L)\Vert_F$
is in the order of $\bigoh(n_L^2)$ and
$\Vert \fim(\bm{h}_L) \Vert_F$ is in the order of $\bigoh(n_L)$.
\end{remark}

The scale of the tensors
$\kurt(\bm{t}) - \fim(\bm{h}_L) \otimes \fim(\bm{h}_L)$ and $\fim(\bm{h}_L)$
is bounded by the diagonal elements of $\kurt(\bm{t})$ and $\fim(\bm{h}_L)$,
or the element-wise central moments of $\bm{t}(\bm{y})$.
Understanding the scale of these 1D central moments helps
to understand the scale of the moment terms in our key statements.

\Cref{tab:exp_fam} presents some 1D exponential families and their cumulants.
\Cref{fig:exp_fam} displays \( \kurt(\bm t) - \var^2(\bm{t}) \) and \( \var(\bm{t}) \)
against the mean of these distributions.
Based on \cref{fig:bernoulli}, if the neural network has Bernoulli output units,
then the scale of $\kurt_{aaaa}(\bm{t}) - (\fim_{aa}(\bm{h}_L))^2$ is smaller
than $\fim_{aa}(\bm{h}_L)$ regardless of $\bm{h}_L$.
Notably, when \( p = 0.5 \), the variance of the first estimator \( \efima(\bm \theta) \)
is 0 --- regardless of $\bm h_{L}$.
For normal distribution output units (corresponding to the mean squared error loss) in \cref{fig:normal},
both central moment quantities are constant.
For Poisson output units in \cref{fig:poisson},
$\fim_{aa}(\bm{h}_L)$ increases linearly with the average number of events \( \lambda \),
while \( \kurt_{aaaa}(\bm t) - (\fim_{aa}(\bm{h}_L))^2 \) increases quadratically.
Thus, the upper bound of $\Vert\cov(\efima(\bm\theta))\Vert_F$ increases faster than
the upper bound of $\Vert\cov(\efimb(\bm\theta))\Vert_F$
as \( \bm h_{L} \) enlarges.
In this case, one may prefer \( \efimb(\bm \theta) \) rather than \( \efima(\bm \theta) \)
and/or control the scale of \( \bm h_{L} \).
In general, $\bm{h}_L$ is desired to be in certain regions in the parameter space of the exponential family
to control the estimation variance of the FIM.
Techniques to achieve this include
regularization on the scale of \( \bm h_{L} \);
temperature scaling~\citep{hinton2015distilling};
or normalization layers~\citep{ba2016layer,weightnorm}.
Of course, they could inversely increase the scale of the derivatives of the neural network,
which can be controlled by imposing additional constraints, \ie, Lipschitz requirements.

See \cref{sup:experimental} for numerical verifications
of the bounds in \cref{thm:varfima11,thm:varfimb11} on the MNIST dataset.

\begin{table}[t]
    \vskip -0.1in
    \caption{Cumulants of univariate exponential family distributions, given by derivatives of the log-partition function.
$p$, $\mu$ and $\lambda$ denote the mean of the Bernoulli, normal, and Poisson distributions, respectively.
    \textdagger~The normal distribution has unit standard deviation (\( \sigma = 1 \)).}
    \centering
    \begin{sc}
    \begin{tabularx}{\textwidth}{XMSSB}
        \toprule
        Dist. & \( F(\bm h) \) & \( \bm h \) & \( \partial^{2} F(\bm h) \) & \( \partial^{4} F(\bm h) \) \\
        \midrule
        Bernoulli & \( \log(1 + \exp(\bm h)) \) & \( \log \nicefrac{p}{1-p} \) & \( p(1-p) \) & \( p(1-p)(6p^2 - 6p + 1) \) \\
        Normal\textdagger & \( \nicefrac{\bm h^2}{2} \) & \( {\mu} \) & 1 & 0 \\
        Poisson & \( \exp(\bm h) \) & \( \log \lambda \) & \( \lambda \) & \( \lambda \) \\
        \bottomrule
    \end{tabularx}
    \end{sc}
    \label{tab:exp_fam}
\end{table}

\begin{figure}
    \centering
    \begin{subfigure}[t]{0.3\textwidth}
\centering
        \includegraphics[width=\textwidth]{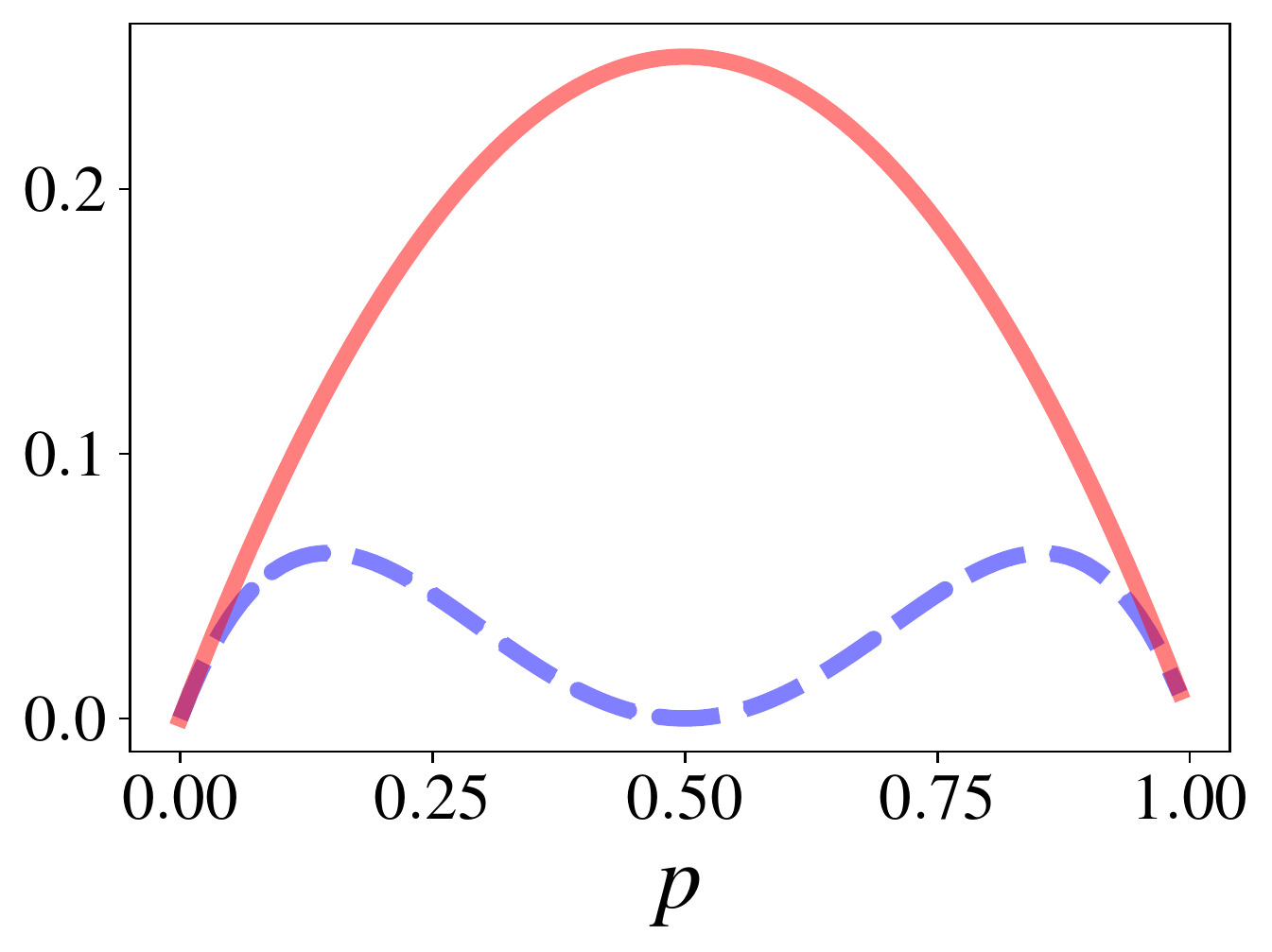}
        \caption{Bernoulli.}
        \label{fig:bernoulli}
    \end{subfigure}
    \hfill
    \begin{subfigure}[t]{0.3\textwidth}
\centering
        \includegraphics[width=\textwidth]{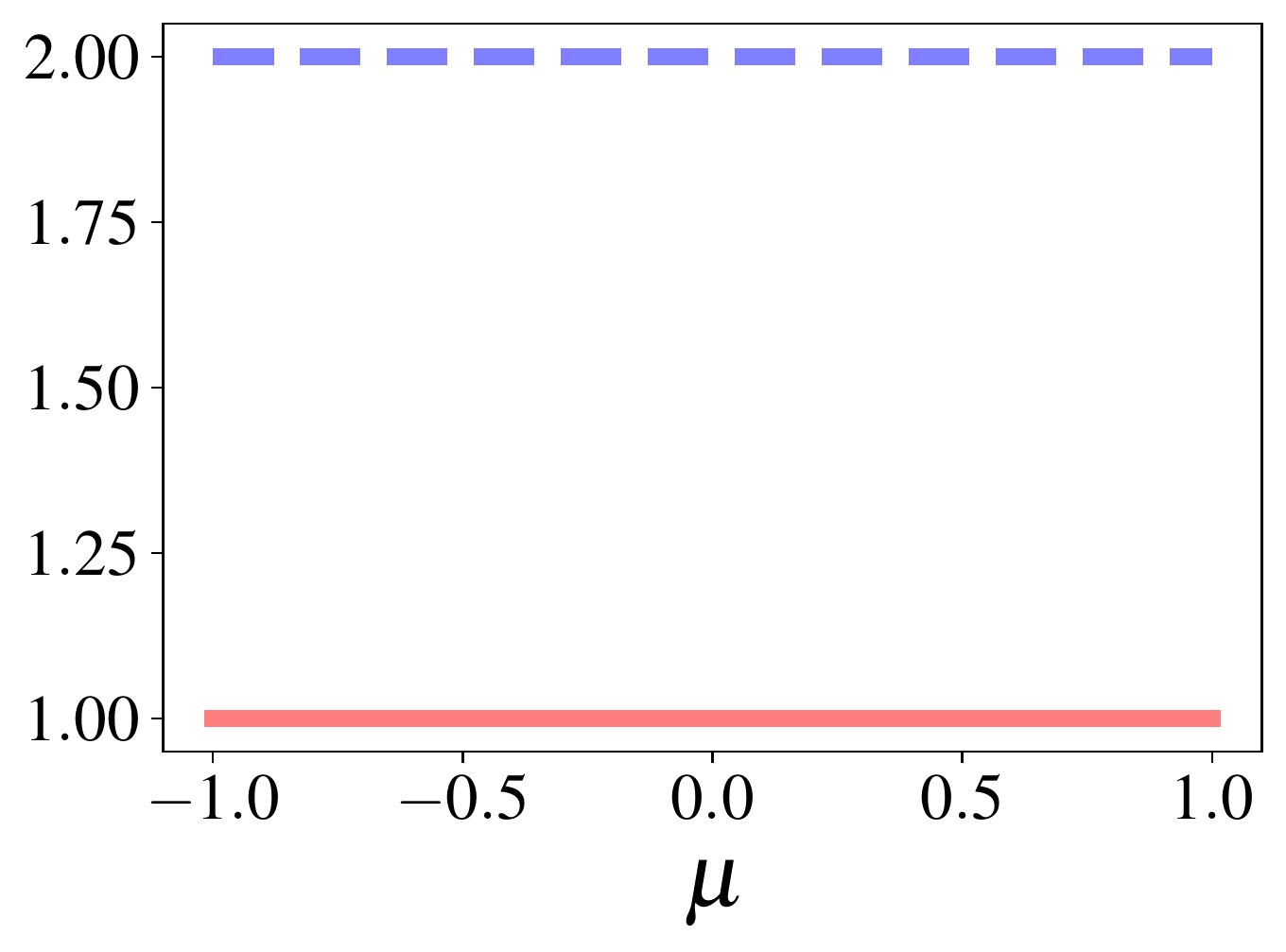}
        \caption{Normal (\( \sigma = 1 \)).}
        \label{fig:normal}
    \end{subfigure}
    \hfill
    \begin{subfigure}[t]{0.3\textwidth}
\centering
        \includegraphics[width=\textwidth]{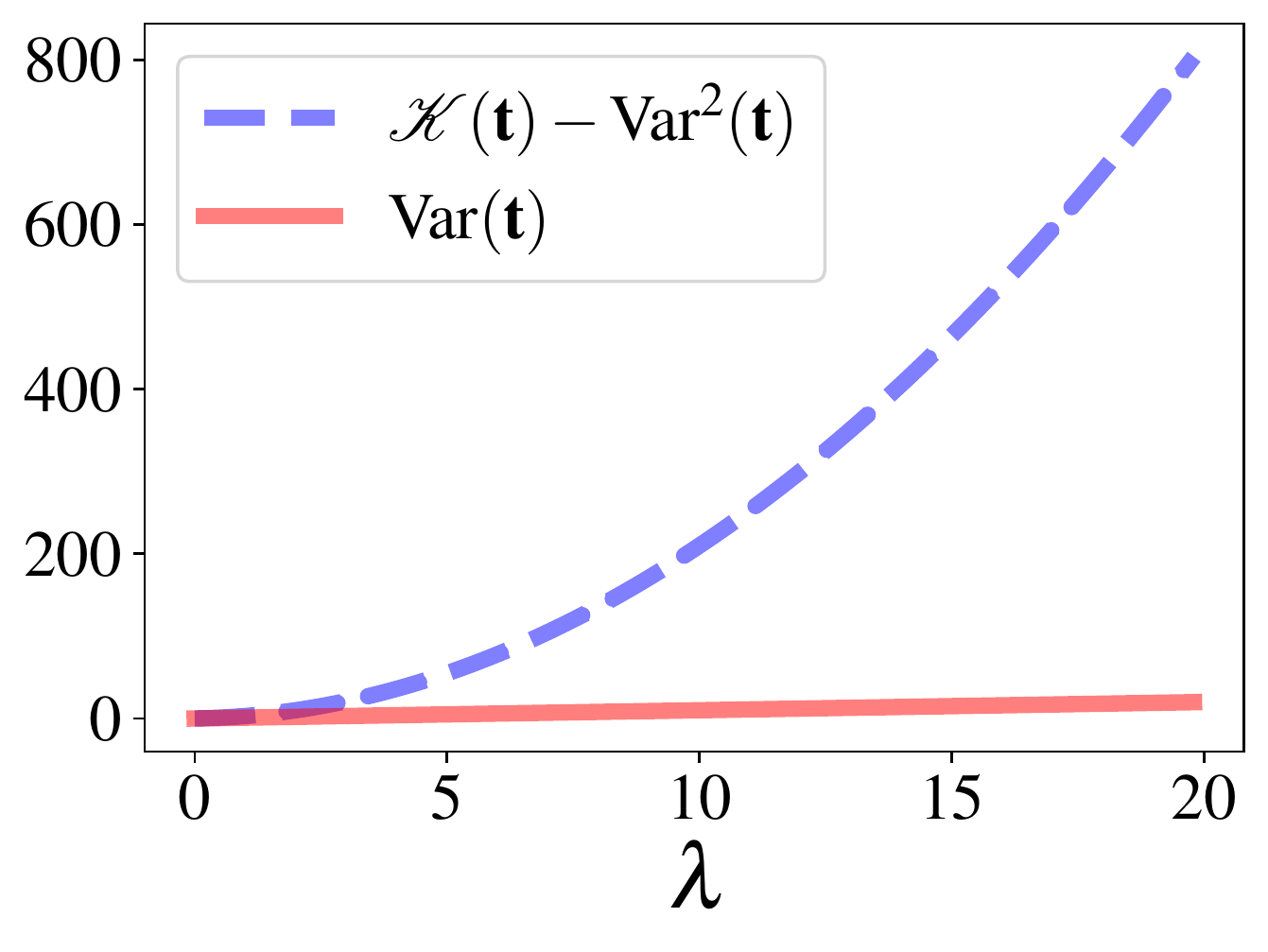}
        \caption{Poisson.}
        \label{fig:poisson}
    \end{subfigure}
       \caption{The scale of
           $\kurt(\bm{t}) - \var^2(\bm{t})$ and $\var(\bm{t})$ for the exponential family distributions in \cref{tab:exp_fam}.
}
       \label{fig:exp_fam}
\end{figure}

\subsection{Positive Definiteness}

By definition, the FIM of any statistical model is positive semidefinite (\psd).
The first estimator \( \efima(\bm \theta) \) is naturally on the \psd manifold (space of \psd matrices).
On the other hand, \( \efimb(\bm \theta) \) can ``fall off'' the \psd manifold.
It is important to examine the likelihood for
\( \efimb(\bm \theta) \) having a negative spectrum and the corresponding scale,
so that any algorithm (\eg natural gradient) relying of the FIM being \psd can be adapted.

\cref{eq:efimb} can be re-expressed as the sum of a \psd matrix and a linear combination
of \( n_{L} \) symmetric matrices.
We provide the likelihood for \( \efimb(\bm \theta) \) staying on the \psd manifold
given conditions on the spectrum of the Hessian.

\begin{theorem}
    \label{thm:est2_psd}
    Let \( \lambda_{\min}(\cdot) \), \( \lambda_{\max}(\cdot) \), and \( \rho(\cdot) \)
    denote the smallest eigenvalue, the largest eigenvalue, and the spectral radius
    (largest absolute value of the spectrum), respectively.
    Let $\bm\rho \defeq ( \rho(\partial^{2} \bm h^{1}_{L}), \cdots
    \rho(\partial^{2} \bm h^{n_L}_{L}) )$.
    If $\lambda_{\min}(\fim(\bm\theta))>0$, then with probability at least
\begin{align*}
        1 - \frac{n_L \cdot \Vert\bm\rho\Vert_2^2 \cdot \lambda_{\max}(\fim(\bm{h}_L))}
        {N \cdot \lambda_{\min}^2(\fim(\bm \theta))},
    \end{align*}
the estimator \( \efimb(\bm \theta) \) with $N$ samples is a \psd matrix.
\end{theorem}

The bound becomes uninformative as the output layer size \( n_{L} \) increases,
as the spectrum of the Hessian of \( \bm h_{L} \) scales up,
or as the spectrum of the FIM \( \fim(\bm{h}_L) \) enlarges.
On the other hand, as the minimal eigenvalue of the FIM \( \fim(\bm \theta) \)
increases, \cref{thm:est2_psd} can give meaningful lower bounds.
In particular, with sample rate \( \bigoh(N^{-1}) \), estimator \( \efimb(\bm \theta) \) will be a \psd matrix.
In practice for over-parametrized networks, $\lambda_{\min}({\fim(\bm \theta)})$ is close to
or equals 0 and \cref{thm:est2_psd} is not meaningful.
In any case, we need to consider the scale of the negative spectrum of $\efimb(\bm\theta)$.

\begin{theorem}\label{thm:worstspectrum}
\begin{equation*}
\lambda_{\min}\left( \efimb(\bm\theta) \right)
\ge
-
\rho(\partial^{2} \bm h^{a}_{L}(\bm x))
\left\vert\bm\eta_a - \frac{1}{N}\sum_{i=1}^N \bm{t}_a(\bm{y}_i) \right\vert.
\end{equation*}
\end{theorem}
\Cref{thm:worstspectrum} guarantees that in the worst case, the scale of the negative spectrum
of $\efimb(\bm\theta)$ is controlled.
By \cref{thm:exp}, $\var( \bm\eta_a - \frac{1}{N}\sum_{i=1}^N \bm{t}_a(\bm{y}_i))=\frac{1}{N}\fim^{aa}(\bm h_{L})$.
Therefore, as $N$ increases or $\fim^{aa}(\bm h_{L})$ decreases,
the negative spectrum of $\efimb(\bm\theta)$ will shrink.
Further analysis on the spectrum of $\efima(\bm\theta)$ and $\efimb(\bm\theta)$
can utilize the geometric structure of the \psd manifold. This is left for future work.

\subsection{Convergence Rate}

The rate of convergence for each of the estimators is of particular interest
when considering their practical viability.  Through a generalized Chebyshev
inequality~\citep{chen2007new}, we can get a simple Frobenius norm convergence rate.

\begin{lemma}
    \label{thm:est_cr_chebyshev}
    Let \( 0 < \varepsilon < 1 \). Then
    \begin{align*}
        \left\Vert \efima(\bm \theta) - {\fim(\bm \theta)} \right\Vert_{F}
        &\leq
        \frac{1}{\sqrt{\varepsilon N}} \cdot
\sqrt{
            \sum^{\dim(\bm \theta)}_{i, j = 1}
        \var
        \left(  \frac{\partial\ell}{\partial\bm\theta_{i}} \frac{\partial\ell}{\partial\bm\theta_{j}} \right)
        } \\
        \intertext{holds with probability at least \( 1 - \varepsilon \);
        and}
    \left\Vert \efimb(\bm \theta) - {\fim(\bm \theta)} \right\Vert_{F}
        &\leq
        \frac{1}{\sqrt{\varepsilon N}} \cdot
\sqrt{
            \sum^{\dim({\bm \theta})}_{i,j = 1}
        \var
        \left(  -\frac{\partial^{2}\ell}{\partial\bm\theta_{i} \partial\bm\theta_{j}}  \right)
        }
    \end{align*}
    hold with probability at least \( 1 - \varepsilon \).
\end{lemma}

Each of these convergence rates only depends on the element-wise variance
of the estimator terms in \cref{eq:varcov}.
Moreover, each of the estimators has a convergence rate of \( \bigoh(N^{-1/2}) \).
The rate's constants are determined by the variance of the estimators given by \cref{thm:varefima,thm:varefimb},
which are influenced by the derivatives of the neural network
and the moments of the output exponential family.

\section{Effect of Neural Network Derivatives}\label{sec:dl_structure}

The derivatives of the deep learning network can affect the estimation variance
of the FIM.  By \cref{thm:varefima}, the variance of the first estimator \(
\efima(\bm \theta) \) scales with the Jacobian of the neural network mapping
\(\bm\theta\to\bm h_{L}(\bm x)\). By \cref{thm:varefimb}, the variance of \( \efimb(\bm
\theta) \) scales with the Hessian of \(\bm\theta\to\bm h_{L}(\bm x)\).
The larger the scale of the Jacobian or the Hessian, the larger the estimation variance.
In this section, we examine these derivatives in more detail.

We give the closed form gradient of the log-likelihood \( \ell \)
and the last layer's output \( \bm h_{L} \) \wrt the neural network parameters.
\begin{lemma}
    \label{lem:dl_grad}
    \begin{equation*}
    \frac{\partial\ell}{\partial\bm{W}_{l}}
    =
    \bm{D}_l
    \frac{\partial\ell}{\partial\bm{h}_{l+1}} \bar{\bm{h}}_{l}^{\T},
    \quad
    \frac{\partial\ell}{\partial\bm{h}_{l}}
    =
    {\bm B}_{l}^\T \left( \bm{t}(\bm{y}) - \bm{\eta}(\bm{h}_L) \right),
    \quad
    \frac{\partial \bm h_{L}^{a}}{\partial \bm W_{l}}
    =
    \bm{D}_l \bm{B}_{l+1}^\T \bm{e}_{a} \bar{\bm{h}}_{l}^{\T},
    \end{equation*}
    where \( \bm e_{a} \) is the \( a^{\textrm{th}} \) standard basis vector,
    $\bm{B}_l$ and $\bm{D}_l$ are recursively defined by
    \begin{align*}
        &\bm{B}_{L}  = \bm{I},
         \quad
         \bm{B}_{l}  = \bm{B}_{l+1} \bm{D}_{l} \bm{W}^{-}_{l}, \nonumber\\
        &\bm{D}_{L-1} = \bm{I},
         \quad
         \bm{D}_{l}   = \diag\left( \sigma^{\prime}(\bm{W}_{l} \bar{\bm{h}}_{l}) \right),
    \end{align*}
    $\bm{I}$ is the identity matrix,
    and $\diag(\cdot)$ means a diagonal matrix with given diagonal entries.
\end{lemma}

By \cref{lem:dl_grad}, we can estimate the FIM \wrt the hidden representations \( \bm h_{l} \) through
\begin{equation}
\efima(\bm{h}_l)
=
\frac{1}{N}\sum_{i=1}^N
\frac{\partial\ell_i}{\partial\bm{h}_{l}}
\frac{\partial\ell_i}{\partial\bm{h}_{l}^\T}
=
\bm{B}_l^\T
\left(
\frac{1}{N}\sum_{i=1}^N
\left(\bm{t}(\bm{y}_i) - \bm{\eta}(\bm{h}_L) \right)
\left(\bm{t}(\bm{y}_i) - \bm{\eta}(\bm{h}_L) \right)^\T
\right)
\bm{B}_l.
\end{equation}
As $\bm{B}_l$ is recursively evaluated from the last layer to previous layers,
the FIM can also be recursively estimated based on $\efima(\bm\theta)$. It is similar
to back-propagation, except that the FIMs are back-propagated instead of
gradients of the network. This is similar to the backpropagated metric~\citep{ollivier}.

To investigate how the first estimator \( \efima(\bm \theta) \) is affected by
the loss landscape, we bound the Frobenius norm of the parameter-output Jacobian
$\partial\bm{h}_L/\partial\bm\theta$.
\begin{lemma}
    \label{lem:nn_grad_norm_bounded}
    If the activation function has bounded gradient and $\forall{z}\in\Re$,
    \( \vert \sigma^{\prime}(z) \vert \leq 1 \), then
    \begin{equation}
        \label{eq:nn_grad_norm_bounded}
        \left\Vert
            \frac{\partial \bm{h}_L}{\partial \bm{W}_l}
        \right\Vert_{F}
        =
        \Vert \bm{B}_{l+1} \bm{D}_l \Vert_F \cdot \Vert \bar{\bm{h}}_l \Vert_2
        \le
        \prod_{i=l+1}^{L-1}
        \Vert \bm{W}^{-}_{i} \Vert_F
        \cdot
        \Vert \bar{\bm{h}}_l \Vert_2,
    \end{equation}
\end{lemma}
where ${\partial \bm{h}_L}/{\partial \bm{W}_l} =
\left[
    {\partial \bm{h}_L^1}/{\partial \bm{W}_l},
    \cdots,
    {\partial \bm{h}_L^{n_L}}/{\partial \bm{W}_l}
    \right]$ is the derivative of a vector \wrt a matrix that is a 3D tensor.

Given \cref{lem:nn_grad_norm_bounded}, we see that the gradient $\partial
\bm{h}_L/\partial \bm{W}_l$ scales with both the neural network weights
$\bm{W}_i$ and the gradient of the activation function $\bm{D}_l$.
Common activation functions have both bounded outputs and 1st-order derivatives;
or at least are locally Lipschitz, \ie, sigmoid and ReLU activation functions.
During training, regularizing the scale of the neural network weights is a
sufficient condition for bounding the variance of \( \efima(\bm \theta) \).

An alternative bound can be established which depends on the maximum singular
values of the weight matrices.
\begin{lemma}
    \label{lem:nn_grad_norm_bounded2}
    Suppose that the activation function has bounded gradient
    $\forall{z}\in\Re$, $\vert \sigma^{\prime}(z) \vert \leq 1$. Then
\begin{equation}
        \label{eq:nn_grad_norm_bounded2}
        \left \Vert
        \frac{\partial \bm h_L}{\partial \bm W_{l}}
        \right \Vert_{2_\sigma}
        \leq
        \left( \prod_{i=l+1}^{L-1} s_{\max}(\bm W_{i}^-) \right)
        \cdot
        \Vert \bar{\bm{h}}_{l} \Vert_{2},
    \end{equation}
    where $s_{\max}(\cdot)$ denotes the maximum singular value and \( \Vert
    \mathcal{T} \Vert_{2_\sigma} \) denotes the tensor spectral norm for a 3D
    tensor $\mathcal{T}$, defined by
    \begin{equation*}
        \Vert \mathcal{T} \Vert_{2_{\sigma}} =
        \max
        \left\{ \langle
            \mathcal{T}, \bm\alpha \otimes \bm\beta \otimes \bm\gamma
            \rangle :
        \Vert \bm\alpha \Vert_2 = \Vert \bm\beta \Vert_2
        =\Vert \bm\gamma \Vert_2 = 1 \right\}.
    \end{equation*}
\end{lemma}
Therefore, regularizing $s_{\max}(\bm W_{i}^-)$, or the spectral norm of the
weight matrices, also helps to improve the estimation accuracy of the FIM.

We further reveal the relationship between the loss landscape and the
FIM estimators.
For a given target $\tilde{\bm{y}}$, the log-likelihood is denoted as
$\tilde{l}\defeq\log{p}(\tilde{\bm{y}}\,\vert\,\bm{x},\bm\theta)$.
Furthermore, let us define \( {\Delta}\efima(\bm \theta) \defeq ({\partial\tilde{l}}/{\partial\bm\theta}) ({\partial\tilde{l}}/{\partial\bm\theta^\T}) - \efima(\bm\theta) \) and \( {\Delta}\efimb(\bm \theta) \defeq -{\partial^2\tilde{l}}/ {\partial\bm\theta\partial\bm\theta^\T} - \efimb(\bm\theta) \).
By \cref{eq:gradl,eq:hessianl},
\begin{align*}
{\Delta}\efima(\bm \theta)
&=
    \frac{\partial\bm{h}_L^a}{\partial\bm\theta}
    \left[
    (\bm{t}_a(\tilde{\bm{y}}) - \bm\eta_a)
    (\bm{t}_b(\tilde{\bm{y}}) - \bm\eta_b)
    -
    \frac{1}{N} \sum_{i=1}^{N}
    (\bm{t}_a(\bm{y}_i) - \bm\eta_a)
    (\bm{t}_b(\bm{y}_i) - \bm\eta_b)
    \right]
\frac{\partial\bm{h}_L^b}{\partial\bm\theta^\T}, \nonumber\\
{\Delta}\efimb(\bm \theta)
&=
\left[
    \frac{1}{N}\sum_{i=1}^N\bm{t}_a(\bm{y}_i) - \bm{t}_a( \tilde{\bm{y}} )
\right]
\frac{\partial^2\bm{h}_L^a}{\partial\bm\theta\partial\bm\theta^\T}.
\end{align*}
Hence, the difference between $\efima(\bm\theta)$ (resp. $\efimb(\bm\theta)$)
and the squared gradient (resp. Hessian) of the loss $-\tilde{\ell}$ depends on
how $\bm{y}_i$ differs from $\tilde{\bm{y}}$. If the network $\bm\theta$ is
trained, then the random samples
$\bm{y}_i\sim{}p(\bm{y}\,\vert\,\bm{x},\bm\theta)$ are close to the given target
$\tilde{\bm{y}}$. In this case, $\efima(\bm\theta)$ corresponds to the squared
gradient, and $\efimb(\bm\theta)$ corresponds to the Hessian. This is not true
for untrained neural networks with random weights.

\section{Related Work}\label{sec:related}

The two estimators $\efima(\bm\theta)$ and $\efimb(\bm\theta)$ are not new as one usually utilizes one of them to compute the FIM.  \citet{spall}
analyzed their accuracy for univariate symmetric density functions based on the
central limit theorem.  Fisher information estimation is also examined in latent
variable models~\citep{delattre2019estimating}. Under the same topic, our work analyzes the factors affecting the variance of $\efima(\bm\theta)$
and $\efimb(\bm\theta)$ for deep neural networks.

A large body of work tries to approximate the FIM or define similar curvature
tensors for performing natural gradient
descent~\citep{aIGI,martens,martens2015optimizing,adam,ollivier,rfim}.  If the
loss is an empirical expectation of $-\log{p}(\bm{y}\,\vert\,\bm{x},\bm\theta)$,
its Hessian \wrt $\bm{h}_L$ is
exactly $\fim(\bm{h}_L)$. Then, the FIM in \cref{eq:fimexp} is in the form of a
Generalized Gauss-Newton matrix (GNN)~\citep{schraudolph2002fast,martens}.
\citet{ollivier} provided algorithm procedures to compute the unit-wise FIM and
discussed Monte Carlo natural gradient. The FIM can be computed
locally~\citep{rfim,ay} based on a joint distribution representation of the
neural network. Efficient computational methods are developed to evaluate the
FIM inverse~\citep{park}.

The estimator $\efima(\bm\theta)$ is \emph{not} the ``empirical Fisher''~(see \eg
~\cite[Section 11]{martens}) as $\bm{y}_i$ is randomly sampled from
$p(\bm{y}\,\vert\,\bm{x},\bm\theta)$, making $\efima(\bm\theta)$ an unbiased
estimator. The difference between the empirical Fisher and the FIM is
clarified~\citep{kunstner2020limitations}.  Similarly, the estimator
$\efimb(\bm\theta)$ is \emph{not} the Hessian of the loss, as $\bm{y}_i$ is randomly
sampled rather than fixed to the given target.

Recently, the structure of the FIM (or its partial approximations) are examined
in deep learning.
The FIM of randomized networks is analyzed~\citep{amari2019fisher}, where the weights of the neural network are assumed to be random. Often the analysis of randomized networks uses spectral analysis and random matrix theory~\citep{pennington2018spectrum}.
An insight from this body of work is that most of the eigenvalues of
the FIM are close to 0; while the high end of spectrum has large
values~\citep{karakida2019universal}. In this paper, the estimators
$\efima(\bm\theta)$ and $\efimb(\bm\theta)$ are random matrices due to the
sampling of $\bm{y}_i\sim{}p(\bm{y}\,\vert\,\bm{x})$ (the weights are considered
fixed).

In information geometry~\citep{aIGI,elementig}, the FIM serves as a Riemannian
metric in the space of probability distributions.  The FIM is a covariant tensor
and is invariant to diffeomorphism on the sample space~\citep{elementig}.
Higher order tensors are used to describe the intrinsic structure in this space.
For example, the Riemannian curvature is a 4D tensor, while the Ricci curvature
is 2D. The third cumulants of the sufficient statistics give an affine
connection (belonging to the $\alpha$-connections or the Amari-\v{C}ensov
tensor) of the exponential family. The FIM is generalized to a one-parameter
family~\citep{alphaFIM}.

In statistics, our estimator $\efimb(\bm\theta)$ is \citet{efron}'s ``observed
Fisher information'', which is usually evaluated at the maximum likelihood
estimation $\hat{\bm\theta}$.  Higher order moments of the maximum likelihood
estimator (MLE) were discussed (see \eg~\citet{gamma}). These moments are
associated with parameter estimators and differ from the concept of the FIM
estimators. Similar concepts are examined in higher-order asymptotic
theory~\cite[Chapter 7]{aIGI}.

\section{Conclusion}\label{sec:con}

The FIM is a covariant \psd tensor revealing the intrinsic geometric structure of the parameter manifold.
It yields useful practical methods such as the natural gradient.
In practice, the true FIM $\fim(\bm\theta)$ is usually expensive or impossible to obtain.
Estimators of the FIM based on empirical samples is used in the deep learning practice.
We analyzed two different estimators $\efima(\bm\theta)$ and $\efimb(\bm\theta)$
of the FIM of a deep neural network.
These estimators are convenient to compute using auto-differentiation frameworks but
randomly deviates from $\fim(\bm\theta)$.
Our central results, \cref{thm:varefima,thm:varefimb},
present the variance of $\efima(\bm\theta)$ and $\efimb(\bm\theta)$ in closed
form, which is further extended to upper bounds in simpler forms.
Two factors affecting the estimation variance are
\ding{192} the derivatives of neural network output $\bm{h}_L$ \wrt the weight
parameters $\bm\theta$; and
\ding{193} the property of $\bm{h}_L$ as an exponential family distribution.
A large scale of the 1st- and/or 2nd-order derivatives leads to a large variance
when estimating the FIM. Our analytical results can be useful to measure the
quality of the estimated FIM and could lead to variance reduction techniques.

\section*{Acknowledgements}
    We thank the anonymous NeurIPS reviewers for their constructive comments.
    We thank Frank Nielsen for the insightful feedback.
    We thank James C. Spall for pointing us to related work.

\nocite{weightnorm,chen2020tensor,lim2005singular}

\newpage
\appendix

\setcounter{equation}{0}
\setcounter{theorem}{0}

\renewcommand{\theequation}{A.\arabic{equation}}
\renewcommand{\thefigure}{A.\arabic{figure}}
\renewcommand{\thetable}{A.\arabic{table}}
\renewcommand{\thetheorem}{A.\arabic{theorem}}

\addcontentsline{toc}{section}{Appendix} \part{Appendix}

This appendix contains a proofs of the results in the main text and further analysis on the two FIM estimators \( \efima(\bm \theta) \) and \( \efimb(\bm \theta) \). In particular, \cref{sup:changecoordinates} presents an analysis of how the FIM estimators and their covariance tensors change under reparametrization. \Cref{sup:elementwise} presents element-wise bound alternatives to those presented in \cref{subsec:variance_bounds}. \Cref{sup:alternativenorm} explores various results using alternative norms to the Frobenius norm results of the main text. \Cref{sup:combination} presents an analysis on taking a linear combination of the two FIM estimators. \Cref{sup:experimental} presents a numerical experiments of the FIM estimators on the MNIST dataset. \parttoc

\section{List of Symbols}

\begin{table}[H]
    \caption{Table of symbols.}
    \begin{tabularx}{\textwidth}{LND}
        \toprule
        \bf Symbol & \bf Meaning & \bf Defined \\
        \midrule
        \( \fim(\bm \theta) \) & Fisher information matrix (FIM) & \cref{eq:fim} \\
        \( \efima(\bm \theta) \) & An estimator of the FIM & \cref{eq:estimators} \\
        \( \efimb(\bm \theta) \) & Another estimator of the FIM & \\
        \( \var(\cdot) \) & Element-wise variance of input; & \\
                          & output is the same dimension as the input dimension & \cref{eq:varcov} \\
        \( \cov(\cdot) \) & Pair-wise covariance of input; \\
                          &output is the dimension of an outer product on the input & \cref{thm:exp} \\
        \( \bm h \) & Natural parameter of exponential family & \cref{tab:exp_fam} \\
        \( \bm h_{l} \) & Hidden layer output in our neural network model & \cref{eq:exp} \\
        \( \bm h_{L} \) & Last layer's output in our neural network model & \\
                        & and the natural parameter of the exponential family & \cref{eq:exp} \\
        \( n_{l} \) & Size of layer \( l \) & \cref{eq:exp} \\
        \( F(\bm h_{L}) \) & Log-partition function of exponential family & \cref{thm:exp} \\
        \( \bm \eta = \bm \eta(\bm h_{L}) \) & Dual parameterization of exponential family & \cref{thm:exp} \\
        \( \fim(\bm h_{L}) = \var(\bm t) \) & Covariance matrix of \( \bm t \) & \cref{thm:exp} \\
        \( \kurt(\bm t) \) & 4th central moment of \( \bm t \) & \cref{thm:varefima} \\
        \( \kappa_{a,b,c,d} \) & 4th order cumulant of exponential family & \cref{lem:kurt} \\
        \( \lambda_{\min}(M) \) & Smallest eigenvalue of \( M \) & \cref{thm:est2_psd} \\
        \( \rho(M) \) & Spectral radius of \( M \) (absolute value of spectrum) & \cref{thm:est2_psd} \\
        \( \partial \ell; \partial^{2} \ell \) & Partial derivatives of likelihood \wrt weights & \cref{eq:estimators} \\
        \( \partial \bm h_{L}; \partial^{2} \bm h_{L} \) & Partial derivatives of neural network \wrt weights & \cref{eq:gradl} \\
        \( \Vert \cdot \Vert_{F} \) & Frobenius norm / \( L_{2} \)-norm & \cref{prop:unbiased_consistent} \\
        \( \Vert \cdot \Vert_{2} \) & 2-norm / \( L_{2} \)-norm & \cref{prop:unbiased_consistent} \\
        \( \Vert \cdot \Vert_{2_\sigma} \) & Tensor spectral norm & \cref{lem:nn_grad_norm_bounded2} \\
        \( \Vert \cdot \Vert_{1} \) & \( L_{1} \)-norm & \cref{thm:varefimboth2} \\
        \( \Vert \cdot \Vert_{\infty} \) & \( L_{\infty} \)-norm & \cref{thm:varefimboth2} \\
        \( M^{\T} \) & Matrix transpose & \cref{eq:exp} \\
        \( \otimes \) & Tensor product & \cref{thm:varefima1} \\
\bottomrule
    \end{tabularx}
\end{table}

\section{Variance of FIM estimators}

\subsection{Proof of \cref{prop:unbiased_consistent}}

\begin{proof}
    The first statement holds as both \( \efima(\bm \theta) \) and \( \efimb(\bm \theta) \) are point-wise estimators, they are unbiased (central limit theorem).

    The second statement holds by the law of large numbers (and triangle inequality with \( \varepsilon/2\) and a union bound).
\end{proof}

\subsection{Proof of \cref{thm:exp}}

\begin{proof}
    The statement follows as \( p(\bm y \mid \bm x, \bm \theta) \) is given by an exponential family. See \cite{aIGI}.
\end{proof}

\subsection{Proof of \cref{thm:fim2}}

\begin{proof}
    Consider the alternative formulation of the FIM.
    \begin{align*}
        \expect_{x, y} \left[ -\frac{\partial^{2}}{\partial \bm\theta \partial \bm\theta^{\T}} \log p(\bm y \mid \bm x) \right]
        &=  \expect_{x, y} \left[ \frac{\partial \ell}{\partial \bm \theta} \frac{\partial \ell}{\partial \bm \theta^{\T}} \right] - \expect_{x, y} \left[ \frac{\frac{\partial^{2}}{\partial \bm\theta \partial \bm\theta^{\T}} p(\bm y \mid \bm x)}{p(\bm y \mid \bm x)}\right] \\
        &= \fim{}(\bm\theta) - \int p(\bm x) \frac{\partial^{2}}{\partial \bm\theta \partial \bm\theta^{\T}} p(\bm y \mid \bm x) \dy \dx.
    \end{align*}
    Thus for this to be equivalent to the Jacobian definition, we need the residual term to be zero.

    As \( \sigma \in C^2(\Re) \) and thus \( \sigma^{\prime} \) is smooth, it follows by the composition of smooth functions that \( p(\bm y \mid \bm x) \) and \( \partial p(\bm y \mid \bm x) \) is also a smooth function. This provides sufficient conditions for the Leibniz integration rule to be used (switch order of integration and differentiation). As such,
\begin{align*}
        \expect_{x, y} \left[ -\frac{\partial^{2}}{\partial \bm\theta \partial \bm\theta^{\T}} \log p(\bm y \mid \bm x) \right]
        &= \fim{}(\bm\theta) - \int p(\bm x) \frac{\partial^{2}}{\partial \bm\theta \partial \bm\theta^{\T}} p(\bm y \mid \bm x) \dy \dx \\
        &= \fim{}(\bm\theta) - \frac{\partial^{2}}{\partial \bm\theta \partial \bm\theta^{\T}}\int p(\bm x) p(\bm y \mid \bm x) \dy \dx \\
        &= \fim{}(\bm\theta).
    \end{align*}
\end{proof}

\subsection{Proof of \cref{thm:varefima}}
\begin{proof}
    First define \( \bm \delta \defeq \bm \delta(\bm x, \bm y; \bm \theta) = {\bm t}({\bm y}) - {\bm \eta}({\bm h}_{L}({\bm x})) \).
    \begin{align*}
        \cov
        \left(
        \partial_{i} \ell \cdot \partial_{j} \ell
        ,
        \partial_{k} \ell \cdot \partial_{l} \ell
        \right)
        &=
        \expect_{\bm y \mid \bm x; \bm \theta}[
            (\partial_{i} \ell \cdot \partial_{j} \ell)
            (\partial_{k} \ell \cdot \partial_{l} \ell)
        ]
        -
        \expect_{\bm y \mid \bm x; \bm \theta}[
            \partial_{i} \ell \cdot \partial_{j} \ell
        ]
        \cdot
        \expect_{\bm y \mid \bm x; \bm \theta}[
            \partial_{k} \ell \cdot \partial_{l} \ell
        ].
    \end{align*}
    We then calculate each components. For the first:
\begin{align*}
        \expect_{\bm y \mid \bm x; \bm \theta}[
            (\partial_{i} \ell \cdot \partial_{j} \ell)
            (\partial_{k} \ell \cdot \partial_{l} \ell)
        ]
        &=
        \expect_{\bm y \mid \bm x; \bm \theta} \left[ \partial_{i} {\bm h}_{L}^{a}(\bm x) \cdot \partial_{j} {\bm h}_{L}^{b}(\bm x) \cdot \partial_{k} {\bm h}_{L}^{a}(\bm x) \cdot \partial_{l} {\bm h}_{L}^{b}(\bm x) \cdot \delta_{a} \cdot \delta_{b} \cdot \delta_{c} \cdot \delta_{d} \right] \\
        &=
        \partial_{i} {\bm h}_{L}^{a}(\bm x) \cdot \partial_{j} {\bm h}_{L}^{b}(\bm x) \cdot \partial_{k} {\bm h}_{L}^{a}(\bm x) \cdot \partial_{l} {\bm h}_{L}^{b}(\bm x) \cdot \expect_{\bm y \mid \bm x; \bm \theta} \left[ \delta_{a} \cdot \delta_{b} \cdot \delta_{c} \cdot \delta_{d} \right].
    \end{align*}

    For the second term, we can first consider the expectation:
\begin{align*}
        &\expect_{\bm y \mid \bm x; \bm \theta}[ \partial_{i} \ell \cdot \partial_{j} \ell ] \\
        &=
        \expect_{\bm y \mid \bm x; \bm \theta}\left[ \partial_{i} {\bm h}_{L}^{a}(\bm x) \cdot \partial_{j} {\bm h}_{L}^{b}(\bm x) \cdot \delta_{a} \cdot \delta_{b} \right] \\
        &=
        \partial_{i} {\bm h}_{L}^{a}(\bm x) \cdot \partial_{j} {\bm h}_{L}^{b}(\bm x) \cdot \expect_{\bm y \mid \bm x; \bm \theta}\left[ \delta_{a} \cdot \delta_{b} \right].
    \end{align*}
    Thus the product of this gives:
    \begin{align*}
        &\expect_{\bm y \mid \bm x; \bm \theta}[
            \partial_{i} \ell \cdot \partial_{j} \ell
        ]
        \cdot
        \expect_{\bm y \mid \bm x; \bm \theta}[
            \partial_{k} \ell \cdot \partial_{l} \ell
        ] \\
        &\quad=
        \partial_{i} {\bm h}_{L}^{a}(\bm x) \cdot \partial_{j} {\bm h}_{L}^{b}(\bm x) \cdot \partial_{k} {\bm h}_{L}^{c}(\bm x) \cdot \partial_{l} {\bm h}_{L}^{d}(\bm x) \cdot \expect_{\bm y \mid \bm x; \bm \theta}\left[ \delta_{a} \cdot \delta_{b} \right] \cdot \expect_{\bm y \mid \bm x; \bm \theta}\left[ \delta_{c} \cdot \delta_{d} \right].
    \end{align*}

    This gives the total covariance:
\begin{align*}
        &\cov
        \left(
        \partial_{i} \ell \cdot \partial_{j} \ell
        ,
        \partial_{k} \ell \cdot \partial_{l} \ell
        \right)
        = \partial_{i} {\bm h}_{L}^{a}(\bm x) \cdot \partial_{j} {\bm h}_{L}^{b}(\bm x) \cdot \partial_{k} {\bm h}_{L}^{c}(\bm x) \cdot \partial_{l} {\bm h}_{L}^{d}(\bm x) \cdot
        \\
        &\quad \quad \quad \quad \left(
        \expect_{\bm y \mid \bm x; \bm \theta} \left[ \delta_{a} \cdot \delta_{b} \cdot \delta_{c} \cdot \delta_{d} \right]
        -
        \expect_{\bm y \mid \bm x; \bm \theta}\left[ \delta_{a} \cdot \delta_{b} \right] \cdot \expect_{\bm y \mid \bm x; \bm \theta}\left[ \delta_{c} \cdot \delta_{d} \right]
        \right). \end{align*}
\end{proof}

\subsection{Proof of \cref{lem:kurt}}
\begin{proof}
    We first consider the following notation, consistent with \citep{mccullagh2018tensor} except as subscripts, to denote the central moments and different multivariate cumulants.
    The \emph{non-central moment} of a vector random variable \( \bm X \) is denoted as
\begin{equation*}
\kappa_{r_{1} \ldots r_{m}} = \expect[\bm X_{r_{1}} \cdots \bm X_{r_{m}}],
    \end{equation*}
    where \( \bm r \) is an integer vector denoting the dimensions in which the non-central moment is taken as. The dimensions can be repeated, \ie, \( r_{1} = r_{2} \) \etc. Note that in the notation \( r_{1} \ldots r_{m} \) are not comma separated.

    Similarly, for \emph{central moments} we have:
\begin{equation*}
        \kurt_{r_{1} \ldots r_{m}} = \expect[(\bm X_{r_{1}} - \expect[\bm X_{r_{1}}]) \cdots (\bm X_{r_{m}} - \expect[\bm X_{r_{m}}])].
    \end{equation*}

    For cumulants, we use a comma separated subscripts to distinguish it from central moments. In our case (\( X \) is from an exponential family), this just reduces to \( m \)-derivative of the log-partition function:
\begin{equation*}
        \kappa_{r_{1}, \ldots, r_{m}} = \frac{\partial^{m} F(\bm h_{L})}{\partial \bm h_{r_{1}} \cdots \partial \bm h_{r_{m}}}
    \end{equation*}

    In addition to these moment-like quantities, we introduce the \( [n] \) notation from \citep{mccullagh2018tensor} to denote the possible permutations in indices. For example:
\begin{equation*}
        \kappa_{i} \kappa_{j,k}[3] = \kappa_{i} \kappa_{j,k} + \kappa_{j} \kappa_{i,k} + \kappa_{k} \kappa_{i, j},
    \end{equation*}
    noting that the (central and non-central) moments and cumulants are index order invariant. The numbering also must be equal to the number of available permutations. Also note that \( \kappa_{i} \) denotes a cumulant and not the moment.

    \citep{mccullagh2018tensor} \emph{only} provides a formulation from cumulants to non-central moments. Thus we need to rewrite \( \kurt_{abcd} \) in terms of non-central moments. The following is given in \citep[Equation (2.6)]{mccullagh2018tensor}:
\begin{align*}
        \kappa_{ij} &= \kappa_{i,j} + \kappa_{i} \kappa_{j}; \\
        \kappa_{ijk} &= \kappa_{i,j,k} + \kappa_{i} \kappa_{j,k}[3] + \kappa_{i} \kappa_{j} \kappa_{k}; \\
        \kappa_{ijkl} &= \kappa_{i,j,k,l} + \kappa_{i} \kappa_{j,k,l} [4] + \kappa_{i, j} \kappa_{k, l} [3] + \kappa_{i} \kappa_{j} \kappa_{k,l} [6] + \kappa_{i} \kappa_{j} \kappa_{k} \kappa_{l}.
    \end{align*}

    Expanding \( \kurt_{abcd} \), the 4th central moment, yields the following:
\begin{align*}
        \kurt_{abcd}
        &= \kappa_{abcd} - \kappa_{a} \kappa_{bcd} [4] + \kappa_{a}\kappa_{b}\kappa_{cd}[6] - 3 \kappa_{a}\kappa_{b}\kappa_{c}\kappa_{d} \\
        &= \kappa_{abcd} - \kappa_{a} \kappa_{bcd} [4] + (\kappa_{a}\kappa_{b} (\kappa_{c,d} + \kappa_{c} \kappa_{d}))[6] - 3 \kappa_{a}\kappa_{b}\kappa_{c}\kappa_{d} \\
        &= \kappa_{abcd} - \kappa_{a} \kappa_{bcd} [4] + \kappa_{a}\kappa_{b} \kappa_{c,d}[6] + 3 \kappa_{a}\kappa_{b}\kappa_{c}\kappa_{d} \\
        &= \kappa_{abcd} - (\kappa_{a} (\kappa_{b,c,d} + \kappa_{a} \kappa_{c,d}[3] + \kappa_{b} \kappa_{c} \kappa_{d})) [4] + \kappa_{a}\kappa_{b} \kappa_{c,d}[6] + 3 \kappa_{a}\kappa_{b}\kappa_{c}\kappa_{d} \\
        &= \kappa_{abcd}
        - \kappa_{a} \kappa_{b,c,d} [4]
        - \kappa_{a} \kappa_{b} \kappa_{c,d}[6]
- \kappa_{a}\kappa_{b}\kappa_{c}\kappa_{d} \\
        &= \kappa_{a,b,c,d} + \kappa_{a,b} \kappa_{c,d}[3],
    \end{align*}
    which when substituting for derivatives proves the Lemma.
\end{proof}

\subsection{Proof of \cref{thm:varefimb}}
\begin{proof}
    \begin{align*}
        &\cov \left(
            -\frac{\partial^{2}\ell}{\partial\bm\theta_{i} \partial\bm\theta_{j}}
            ,
            -\frac{\partial^{2}\ell}{\partial\bm\theta_{k} \partial\bm\theta_{l}}
            \right) \\
        &= \cov \left(
            \partial_{i} {\bm h}^{a}_{L}({\bm x})
            \fim_{ab}(\bm h_{L})
            \partial_{j} {\bm h}^{b}_{L}({\bm x})
            -
            [{\bm t}_{\alpha}({\bm y}) - {\bm \eta}_{\alpha}]
            \frac{\partial^{2}{\bm h}^{\alpha}_{L}({\bm x})}{\partial\bm\theta_{i} \partial\bm\theta_{j}}
            , \right. \\
            & \left. \quad \quad \quad \quad \quad \quad \quad \quad
            \partial_{k} {\bm h}^{a}_{L}({\bm x})
            \fim_{ab}(\bm h_{L})
            \partial_{l} {\bm h}^{b}_{L}({\bm x})
            -
            [{\bm t}_{\beta}({\bm y}) - {\bm \eta}_{\beta}]
            \frac{\partial^{2}{\bm h}^{\beta}_{L}({\bm x})}{\partial\bm\theta_{k} \partial\bm\theta_{l}}
         \right) \\
        &= \cov \left(
            [{\bm t}_{\alpha}({\bm y}) - {\bm \eta}_{\alpha}]
            \frac{\partial^{2}{\bm h}^{\alpha}_{L}({\bm x})}{\partial\bm\theta_{i} \partial\bm\theta_{j}}
            ,
            [{\bm t}_{\beta}({\bm y}) - {\bm \eta}_{\beta}]
            \frac{\partial^{2}{\bm h}^{\beta}_{L}({\bm x})}{\partial\bm\theta_{k} \partial\bm\theta_{l}}
         \right) \\
        &= E \left[
            [{\bm t}_{\alpha}({\bm y}) - {\bm \eta}_{\alpha}]
            \frac{\partial^{2}{\bm h}^{\alpha}_{L}({\bm x})}{\partial\bm\theta_{i} \partial\bm\theta_{j}}
            \cdot
            [{\bm t}_{\beta}({\bm y}) - {\bm \eta}_{\beta}]
            \frac{\partial^{2}{\bm h}^{\beta}_{L}({\bm x})}{\partial\bm\theta_{k} \partial\bm\theta_{l}}
         \right] \\
        & \quad \quad \quad \quad \quad \quad \quad \quad
         -
        E \left[
            [{\bm t}_{\alpha}({\bm y}) - {\bm \eta}_{\alpha}]
            \frac{\partial^{2}{\bm h}^{\alpha}_{L}({\bm x})}{\partial\bm\theta_{i} \partial\bm\theta_{j}}
            \right]
            \cdot
        E \left[
            [{\bm t}_{\beta}({\bm y}) - {\bm \eta}_{\beta}]
            \frac{\partial^{2}{\bm h}^{\beta}_{L}({\bm x})}{\partial\bm\theta_{k} \partial\bm\theta_{l}}
         \right] \\
        &= E \left[
            [{\bm t}_{\alpha}({\bm y}) - {\bm \eta}_{\alpha}]
            \frac{\partial^{2}{\bm h}^{\alpha}_{L}({\bm x})}{\partial\bm\theta_{i} \partial\bm\theta_{j}}
            \cdot
            [{\bm t}_{\beta}({\bm y}) - {\bm \eta}_{\beta}]
            \frac{\partial^{2}{\bm h}^{\beta}_{L}({\bm x})}{\partial\bm\theta_{k} \partial\bm\theta_{l}}
         \right] \\
         &=
            \frac{\partial^{2}{\bm h}^{\alpha}_{L}({\bm x})}{\partial\bm\theta_{i} \partial\bm\theta_{j}}
            \frac{\partial^{2}{\bm h}^{\beta}_{L}({\bm x})}{\partial\bm\theta_{k} \partial\bm\theta_{l}}
            E [\delta_{\alpha} \cdot \delta_{\beta}].
    \end{align*}
    The theorem follows immediately.
\end{proof}

\subsection{Proof of \cref{thm:varfima11}}

\begin{proof}
    Let \( 1 \leq p, q \leq \infty \) such that \( 1 / p  + 1 / q = 1 \).
    Let \( \mathcal{T}_{abcd} = \kurt_{abcd}(\bm{t}) - \fim_{ab}(\bm{h}_L) \cdot \fim_{cd}(\bm{h}_L) \).
    From the last inequality in the proof of \cref{thm:varefima1} we have,
\begin{align*}
        \left \vert
        \left[
        \cov
        \left(
            \efima(\bm \theta)
        \right)
        \right]^{ijkl} \right \vert
        &\leq
        \frac{1}{N}
        \cdot
        \Vert \partial_{i} {\bm h}_{L}(\bm x) \Vert_{p}
        \cdot
        \Vert \partial_{j} {\bm h}_{L}(\bm x) \Vert_{p}
        \cdot
        \Vert \partial_{k} {\bm h}_{L}(\bm x) \Vert_{p}
        \cdot
        \Vert \partial_{l} {\bm h}_{L}(\bm x) \Vert_{p}
        \cdot
        \Vert \mathcal{T} \Vert_q.
    \end{align*}

    Thus for the \( p \)-norm,
\begin{align*}
        &\left \Vert
        \cov
        \left(
            \efima(\bm \theta)
        \right)
        \right \Vert_{p} \\
        &=
        \left(
            \sum_{i,j,k,l}
            \left \vert
            \left[
            \cov
            \left(
                \efima(\bm \theta)
            \right)
            \right]^{ijkl} \right \vert^{p}
        \right)^{1/p} \\
        &\leq
        \frac{1}{N}
        \cdot
        \Vert \mathcal{T} \Vert_{q}
        \left(
            \sum_{i,j,k,l}
            \left \vert
            \Vert \partial_{i} {\bm h}_{L}(\bm x) \Vert_{p}
            \cdot
            \Vert \partial_{j} {\bm h}_{L}(\bm x) \Vert_{p}
            \cdot
            \Vert \partial_{k} {\bm h}_{L}(\bm x) \Vert_{p}
            \cdot
            \Vert \partial_{l} {\bm h}_{L}(\bm x) \Vert_{p}
            \right \vert^{p}
        \right)^{1/p} \\
        &=
        \frac{1}{N}
        \cdot
        \Vert \mathcal{T} \Vert_{q}
        \left(
            \left(
            \sum_{i}
            \Vert \partial_{i} {\bm h}_{L}(\bm x) \Vert_{p}^{p}
            \right)^{4}
        \right)^{1/p} \\
        &=
        \frac{1}{N}
        \cdot
        \Vert \mathcal{T} \Vert_{q}
        \left(
            \left(
            \sum_{i}
            \Vert \partial_{i} {\bm h}_{L}(\bm x) \Vert_{p}^{p}
            \right)^{1/p}
        \right)^{4} \\
        &=
        \frac{1}{N}
        \cdot
        \Vert \mathcal{T} \Vert_{q}
        \cdot
        \Vert \partial {\bm h}_{L}(\bm x) \Vert_{p}^{4}.
    \end{align*}
    Thus, by taking the \( p = q = 2 \) the Theorem holds.
\end{proof}

\subsection{Proof of \cref{thm:varfimb11}}

\begin{proof}
    Let \( 1 \leq p, q \leq \infty \) such that \( 1 / p  + 1 / q = 1 \).
    From the last inequality in the proof of \cref{thm:varefimb1} we have,
\begin{align*}
        \left \vert
        \left[
        \cov
        \left(
            \efimb(\bm \theta)
        \right)
        \right]^{ijkl} \right \vert
        &\leq
        \frac{1}{N}
        \cdot
        \Vert
        \partial^{2}_{ij}{\bm h}_{L}({\bm x})
        \Vert_{p}
        \cdot
        \Vert
        \partial^{2}_{kl}{\bm h}_{L}({\bm x})
        \Vert_{p}
        \cdot
        \Vert
        \fim( \bm{h}_L )
        \Vert_{q}
    \end{align*}
    Thus for the \( p \)-norm,
\begin{align*}
        \left \Vert
        \cov
        \left(
            \efimb(\bm \theta)
        \right)
        \right \Vert_{q}
        &= \left( \sum_{i,j,k,l}
        \left \vert
        \left[
        \cov
        \left(
            \efimb(\bm \theta)
        \right)
        \right]^{ijkl} \right \vert^{p} \right)^{1/p}\\
        &\leq
        \frac{1}{N}
        \cdot
        \Vert
        \fim( \bm{h}_L )
        \Vert_{q}
        \cdot
        \left( \sum_{i,j,k,l}
        \left \vert
        \Vert
        \partial^{2}_{ij}{\bm h}_{L}({\bm x})
        \Vert_{p}
        \cdot
        \Vert
        \partial^{2}_{kl}{\bm h}_{L}({\bm x})
        \Vert_{p}
        \right \vert^{p} \right)^{1/p}\\
        &=
        \frac{1}{N}
        \cdot
        \Vert
        \fim( \bm{h}_L )
        \Vert_{q}
        \cdot
        \left(
        \left(
        \sum_{i,j}
        \left \vert
        \Vert
        \partial^{2}_{ij}{\bm h}_{L}({\bm x})
        \Vert_{p}
        \right \vert^{p}
        \right)^{2}
        \right)^{1/p}\\
        &=
        \frac{1}{N}
        \cdot
        \Vert
        \fim( \bm{h}_L )
        \Vert_{q}
        \cdot
        \left(
        \left(
        \sum_{i,j}
        \left \vert
        \Vert
        \partial^{2}_{ij}{\bm h}_{L}({\bm x})
        \Vert_{p}
        \right \vert^{p}
        \right)^{1/p}
        \right)^{2}\\
        &=
        \frac{1}{N}
        \cdot
        \Vert
        \fim( \bm{h}_L )
        \Vert_{q}
        \cdot
        \Vert
        \partial^{2}{\bm h}_{L}({\bm x})
        \Vert_{p}^{2}.
    \end{align*}
    Thus, by taking the \( p = q = 2 \) the Theorem holds.
\end{proof}

\subsection{Proof of \cref{thm:elementwisemoment}}

\begin{proof}
    We first prove the second part of the Lemma, which is simpler.
    \begin{align}
        \Vert \fim(\bm{h}_L) \Vert_{F}
        &= \sqrt{\sum_{a,b} \cov^2(\bm{t}_a,\bm{t}_b)} \nonumber\\
        &\le \sqrt{\sum_{a,b} \var(\bm{t}_a) \var(\bm{t}_b) } \quad\text{(by Cauchy-Schwarz inequality)}\nonumber\\
        &= \sqrt{\sum_{a} \var(\bm{t}_a) \sum_b \var(\bm{t}_b) } \nonumber\\
        &= \sum_{a} \var(\bm{t}_a)\nonumber\\
        &= \sum_{a} \fim_{aa}(\bm{h}_L).
    \end{align}

    We are left with the first part of the Lemma.
    \begin{align*}
        &\Vert \kurt(\bm{t}) - \fim(\bm{h}_L) \otimes \fim(\bm{h}_L)
        \Vert_{F}^2 \\
        = &
        \sum_{a,b,c,d}
        \bigg(
        \expect\left((\bm{t}_a-\bm\eta_a)(\bm{t}_b-\bm\eta_b)(\bm{t}_c-\bm\eta_c)(\bm{t}_d-\bm\eta_d)\right)\\
       &\quad
          -
        \expect\left((\bm{t}_a-\bm\eta_a)(\bm{t}_b-\bm\eta_b)\right)
        \expect\left((\bm{t}_c-\bm\eta_c)(\bm{t}_d-\bm\eta_d)\right)\bigg)^2\\
       \le &
        2
        \sum_{a,b,c,d}
        \bigg(
        \underbrace{\expect^2\left((\bm{t}_a-\bm\eta_a)(\bm{t}_b-\bm\eta_b)(\bm{t}_c-\bm\eta_c)(\bm{t}_d-\bm\eta_d)\right)}_{S_{abcd}}\\
           &\quad
           +
           \underbrace{ \expect^2\left((\bm{t}_a-\bm\eta_a)(\bm{t}_b-\bm\eta_b)\right)
                        \expect^2\left((\bm{t}_c-\bm\eta_c)(\bm{t}_d-\bm\eta_d)\right)}_{T_{abcd}} \bigg)\\
       = &
       2 \sum_{a,b,c,d} ( S_{abcd} + T_{abcd} ).
    \end{align*}
    We have
    \begin{align*}
        S_{abcd}
        &= \expect^2\left((\bm{t}_a-\bm\eta_a)(\bm{t}_b-\bm\eta_b)(\bm{t}_c-\bm\eta_c)(\bm{t}_d-\bm\eta_d)\right) \\
        &\le
            \expect\left((\bm{t}_a-\bm\eta_a)^2(\bm{t}_b-\bm\eta_b)^2 \right)
            \expect\left((\bm{t}_c-\bm\eta_c)^2(\bm{t}_d-\bm\eta_d)^2 \right) \\
        &\le
            \expect^{1/2}(\bm{t}_a-\bm\eta_a)^4 \cdot
            \expect^{1/2}(\bm{t}_b-\bm\eta_b)^4 \cdot
            \expect^{1/2}(\bm{t}_c-\bm\eta_c)^4 \cdot
            \expect^{1/2}(\bm{t}_d-\bm\eta_d)^4 \\
        & =
            \sqrt{ \kurt_{aaaa}(\bm{t}) } \cdot
            \sqrt{ \kurt_{bbbb}(\bm{t}) } \cdot
            \sqrt{ \kurt_{cccc}(\bm{t}) } \cdot
            \sqrt{ \kurt_{dddd}(\bm{t}) }.
    \end{align*}
    At the same time,
    \begin{align*}
        T_{abcd}
        & = \expect^2\left((\bm{t}_a-\bm\eta_a)(\bm{t}_b-\bm\eta_b)\right)
            \expect^2\left((\bm{t}_c-\bm\eta_c)(\bm{t}_d-\bm\eta_d)\right) \\
        & \le
            \expect(\bm{t}_a-\bm\eta_a)^2 \cdot
            \expect(\bm{t}_b-\bm\eta_b)^2 \cdot
            \expect(\bm{t}_c-\bm\eta_c)^2 \cdot
            \expect(\bm{t}_d-\bm\eta_d)^2 \\
        & =
            \fim_{aa}(\bm{h}_L) \cdot
            \fim_{bb}(\bm{h}_L) \cdot
            \fim_{cc}(\bm{h}_L) \cdot
            \fim_{dd}(\bm{h}_L).
    \end{align*}
    To sum up, we have
    \begin{align*}
        \Vert \kurt(\bm{t}) - \fim(\bm{h}_L) \otimes \fim(\bm{h}_L) \Vert_{F}^2
        &\le
            2 \sum_{a,b,c,d} (S_{abcd} + T_{abcd}) \\
        &=
            2 \left( \sum_{a} \sqrt{ \kurt_{aaaa}(\bm{t}) }\right)^4
          + 2 \left( \sum_{a} \fim_{aa}(\bm{h}_L) \right)^4 \nonumber\\
        &
        \le
        2
        \left( \sum_{a} ( \sqrt{ \kurt_{aaaa}(\bm{t}) } + \fim_{aa}(\bm{h}_L) ) \right)^4.
    \end{align*}
    Taking the square root of both sides gives the result.
\end{proof}

\subsection{Proof of \cref{thm:est2_psd}}

\begin{proof}
The estimator is a \psd matrix subtracted by a linear combination of symmetric matrices.

    \begin{align*}
        \efimb(\bm x; \bm \theta)
        &=
        \fim(\bm \theta)
        - \frac{1}{N} \sum_{i=1}^{N}
        [{\bm t}({\bm y}_{i}) - {\bm \eta}({\bm h}_{L}({\bm x}))]^{\T}
        \frac{\partial^{2}{\bm h}_{L}({\bm x})}{\partial\bm\theta \partial\bm\theta^{\T}} \\
        &=
        \fim(\bm \theta)
        -
        \sum_{a=1}^{n}
        (\bar{\bm t}_{a} - \expect[\bar{\bm t}_{a}])
        \frac{\partial^{2}{\bm h}^{a}_{L}({\bm x})}{\partial\bm\theta \partial\bm\theta^{\T}},
    \end{align*}
    where \( \bar{\bm t}_{a} = \frac{1}{N} \sum_{i=1}^{N} \bm t_{a}(\bm y_{i}) \).

    We consider the Chebyshev inequalities for
    \begin{equation*}
    k = \frac{\sqrt{N} \lambda_{\min}(\fim(\bm \theta))}
    {\Vert\bm\rho\Vert_2 \sqrt{\lambda_{\max}(\fim(\bm{h}_L))}} > 0.
    \end{equation*}
\begin{align*}
        & \prob\left( \Vert \bar{\bm t} - \expect[\bar{\bm t}] \Vert_2
        \leq k \sqrt{\frac{1}{N}\lambda_{\max}(\fim(\bm{h}_L))} \right)\\
        = &
        \prob\left( \frac{\Vert \bar{\bm t} - \expect[\bar{\bm t}] \Vert_2^2}{
                  \frac{1}{N}\lambda_{\max}(\fim(\bm{h}_L)) }
        \leq k^2 \right)\\
        \geq &
        \prob\left(
            (\bar{\bm t} - \expect[\bar{\bm t}])^\T
            \left(\frac{1}{N} \fim(\bm{h}_L) \right)^{-1}
            (\bar{\bm t} - \expect[\bar{\bm t}])
        \leq k^2 \right)\\
        \geq &
        1 - \frac{n_L}{k^{2}}\\
        = &
        1 - \frac{n_L \Vert\bm\rho\Vert_2^2 \lambda_{\max}(\fim(\bm{h}_L))}
        {N \lambda_{\min}^2(\fim(\bm \theta))}.
    \end{align*}

    Thus with probability at least
\begin{align*}
        1 - \frac{n_L \Vert\bm\rho\Vert_2^2 \lambda_{\max}(\fim(\bm{h}_L))}
        {N \lambda_{\min}^2(\fim(\bm \theta))},
    \end{align*}
    the following statement is true: for all unit vector \( \bm v \) we have
\begin{align*}
        \bm v^{\T} \efimb(\bm \theta) \bm v
        &=
        \bm v^{\T}
        \left(
        \fim(\bm \theta)
        -
        \sum_{a=1}^{n}
        (\bar{\bm t}_{a} - \expect[\bar{\bm t}_{a}])
        \frac{\partial^{2}{\bm h}^{a}_{L}({\bm x})}{\partial\bm\theta \partial\bm\theta^{\T}}
        \right)
        \bm v \\
        &=
        \bm v^{\T}
        \fim(\bm \theta)
        \bm v
        -
        \sum_{a=1}^{n}
        (\bar{\bm t}_{a} - \expect[\bar{\bm t}_{a}])
        \left(
        \bm v^{\T}
        \frac{\partial^{2}{\bm h}^{a}_{L}({\bm x})}{\partial\bm\theta \partial\bm\theta^{\T}}
        \bm v
        \right)
        \\
        &\geq
        \lambda_{\min}(\fim(\bm \theta))
        -
        \sum_{a=1}^{n}
        (\bar{\bm t}_{a} - \expect[\bar{\bm t}_{a}])
        \left(
        \bm v^{\T}
        \frac{\partial^{2}{\bm h}^{a}_{L}({\bm x})}{\partial\bm\theta \partial\bm\theta^{\T}}
        \bm v
        \right)
        \\
        &\geq
        \lambda_{\min}(\fim(\bm \theta))
        -
        \sum_{a=1}^{n}
        \vert \bar{\bm t}_{a} - \expect[\bar{\bm t}_{a}] \vert
        \cdot
        \rho( \partial^{2} \bm h^{a}_{L} )
        \\
        &\geq
        \lambda_{\min}(\fim(\bm \theta))
        -
        \left\Vert \bar{\bm t} - \expect[\bar{\bm t}] \right\Vert_2
        \cdot
        \Vert \bm{\rho} \Vert_2
        \\
        &\geq
        \lambda_{\min}(\fim(\bm \theta))
        -
        k \cdot \sqrt{\frac{1}{N}\lambda_{\max}(\fim(\bm{h})_L)} \cdot
        \Vert \bm{\rho} \Vert_2
        \\
        &\geq
        \lambda_{\min}(\fim(\bm \theta))
        -
        \lambda_{\min}(\fim(\bm \theta))
        \\
        &= 0.
    \end{align*}

    Thus with the specified probability, estimator \( \efimb(\bm \theta) \) is a positive semidefinite matrix.
\end{proof}

\subsection{Proof of \cref{thm:worstspectrum}}

\begin{proof}
    Let \( \bar{\bm t}_{a} = \frac{1}{N} \sum_{i=1}^{N} \bm t_{a}(\bm y_{i}) \).
    As the FIM is a \psd matrix, $\forall\bm{v}$ such that $\Vert{}\bm{v}\Vert=1$, we have
\begin{align*}
        \bm v^{\T} \efimb(\bm \theta) \bm v
        &=
        \bm v^{\T}
        \left(
        \fim(\bm \theta)
        -
        (\bar{\bm t}_{a} - \expect[\bar{\bm t}_{a}])
        \frac{\partial^{2}{\bm h}^{a}_{L}({\bm x})}{\partial\bm\theta \partial\bm\theta^{\T}}
        \right)
        \bm v \\
        &\geq
        -
        \bm v^{\T}
        \left(
        (\bar{\bm t}_{a} - \expect[\bar{\bm t}_{a}])
        \frac{\partial^{2}{\bm h}^{a}_{L}({\bm x})}{\partial\bm\theta \partial\bm\theta^{\T}}
        \right)
        \bm v \\
        &=
        (\bm \eta_{a} - \bar{\bm t}_{a})
        \cdot
        \bm v^{\T}
        \frac{\partial^{2}{\bm h}^{a}_{L}({\bm x})}{\partial\bm\theta \partial\bm\theta^{\T}}
        \bm v \\
        &\geq
        -\left \vert \bm \eta_{a} - \bar{\bm t}_{a} \right \vert
        \cdot
        \rho(\partial^{2}\bm h_{L}^{a}) \cdot
        \bm v^{\T}
        \bm v \\
       &= -
        \rho(\partial^{2}\bm h_{L}^{a})
        \left \vert \bm \eta_{a} - \bar{\bm t}_{a} \right \vert.
    \end{align*}

    As $\efimb(\bm \theta)$ is a real symmetric matrix, we can write
    its spectrum decomposition as
    \begin{equation*}
    \efimb(\bm \theta) = \hat{\lambda}^\alpha \bm{v}_\alpha \bm{v}_\alpha^\T,
    \end{equation*}
    where $\left\{\bm{v}_\alpha\right\}$ are orthonormal vectors and \( \hat{\lambda}^{\alpha} \) are the corresponding eigenvalues of \( \efimb(\bm \theta) \).

    Therefore
    \begin{equation*}
    \hat{\lambda}^\alpha
    =
    \bm{v}_\alpha^\T \efimb(\bm \theta) \bm{v}_\alpha
    \ge
     - \rho(\partial^{2}\bm h_{L}^{a})
        \left \vert \bm \eta_{a} - \bar{\bm t}_{a} \right \vert.
    \end{equation*}
    The statement in \cref{thm:worstspectrum} follows immediately.

\end{proof}

\subsection{Proof of \cref{thm:est_cr_chebyshev}}

To prove the Lemma, we first consider a variant of the result presented in \citet{chen2007new}.
\begin{lemma}[Variant of \citet{chen2007new}]
    \label{lem:chebyshev_variant}
    For any random vector \( X \in \Re^{n} \) with variances \( \var(X) \),
\begin{equation*}
        \Pr\left( (X - \expect X)^{\T}(X - \expect X) \geq \varepsilon \right) \leq \frac{1}{\varepsilon} \cdot \sum_{i=1}^{n} \var(X_{i}), \quad \forall \varepsilon > 0.
    \end{equation*}
\end{lemma}
\begin{proof}
    The proofs follows closely to \citet{chen2007new}, with the slight change in set of variables considered. Let \( \varepsilon > 0 \) and \( {D}_{\varepsilon} \defeq \{ V \in \Re^{n} : (V - \expect X)^{\T}(V - \expect X) \geq \varepsilon \} \).
    By definition we have that for all \( V \in D_{\varepsilon} \),
\begin{equation*}
        (V - \expect X)^{\T}(V - \expect X) \cdot \frac{1}{\varepsilon} \geq 1
    \end{equation*}
    Thus for the probability of the set, we have:
\begin{align*}
        \Pr\left( (X - \expect X)^{\T}(X - \expect X) \geq \varepsilon \right)
        &= \Pr\left( X \in D_{ \varepsilon} \right) \\
        &= \expect\left[ \indicator{X \in D_{ \varepsilon}} \right] \\
        &\leq \frac{1}{\varepsilon} \cdot \expect\left[ (X - \expect X)^{\T}(X - \expect X) \cdot\indicator{X \in D_{ \varepsilon}} \right] \\
        &\leq \frac{1}{\varepsilon} \cdot \expect\left[ (X - \expect X)^{\T}(X - \expect X) \right] \\
        &= \frac{1}{\varepsilon} \cdot \sum_{i=1}^{n} \expect\left[ (X_{i} - \expect X_{i})(X_{i} - \expect X_{i}) \right] \\
        &= \frac{1}{\varepsilon} \cdot \sum_{i=1}^{n} \var(X_{i}),
    \end{align*}
    where \( \indicator{X \in S} \) denotes the indicator function for \( X \) being in a set \( S \).
\end{proof}

We can now prove the main Lemma through standard tricks:
\begin{proof}
    From \cref{lem:chebyshev_variant} (and utilizing the \( \vectorise \) operator) we have that for any \( \delta > 0 \) and each \( z \in \{ 1, 2 \} \):
\begin{equation*}
        \Pr( \Vert \efimz(\bm \theta) - \fim(\bm \theta) \Vert_{F}^{2} \leq \delta) \geq \frac{1}{\delta} \cdot \sum_{i, j=1}^{\dim(\bm \theta)}{\var\left( \efimz(\bm \theta) \right)^{ij}}.
    \end{equation*}
    Letting \( \varepsilon \defeq \frac{1}{\delta} \cdot \sum_{i, j=1}^{\dim(\bm \theta)}{\var\left( \efimz(\bm \theta) \right)^{ij}} \) and rearranging, we have:
\begin{align*}
        &\Pr\left( \Vert \efimz(\bm \theta) - \fim(\bm \theta) \Vert_{F}^{2} \geq \frac{1}{\epsilon} \cdot \sum_{i, j=1}^{\dim(\bm \theta)}{\var\left( \efimz(\bm \theta) \right)^{ij}} \right) \leq 1 - \varepsilon
        \\
        \implies
        &\Pr\left( \Vert \efimz(\bm \theta) - \fim(\bm \theta) \Vert_{F} \geq \frac{1}{\sqrt{\epsilon}} \cdot \sqrt{\sum_{i, j=1}^{\dim(\bm \theta)}{\var\left( \efimz(\bm \theta) \right)^{ij}} } \right) \leq 1 - \varepsilon.
    \end{align*}
    Substituting appropriately from \cref{eq:varfims} \wrt the FIM estimator considered completes the proof.
\end{proof}

\subsection{Proof of \cref{lem:dl_grad}}

\begin{proof}
    We first show that $\bm{B}_l$ is the Jacobian of the mapping $\bm{h}_l\to\bm{h}_L$, that is,
    \begin{equation}\label{eq:bl}
        d\bm{h}_L = \bm{B}_l d\bm{h}_l. \quad (l=0,\cdots,L)
    \end{equation}
    Obviously, this is true for $l=L$, as we have
    \begin{equation*}
        d\bm{h}_L = \bm{I}\cdot{}d\bm{h}_L = \bm{B}_L \cdot d\bm{h}_L.
    \end{equation*}
    From \cref{eq:exp}, we have
    \begin{equation*}
        \bm{h}_{l+1} = \sigma( \bm{W}_{l}\bar{\bm{h}}_{l}),
    \end{equation*}
    where $\sigma$ is abused to denote the non-linear activation function for
    $l=0,\cdots,L-2$ and $\sigma$ is the identity map for $l={L-1}$.
    Therefore
    \begin{equation*}
        d\bm{h}_{l+1} = \bm{D}_l \cdot \bm{W}_l^- \cdot d\bm{h}_{l},
    \end{equation*}
    where $\bm{D}_{l} = \diag\left( \sigma^{\prime}(\bm{W}_{l} \bar{\bm{h}}_{l}) \right)$
    for $l=0,\cdots,L-2$, and $\bm{D}_{L-1} = \bm{I}$.
    Assume \cref{eq:bl} is true for $l+1$, then
    \begin{equation*}
        d\bm{h}_L
        = \bm{B}_{l+1} d\bm{h}_{l+1}
        = \bm{B}_{l+1} \cdot \bm{D}_l \cdot \bm{W}_l^- \cdot d\bm{h}_{l}
        = \bm{B}_l d\bm{h}_{l}.
    \end{equation*}
    Hence $\bm{B}_l$ is the Jacobian of $\bm{h}_l\to\bm{h}_L$ for $l=0,\cdots,{L}$.

    Now we are ready to derive the expression of $\frac{\partial\ell}{\partial\bm{h}_{l}}$.
    \begin{align*}
        d\ell
        &= d (\log p(\bm{y} \mid \bm{x})) \\
        &= d (\bm{t}^\T(\bm{y})\bm{h}_L - F(\bm{h}_L)) \\
        &= d (\trace{}\left\{ \bm{t}^\T(\bm{y})\bm{h}_L - F(\bm{h}_L)\right\}) \\
        &= \trace{}\left\{d (\bm{t}^\T(\bm{y})\bm{h}_L - F(\bm{h}_L))\right\} \\
        &= \trace{}\left\{d (\bm{t}^\T(\bm{y})\bm{h}_L) - d( F(\bm{h}_L))\right\} \\
        &= \trace{}\left\{\bm{t}^\T(\bm{y}) d\bm{h}_L - \grad F(\bm{h}_L)^{\T} d\bm{h}_L\right\} \\
        &= \trace{}\left\{\bm{t}^\T(\bm{y}) d\bm{h}_L - \bm{\eta}^{\T}(\bm{h}_L) d\bm{h}_L\right\} \\
        &= \trace{}\left\{(\bm{t}(\bm{y}) - \bm{\eta}(\bm{h}_L))^{\T} d\bm{h}_L \right\}\\
        &= \trace{}\left\{(\bm{t}(\bm{y}) - \bm{\eta}(\bm{h}_L))^{\T} \bm{B}_l d\bm{h}_l \right\}\\
    \end{align*}
    Therefore
    \begin{align*}
        \frac{\partial\ell}{\partial\bm{h}_l}
        =
        \left(
            (\bm{t}(\bm{y}) - \bm{\eta}(\bm{h}_L))^{\T} \bm{B}_l
        \right)^\T
        =
        \bm{B}_l^\T (\bm{t}(\bm{y}) - \bm{\eta}(\bm{h}_L)).
    \end{align*}

    Next, we show the gradient \wrt the neural network weights, \ie $\frac{\partial\ell}{\partial\bm{W}_{l}}$.
    We have
    \begin{align*}
        d\ell
        &= \trace\left\{ \left(\frac{\partial\ell}{\partial\bm{h}_{l+1}}\right)^\T
        d\bm{h}_{l+1} \right\}\nonumber\\
        &= \trace\left\{ \left(\frac{\partial\ell}{\partial\bm{h}_{l+1}}\right)^\T
            \bm{D}_l \cdot d\bm{W}_l \cdot \bar{\bm{h}}_l
            \right\}\nonumber\\
        &= \trace\left\{
            \bar{\bm{h}}_l
            \left(\frac{\partial\ell}{\partial\bm{h}_{l+1}}\right)^\T
            \bm{D}_l \cdot d\bm{W}_l \right\}.
    \end{align*}
    Therefore
    \begin{align*}
        \frac{\partial\ell}{\partial\bm{W}_l}
        &=
        \left(
        \bar{\bm{h}}_l
            \left(\frac{\partial\ell}{\partial\bm{h}_{l+1}}\right)^\T
            \bm{D}_l
        \right)^\T\nonumber\\
        &=
        \bm{D}_l^\T
        \left(\frac{\partial\ell}{\partial\bm{h}_{l+1}}\right)
        \bar{\bm{h}}_l^\T\nonumber\\
        &=
        \bm{D}_l
        \left(\frac{\partial\ell}{\partial\bm{h}_{l+1}}\right)
        \bar{\bm{h}}_l^\T.
    \end{align*}

    Finally, we give the gradient of $\bm{h}_L$ \wrt the neural network parameters.
    By definition, $\bm{B}_{l+1}$ is the Jacobian of the mapping $\bm{h}_{l+1}\to\bm{h}_L$.
    Therefore
    \begin{equation*}
        d\bm{h}_L
        = \bm{B}_{l+1} \cdot d\bm{h}_{l+1}
        = \bm{B}_{l+1} \cdot \bm{D}_l d\bm{W}_l \bar{\bm{h}}_l.
    \end{equation*}
    We rewrite the above equation in the element-wise form
    \begin{align*}
        d\bm{h}_L^a
        & = \trace{}\left\{
            \bm{B}_{l+1} \cdot \bm{D}_l d\bm{W}_l \bar{\bm{h}}_l \bm{e}_a^\T
        \right\}\nonumber\\
        & = \trace{}\left\{
            \bar{\bm{h}}_l \bm{e}_a^\T
            \bm{B}_{l+1} \bm{D}_l \cdot d\bm{W}_l
        \right\}.
    \end{align*}
    Therefore
    \begin{align*}
        \frac{\partial {\bm h}^{a}_{L}}{\partial {\bm W}_{l}}
        &=
        \left(
        \bar{\bm{h}}_l \bm{e}_a^\T \bm{B}_{l+1} \bm{D}_l
        \right)^\T\nonumber\\
        &=
        \bm{D}_l^\T
        \bm{B}_{l+1}^\T
        \bm{e}_a
        \bar{\bm{h}}_l^\T\nonumber\\
        &=
        \bm{D}_l \bm{B}_{l+1}^\T \bm{e}_a
        \bar{\bm{h}}_l^\T.
    \end{align*}
\end{proof}

\subsection{Proof of \cref{lem:nn_grad_norm_bounded}}

\begin{proof}
    First, we notice that
    \begin{equation*}
        \frac{\partial {\bm h}^{a}_{L}}{\partial {\bm W}_{l}}
        =
        \bm{D}_l \bm{B}_{l+1}^\T \bm{e}_a
        \cdot
        \bar{\bm{h}}_l^\T
    \end{equation*}
    has rank one. Therefore
    \begin{equation*}
        \left \Vert
        \frac{\partial {\bm h}^{a}_{L}}{\partial {\bm W}_{l}}
        \right \Vert_F
        =
        \Vert \bm{D}_l \bm{B}_{l+1}^\T \bm{e}_a \Vert_2
        \cdot
        \Vert \bar{\bm{h}}_l \Vert_2.
    \end{equation*}
    Therefore
    \begin{align*}
        \left\Vert
        \frac{\partial {\bm h}_{L}}{\partial {\bm W}_{l}}
        \right\Vert_F
        &=
        \sqrt{ \sum_{a}
         \left \Vert
        \frac{\partial {\bm h}^{a}_{L}}{\partial {\bm W}_{l}}
        \right \Vert_F^2 } \nonumber\\
        &=
        \sqrt{
            \sum_{a}
        \Vert \bm{D}_l \bm{B}_{l+1}^\T \bm{e}_a \Vert_2^2
        \cdot
        \Vert \bar{\bm{h}}_l \Vert^2_2
        } \nonumber\\
        &=
        \sqrt{ \sum_{a}
        \Vert \bm{D}_l \bm{B}_{l+1}^\T \bm{e}_a \Vert_2^2}
        \cdot
        \Vert \bar{\bm{h}}_l \Vert_2\nonumber\\
        &=
        \Vert \bm{D}_l \bm{B}_{l+1}^\T \Vert_F
        \cdot
        \Vert \bar{\bm{h}}_l \Vert_2\nonumber\\
        &=
        \Vert \bm{B}_{l+1} \bm{D}_l \Vert_F
        \cdot
        \Vert \bar{\bm{h}}_l \Vert_2.
    \end{align*}

    The matrix $\bm{D}_l$ is diagonal, with entries bounded in the range
    $[-1,1]$.  Therefore a left or right multiplication by $\bm{D}_l$ does not
    increase the Frobenius norm. Hence
    \begin{equation*}
        \left\Vert
        \frac{\partial {\bm h}_{L}}{\partial {\bm W}_{l}}
        \right\Vert_F
        \le
        \Vert \bm{B}_{l+1} \Vert_F
        \cdot
        \Vert \bar{\bm{h}}_l \Vert_2.
    \end{equation*}
    By the recursive definition of $\bm{B}_l$, we have
    \begin{equation*}
        \bm{B}_{l}  = \bm{B}_{l+1} \bm{D}_{l} \bm{W}^{-}_{l}.
    \end{equation*}
    Therefore
    \begin{equation*}
        \Vert \bm{B}_{l} \Vert_F
        \le
        \Vert \bm{B}_{l+1}  \Vert_F
        \cdot
        \Vert \bm{D}_{l} \bm{W}^{-}_{l} \Vert_F
        \le
        \Vert \bm{B}_{l+1}  \Vert_F
        \cdot
        \Vert \bm{W}^{-}_{l} \Vert_F.
    \end{equation*}
    Applying the above inequality repeatably leads to
    \begin{align*}
        \Vert \bm{B}_l \Vert_F
        &\le
        \Vert \bm{B}_{L-1} \Vert_F
        \cdot
        \prod_{i=l}^{L-2}
        \Vert \bm{W}^{-}_{i} \Vert_F\nonumber\\
        &=
        \Vert \bm{B}_{L} \bm{D}_{L-1} \bm{W}_{L-1}^- \Vert_F
        \cdot
        \prod_{i=l}^{L-2}
        \Vert \bm{W}^{-}_{i} \Vert_F\nonumber\\
        &=
        \Vert \bm{W}_{L-1}^- \Vert_F
        \cdot
        \prod_{i=l}^{L-2}
        \Vert \bm{W}^{-}_{i} \Vert_F\nonumber\\
        &=
        \prod_{i=l}^{L-1}
        \Vert \bm{W}^{-}_{i} \Vert_F.
    \end{align*}
    Hence
    \begin{align*}
        \left\Vert
        \frac{\partial {\bm h}_{L}}{\partial {\bm W}_{l}}
        \right\Vert_F
        &
        \le
        \Vert \bm{B}_{l+1} \Vert_F
        \cdot
        \Vert \bar{\bm{h}}_l \Vert_2 \nonumber\\
        &
        \le
        \prod_{i=l+1}^{L-1}
        \Vert \bm{W}^{-}_{i} \Vert_F
        \cdot
        \Vert \bar{\bm{h}}_l \Vert_2.
    \end{align*}
\end{proof}

\subsection{Proof of \cref{lem:nn_grad_norm_bounded2}}

\begin{proof}
    Recall the 2-spectral norm for a $d$-dimensional tensor is defined by
\begin{equation*}
        \Vert \mathcal{T} \Vert_{2_{\sigma}}
        =
        \max
        \left\{
            \langle
            \mathcal{T}, \bm x^{1} \otimes \ldots \otimes \bm x^{d}
            \rangle
        : \Vert \bm x^{k} \Vert_{2}=1,~\forall{}k\in[d]
        \right\}.
    \end{equation*}
    Let $\bm\alpha$, $\bm\beta$, $\bm\gamma$ be unit vectors,
    so that
    \begin{equation*}
    \Vert \bm\alpha \Vert_2 = 1,\quad
    \Vert \bm\beta \Vert_2 = 1,\quad
    \Vert \bm\gamma \Vert_2 = 1.
    \end{equation*}
    Then
    \begin{align*}
    \max_{\bm\alpha,\bm\beta,\bm\gamma}
    \left\langle
        \frac{\partial \bm{h}_L}{\partial \bm W_{l}},
        \bm\alpha \otimes \bm\beta \otimes \bm\gamma
    \right\rangle
    &=
    \max_{\bm\alpha,\bm\beta,\bm\gamma}
    \left\langle
        \sum_{a}
    \frac{\partial \bm h^a_L}{\partial \bm W_{l}} \gamma_a,
    \bm\alpha \otimes \bm\beta
    \right\rangle
    \nonumber\\
    &=
    \max_{\bm\alpha,\bm\beta,\bm\gamma}
    \sum_{a}
    \bm\alpha^\T \bm{D}_l \bm{B}_{l+1}^\T
    \bm{e}_a
    \bar{\bm{h}}_l^\T \bm\beta
    \gamma_a
    \nonumber\\
    &=
    \max_{\bm\alpha,\bm\beta,\bm\gamma}
    \bm\alpha^\T \bm{D}_l \bm{B}_{l+1}^\T
    ( \sum_a \gamma_a \bm{e}_a )
    ( \bar{\bm{h}}_l^\T \bm\beta ) \nonumber\\
    &
    =
    \max_{\bm\alpha,\bm\beta,\bm\gamma}
    \bm\alpha^\T \bm{D}_l \bm{B}_{l+1}^\T \bm\gamma
    ( \bar{\bm{h}}_l^\T \bm\beta ) \nonumber\\
    &=
    \max_{\bm\alpha,\bm\beta,\bm\gamma}
    \langle \bm{B}_{l+1} \bm{D}_l \bm\alpha, \bm\gamma\rangle
    \cdot
    \langle \bar{\bm{h}}_l, \bm\beta \rangle
    \nonumber\\
    &
    =
    \max_{\bm\alpha}
    \Vert \bm{B}_{l+1} \bm{D}_l \bm\alpha \Vert_2
    \cdot
    \Vert \bar{\bm{h}}_l \Vert_2.
    \end{align*}

    Recall that
    \begin{equation*}
    \bm{B}_{l}  = \bm{B}_{l+1} \bm{D}_{l} \bm{W}^{-}_{l}.
    \end{equation*}
    Therefore
    \begin{equation*}
        \bm{B}_{l}\bm{D}_{l-1}\bm\alpha
        =
        \bm{B}_{l+1} \bm{D}_{l} \bm{W}^{-}_{l} \bm{D}_{l-1} \bm\alpha.
    \end{equation*}
    Moreover,
    \begin{align*}
    \max_{\bm\alpha}
    \Vert \bm{B}_{l}\bm{D}_{l-1}\bm\alpha \Vert_2
    &=
    \max_{\bm\alpha}
    \Vert \bm{B}_{l+1} \bm{D}_{l}
    \cdot \bm{W}^{-}_{l} \bm{D}_{l-1} \bm\alpha \Vert_2
    \nonumber\\
    &\le
    \max_{\bm\alpha^{\prime}}
    s_{\max}( \bm{W}^{-}_{l} )
    \Vert \bm{B}_{l+1} \bm{D}_{l} \bm\alpha^{\prime} \Vert_2
    \nonumber\\
    &=
    s_{\max}( \bm{W}^{-}_{l} )
    \max_{\bm\alpha^{\prime}}
    \Vert \bm{B}_{l+1} \bm{D}_{l} \bm\alpha^{\prime} \Vert_2.
    \end{align*}
    Hence,
    \begin{align*}
    \max_{\bm\alpha,\bm\beta,\bm\gamma}
    \left\langle
        \frac{\partial \bm h_L}{\partial \bm W_{l}},
        \bm\alpha \otimes \bm\beta \otimes \bm\gamma
    \right\rangle
    &=
    \max_{\bm\alpha}
    \Vert \bm{B}_{l+1} \bm{D}_l \bm\alpha \Vert_2
    \cdot
    \Vert \bar{\bm{h}}_l \Vert_2 \nonumber\\
    &
    \le
    \prod_{i=l+1}^{L-1} s_{\max}( \bm{W}^{-}_{i} )
    \max_{\bm\alpha^{\prime}}
    \Vert
    \bm{B}_{L}\bm{D}_{L-1} \bm\alpha^{\prime}
    \Vert_2
    \cdot
    \Vert \bar{\bm{h}}_l \Vert_2 \nonumber\\
    &=
    \prod_{i=l+1}^{L-1} s_{\max}( \bm{W}^{-}_{i} )
    \max_{\bm\alpha}
    \Vert
    \bm{B}_{L}\bm{D}_{L-1} \bm\alpha
    \Vert_2
    \cdot
    \Vert \bar{\bm{h}}_l \Vert_2 \nonumber\\
    &=
    \prod_{i=l+1}^{L-1} s_{\max}( \bm{W}^{-}_{i} )
    \max_{\bm\alpha}
    \Vert
    \bm{I} \bm{I} \bm\alpha
    \Vert_2
    \cdot
    \Vert \bar{\bm{h}}_l \Vert_2 \nonumber\\
    &=
    \prod_{i=l+1}^{L-1} s_{\max}( \bm{W}^{-}_{i} )
    \max_{\bm\alpha}
    \Vert
    \bm\alpha
    \Vert_2
    \cdot
    \Vert \bar{\bm{h}}_l \Vert_2 \nonumber\\
    &=
    \prod_{i=l+1}^{L-1} s_{\max}( \bm{W}^{-}_{i} )
    \cdot
    \Vert \bar{\bm{h}}_l \Vert_2.
    \end{align*}
\end{proof}

\section{Change of Coordinates and Covariance}
\label{sup:changecoordinates}

Reparametrization is a common technique in deep learning. See \eg weight
normalization~\citep{weightnorm}. From a geometric perspective,
reparametrization corresponds to change of coordinates.
We consider how our results can be generalized in a new coordinate system
\( \{ \bm \xi_{i} \} \) instead of the default coordinates \( \{ \bm \theta_{\alpha} \} \).
By definition, the FIM \wrt to the new coordinates is
\begin{equation}\label{eq:fim_xi}
    \fim(\bm \xi) =
    \frac{\partial \bm \theta_{a}}{\partial \bm \xi}
    \frac{\partial\bm{h}_L^\alpha}{\partial\bm\theta^{a}} \fim_{\alpha \beta}(\bm{h}_L) \frac{\partial\bm{h}_L^\beta}{\partial\bm\theta^{b}}
    \frac{\partial \bm \theta_{b}}{\partial \bm \xi^{\T}}
    =
    \frac{\partial \bm \theta_{a}}{\partial \bm \xi}
    \fim^{ab} (\bm \theta)
    \frac{\partial \bm \theta_{b}}{\partial \bm \xi^{\T}}.
\end{equation}
The FIM $\fim(\bm\xi)$ can be estimated by \cref{eq:estimators}, where
the estimators are denoted by
\( \efima(\bm \xi) \) and \( \efimb(\bm \xi) \).

Note that as per the main text, upper- and lower-indices in the Einstein are equivalent. Similarly, we take an upper index for derivatives of \( \bm h_{L} \) \wrt \( \bm \theta \) in this section --- which we take as equivalent to the lower-index notation appearing in the main text. That is \( \partial^{\beta} \bm h_{L} = {\partial \bm h_{L}} / {\partial \bm \theta^{\beta}} = {\partial \bm h_{L}} / {\partial \bm \theta_{\beta}} \).

\begin{theorem} \label{thm:est_xi}
Consider the FIM estimators under the coordinate transformation $\bm\theta\to\bm{\xi}$ with the same
    samples size $N$.
    \begin{align}
        \efima(\bm \xi) &=
        \frac{\partial \bm \theta_{a}}{\partial \bm \xi}
        \efima^{ab}(\bm \theta)
        \frac{\partial \bm \theta_{b}}{\partial \bm \xi^{\T}},  \label{eq:efimaxi_var} \\
\efimb(\bm \xi) &=
        \frac{\partial \bm \theta_{a}}{\partial \bm \xi}
        \efimb^{ab}(\bm \theta)
        \frac{\partial \bm \theta_{b}}{\partial \bm \xi^{\T}}
        + \left(\bm \eta_{\alpha} - \frac{1}{N} \sum_{i=1}^{N} \bm t_{\alpha}(\bm y_{i}) \right)
\partial^{\beta} \bm h_{L}^{\alpha}
        \frac{\partial^{2} \bm \theta_{\beta}}{\partial \bm \xi \partial \bm \xi^{\T}}.
        \label{eq:efimbxi_var}
\end{align}
\end{theorem}
\begin{proof}
    Let \( \hat{\expect}[X] \) denote the empirical expectation of random variable \( X \).

    For the first estimator, consider the partial derivative of the log-likelihood:
\begin{align*}
        \frac{\partial \ell}{\partial \bm \xi}
        &= \left( \frac{\partial \bm \theta}{\partial \bm \xi} \right)^{\T} \left( \frac{\partial \ell}{\partial \bm \theta} \right) \\
        &= \frac{\partial \bm \theta_{a}}{\partial \bm \xi} \frac{\partial \bm h_{L}^{\alpha}}{\partial \bm \theta^{a}} (\bm t_{\alpha}(\bm y) - \bm \eta_{\alpha}).
    \end{align*}
    Thus the first FIM estimator follows immediately:
\begin{align*}
        \hat{\expect}\left[
        \frac{\partial \ell}{\partial \bm \xi}
        \frac{\partial \ell}{\partial \bm \xi^{T}}
        \right]
        &=
        \hat{\expect}\left[
            \frac{\partial \bm \theta_{a}}{\partial \bm \xi} \frac{\partial \bm h_{L}^{\alpha}}{\partial \bm \theta^{a}} (\bm t_{\alpha}(\bm y) - \bm \eta_{\alpha})
            \frac{\partial \bm \theta_{b}}{\partial \bm \xi^{\T}} \frac{\partial \bm h_{L}^{\beta}}{\partial \bm \theta^{b}} (\bm t_{\beta}(\bm y) - \bm \eta_{\beta})
        \right] \\
        &=
        \frac{\partial \bm \theta_{a}}{\partial \bm \xi}
        \hat{\expect}\left[
            \frac{\partial \bm h_{L}^{\alpha}}{\partial \bm \theta^{a}} (\bm t_{\alpha}(\bm y) - \bm \eta_{\alpha})
            \frac{\partial \bm h_{L}^{\beta}}{\partial \bm \theta^{b}} (\bm t_{\beta}(\bm y) - \bm \eta_{\beta})
        \right]
        \frac{\partial \bm \theta_{b}}{\partial \bm \xi^{\T}}
        \\
        &=
        \frac{\partial \bm \theta_{a}}{\partial \bm \xi}
        [\efima(\bm \theta)]^{ab}
        \frac{\partial \bm \theta_{b}}{\partial \bm \xi^{\T}}.
    \end{align*}

    For the second estimator, we consider the second derivative:
\begin{align*}
        \frac{\partial^{2} \ell}{\partial \bm \xi \partial \bm \xi^{\T}}
        &= \frac{\partial}{\partial \bm \xi} \left[
            \frac{\partial \bm \theta_{a}}{\partial \bm \xi^{\T}} \frac{\partial \bm h_{L}^{\alpha}}{\partial \bm \theta^{a}} (\bm t_{\alpha}(\bm y) - \bm \eta_{\alpha})
         \right] \\
        &= - \frac{\partial \bm \eta_{\alpha}}{\partial \bm \xi} \frac{\partial \bm h_{L}^{\alpha}}{\partial \bm \theta^{a}} \frac{\partial \bm \theta_{a}}{\partial \bm \xi^{\T}}
        + (\bm t_{\alpha} - \bm \eta_{\alpha}) \left[ \frac{\partial}{\partial \bm \xi} \left( \frac{\partial \bm h_{L}^{\alpha}}{\partial \bm \theta^{a}} \right) \right] \frac{\partial \bm \theta_{a}}{\partial \bm \xi^{\T}}
        + (\bm t_{\alpha} - \bm \eta_{\alpha}) \frac{\partial \bm h_{L}^{\alpha}}{\partial \bm \theta^{a}} \frac{\partial^{2} \bm \theta_{a}}{\partial \bm \xi \partial \bm \xi^{\T}} \\
        &= - \frac{\partial \bm \theta_{b}}{\partial \bm \xi}
        \left[
            \frac{\partial [\bm h_{L}]_{\beta}}{\partial \bm \theta^{b}}
            \frac{\partial \bm \eta_{\alpha}}{\partial \bm h_{L}^{\beta}} \frac{\partial \bm h_{L}^{\alpha}}{\partial \bm \theta^{a}}
            - (\bm t_{\alpha} - \bm \eta_{\alpha}) \frac{\partial^{2} \bm h_{L}^{\alpha}}{\partial \bm \theta^{b} \partial \bm \theta^{a}}
        \right] \frac{\partial \bm \theta_{a}}{\partial \bm \xi^{\T}}
        + (\bm t_{\alpha} - \bm \eta_{\alpha}) \frac{\partial \bm h_{L}^{\alpha}}{\partial \bm \theta^{a}} \frac{\partial^{2} \bm \theta_{a}}{\partial \bm \xi \partial \bm \xi^{\T}}.
    \end{align*}
    Taking the empirical expectation we have:
    \begin{align*}
        &\hat{\expect}\left[-\frac{\partial^{2} \ell}{\partial \bm \xi \partial \bm \xi^{\T}}\right] \\
        &=
        \hat{\expect} \left[
        \frac{\partial \bm \theta_{b}}{\partial \bm \xi}
        \left[
            \frac{\partial [\bm h_{L}]_{\beta}}{\partial \bm \theta^{b}}
            \frac{\partial \bm \eta_{\alpha}}{\partial \bm h_{L}^{\beta}} \frac{\partial \bm h_{L}^{\alpha}}{\partial \bm \theta^{a}}
            - (\bm t_{\alpha} - \bm \eta_{\alpha}) \frac{\partial^{2} \bm h_{L}^{\alpha}}{\partial \bm \theta^{b} \partial \bm \theta^{a}}
        \right] \frac{\partial \bm \theta_{a}}{\partial \bm \xi^{\T}}
        - (\bm t_{\alpha} - \bm \eta_{\alpha}) \frac{\partial \bm h_{L}^{\alpha}}{\partial \bm \theta^{a}} \frac{\partial^{2} \bm \theta_{a}}{\partial \bm \xi \partial \bm \xi^{\T}} \right] \\
        &=
        \frac{\partial \bm \theta_{b}}{\partial \bm \xi}
        \hat{\expect}
        \left[
            \frac{\partial [\bm h_{L}]_{\beta}}{\partial \bm \theta^{b}}
            \frac{\partial \bm \eta_{\alpha}}{\partial \bm h_{L}^{\beta}} \frac{\partial \bm h_{L}^{\alpha}}{\partial \bm \theta^{a}}
            - (\bm t_{\alpha} - \bm \eta_{\alpha}) \frac{\partial^{2} \bm h_{L}^{\alpha}}{\partial \bm \theta^{b} \partial \bm \theta^{a}}
        \right]
        \frac{\partial \bm \theta_{a}}{\partial \bm \xi^{\T}}
        -
        \hat{\expect} \left[
        (\bm t_{\alpha} - \bm \eta_{\alpha}) \frac{\partial \bm h_{L}^{\alpha}}{\partial \bm \theta^{a}} \frac{\partial^{2} \bm \theta_{a}}{\partial \bm \xi \partial \bm \xi^{\T}} \right] \\
        &=
        \frac{\partial \bm \theta_{b}}{\partial \bm \xi}
        [\efimb(\bm \theta)]^{ba}
        \frac{\partial \bm \theta_{a}}{\partial \bm \xi^{\T}}
        -
        \hat{\expect} \left[
        (\bm t_{\alpha} - \bm \eta_{\alpha}) \frac{\partial \bm h_{L}^{\alpha}}{\partial \bm \theta^{a}} \frac{\partial^{2} \bm \theta_{a}}{\partial \bm \xi \partial \bm \xi^{\T}} \right] \\
        &=
        \frac{\partial \bm \theta_{a}}{\partial \bm \xi}
        [\efimb(\bm \theta)]^{ab}
        \frac{\partial \bm \theta_{b}}{\partial \bm \xi^{\T}}
        +
        \left(\bm \eta_{\alpha} - \frac{1}{N} \sum_{i=1} \bm t_{\alpha}(\bm y_{i}) \right)
        \frac{\partial \bm h_{L}^{\alpha}}{\partial \bm \theta^{a}} \frac{\partial^{2} \bm \theta_{a}}{\partial \bm \xi \partial \bm \xi^{\T}} \\
        &=
        \frac{\partial \bm \theta_{a}}{\partial \bm \xi}
        [\efimb(\bm \theta)]^{ab}
        \frac{\partial \bm \theta_{b}}{\partial \bm \xi^{\T}}
        +
        \left(\bm \eta_{\alpha} - \frac{1}{N} \sum_{i=1} \bm t_{\alpha}(\bm y_{i}) \right)
        \frac{\partial \bm h_{L}^{\alpha}}{\partial \bm \theta^{\beta}} \frac{\partial^{2} \bm \theta_{\beta}}{\partial \bm \xi \partial \bm \xi^{\T}}.
    \end{align*}
    Thus we have both estimators for the theorem.
\end{proof}
Interestingly, the way to compute \( \efima(\bm \xi) \) under coordinate
transformation follows the same rule to compute the FIM $\fim(\bm\xi)$
in the new coordinate system.
See the similarity between \cref{eq:fim_xi} and \cref{eq:efimaxi_var}.
The transformation rule of \( \efimb(\bm \xi) \) introduces an additional term,
which depends on the Hessian of the coordinate transform \( {\partial^{2} \bm \theta_{\beta}} / {\partial \bm \xi \partial \bm \xi^{\T}} \). This term vanishes as the number of samples increases \( N \rightarrow \infty \),
or the transformation $\bm\theta\to\bm\xi$ is affine.

The results we need to generalize for coordinate transformation depend on the (co)variance of the estimators. As such, we present a \( \{ \bm \xi_{i} \}\) variant of \cref{thm:varefima,thm:varefimb}.
\begin{theorem} \label{thm:estvar_xi}
    \begin{align}
        \left[
        \cov
        \left(
            \efima(\bm \xi)
        \right)
        \right]^{ijkl}
&=
        \frac{\partial \bm \theta_{a}}{\partial \bm \xi_{i}}
        \frac{\partial \bm \theta_{b}}{\partial \bm \xi_{j}}
        \frac{\partial \bm \theta_{c}}{\partial \bm \xi_{k}}
        \frac{\partial \bm \theta_{d}}{\partial \bm \xi_{l}}
\left[
        \cov
        \left(
            \efima(\bm \theta)
        \right) \right]^{abcd}
\label{eq:varefimaxi} \\
\left[
        \cov
        \left(
            \efimb(\bm \xi)
        \right)
        \right]^{ijkl}
&=
        \frac{\partial \bm \theta_{a}}{\partial \bm \xi_{i}}
        \frac{\partial \bm \theta_{b}}{\partial \bm \xi_{j}}
        \frac{\partial \bm \theta_{c}}{\partial \bm \xi_{k}}
        \frac{\partial \bm \theta_{d}}{\partial \bm \xi_{l}}
\left[
        \cov
        \left(
            \efimb(\bm \theta)
        \right) \right]^{abcd} + \frac{1}{N}
        \mathcal{C}^{\alpha \beta}
        \fim_{\alpha \beta}(\bm h_{L}),
        \label{eq:varefimbxi}
    \end{align}
    where \( \mathcal{C}^{\alpha \beta} \defeq \)
    \begin{align*}
        \frac{\partial \bm \theta_{a}}{\partial \bm \xi_{i}}
        \frac{\partial \bm \theta_{b}}{\partial \bm \xi_{j}}
        \frac{\partial^{2} \bm \theta_{c}}{\partial \bm \xi_{k} \partial \bm \xi_{l}}
        \partial^{a} \partial^{b} \bm h_{L}^{\alpha}
        \partial^{c} \bm h_{L}^{\beta}
        +
        \frac{\partial^{2} \bm \theta_{a}}{\partial \bm \xi_{i} \partial \bm \xi_{j}}
        \frac{\partial \bm \theta_{b}}{\partial \bm \xi_{k}}
        \frac{\partial \bm \theta_{c}}{\partial \bm \xi_{l}}
        \partial^{a} \bm h_{L}^{\alpha}
        \partial^{b} \partial^{c} \bm h_{L}^{\beta}
            +
        \frac{\partial^{2} \bm \theta_{a}}{\partial \bm \xi_{i} \partial \bm \xi_{j}}
        \frac{\partial^{2} \bm \theta_{b}}{\partial \bm \xi_{k} \partial \bm \xi_{l}}
        \partial^{a} \bm h_{L}^{\alpha}
        \partial^{b} \bm h_{L}^{\beta}.
    \end{align*}
\end{theorem}
\begin{proof}
    For the first estimator's covariance, immediately get the result by the way covariance interacts with constant products:
\begin{align*}
        \cov\left([\efima(\bm \xi)]^{ij}, [\efima(\bm \xi)]^{kl}\right)
        &=
        \cov\left(
        \frac{\partial \bm \theta_{a}}{\partial \bm \xi_{i}}
        [\efima(\bm \theta)]^{ab}
        \frac{\partial \bm \theta_{b}}{\partial \bm \xi_{j}}
        ,
        \frac{\partial \bm \theta_{c}}{\partial \bm \xi_{k}}
        [\efima(\bm \theta)]^{cd}
        \frac{\partial \bm \theta_{d}}{\partial \bm \xi_{l}}
        \right) \\
        &=
        \frac{\partial \bm \theta_{a}}{\partial \bm \xi_{i}}
        \frac{\partial \bm \theta_{b}}{\partial \bm \xi_{j}}
        \frac{\partial \bm \theta_{c}}{\partial \bm \xi_{k}}
        \frac{\partial \bm \theta_{d}}{\partial \bm \xi_{l}}
        \cov\left(
        [\efima(\bm \theta)]^{ab}
        ,
        [\efima(\bm \theta)]^{cd}
        \right) \\
    \end{align*}

    For the second estimator, we must exploit the linear combination property of covariances as well:
\begin{align*}
        &\cov\left([\efimb(\bm \xi)]^{ij}, [\efimb(\bm \xi)]^{kl}\right) \\
        &=
        \cov \left(
        \frac{\partial \bm \theta_{a}}{\partial \bm \xi_{i}}
        [\efimb(\bm \theta)]^{ab}
        \frac{\partial \bm \theta_{b}}{\partial \bm \xi_{j}}
        +
        \left(\bm \eta_{\alpha} - \frac{1}{N} \sum_{i=1} \bm t_{\alpha}(\bm y_{i}) \right)
        \frac{\partial \bm h_{L}^{\alpha}}{\partial \bm \theta^{a}} \frac{\partial^{2} \bm \theta_{a}}{\partial \bm \xi_{i} \partial \bm \xi_{j}}
        ,
        \right. \\
        & \quad \quad \quad \quad \quad \quad \quad \quad \left.
        \frac{\partial \bm \theta_{c}}{\partial \bm \xi_{k}}
        [\efimb(\bm \theta)]^{cd}
        \frac{\partial \bm \theta_{d}}{\partial \bm \xi_{l}}
        +
        \left(\bm \eta_{\beta} - \frac{1}{N} \sum_{i=1} \bm t_{\beta}(\bm y_{i}) \right)
        \frac{\partial \bm h_{L}^{\beta}}{\partial \bm \theta^{b}} \frac{\partial^{2} \bm \theta_{b}}{\partial \bm \xi_{k} \partial \bm \xi_{l}}
        \right) \\
        &= \cov \left(
        \frac{\partial \bm \theta_{a}}{\partial \bm \xi_{i}}
        [\efimb(\bm \theta)]^{ab}
        \frac{\partial \bm \theta_{b}}{\partial \bm \xi_{j}}
        ,
        \frac{\partial \bm \theta_{c}}{\partial \bm \xi_{k}}
        [\efimb(\bm \theta)]^{cd}
        \frac{\partial \bm \theta_{d}}{\partial \bm \xi_{l}}
        \right) \\
        & \quad \quad \quad \quad
        +
        \cov \left(
        \frac{\partial \bm \theta_{a}}{\partial \bm \xi_{i}}
        [\efimb(\bm \theta)]^{ab}
        \frac{\partial \bm \theta_{b}}{\partial \bm \xi_{j}}
        ,
        \left(\bm \eta_{\beta} - \frac{1}{N} \sum_{i=1} \bm t_{\beta}(\bm y_{i}) \right)
        \frac{\partial \bm h_{L}^{\beta}}{\partial \bm \theta^{b}} \frac{\partial^{2} \bm \theta_{b}}{\partial \bm \xi_{k} \partial \bm \xi_{l}}
        \right) \\
        & \quad \quad \quad \quad
        +
        \cov \left(
        \left(\bm \eta_{\alpha} - \frac{1}{N} \sum_{i=1} \bm t_{\alpha}(\bm y_{i}) \right)
        \frac{\partial \bm h_{L}^{\alpha}}{\partial \bm \theta^{a}} \frac{\partial^{2} \bm \theta_{a}}{\partial \bm \xi_{i} \partial \bm \xi_{j}}
        ,
        \frac{\partial \bm \theta_{c}}{\partial \bm \xi_{k}}
        [\efimb(\bm \theta)]^{cd}
        \frac{\partial \bm \theta_{d}}{\partial \bm \xi_{l}}
        \right) \\
        & \quad \quad \quad \quad
        +
        \cov \left(
        \left(\bm \eta_{\alpha} - \frac{1}{N} \sum_{i=1} \bm t_{\alpha}(\bm y_{i}) \right)
        \frac{\partial \bm h_{L}^{\alpha}}{\partial \bm \theta^{a}} \frac{\partial^{2} \bm \theta_{a}}{\partial \bm \xi_{i} \partial \bm \xi_{j}}
        ,
        \left(\bm \eta_{\beta} - \frac{1}{N} \sum_{i=1} \bm t_{\beta}(\bm y_{i}) \right)
        \frac{\partial \bm h_{L}^{\beta}}{\partial \bm \theta_{b}} \frac{\partial^{2} \bm \theta_{b}}{\partial \bm \xi_{k} \partial \bm \xi_{l}}
        \right).
    \end{align*}
    It follows that the first covariance term is exactly the coordinate transform of the original variance:
\begin{align*}
        \cov \left(
        \frac{\partial \bm \theta_{a}}{\partial \bm \xi_{i}}
        [\efimb(\bm \theta)]^{ab}
        \frac{\partial \bm \theta_{b}}{\partial \bm \xi_{j}}
        ,
        \frac{\partial \bm \theta_{c}}{\partial \bm \xi_{k}}
        [\efimb(\bm \theta)]^{cd}
        \frac{\partial \bm \theta_{d}}{\partial \bm \xi_{l}}
        \right)
        =
        \frac{\partial \bm \theta_{a}}{\partial \bm \xi_{i}}
        \frac{\partial \bm \theta_{b}}{\partial \bm \xi_{j}}
        \frac{\partial \bm \theta_{c}}{\partial \bm \xi_{k}}
        \frac{\partial \bm \theta_{d}}{\partial \bm \xi_{l}}
        \cov \left(
        [\efimb(\bm \theta)]^{ab}
        ,
        [\efimb(\bm \theta)]^{cd}
        \right).
    \end{align*}

    For the final covariance term we have:
\begin{align*}
        &\cov \left(
        \left(\bm \eta_{\alpha} - \frac{1}{N} \sum_{i=1} \bm t_{\alpha}(\bm y_{i}) \right)
        \frac{\partial \bm h_{L}^{\alpha}}{\partial \bm \theta^{a}} \frac{\partial^{2} \bm \theta_{a}}{\partial \bm \xi_{i} \partial \bm \xi_{j}}
        ,
        \left(\bm \eta_{\beta} - \frac{1}{N} \sum_{i=1} \bm t_{\beta}(\bm y_{i}) \right)
        \frac{\partial \bm h_{L}^{\beta}}{\partial \bm \theta^{b}} \frac{\partial^{2} \bm \theta_{b}}{\partial \bm \xi_{k} \partial \bm \xi_{l}}
        \right) \\
        &=
        \frac{\partial^{2} \bm \theta_{a}}{\partial \bm \xi_{i} \partial \bm \xi_{j}}
        \frac{\partial^{2} \bm \theta_{b}}{\partial \bm \xi_{k} \partial \bm \xi_{l}}
        \cov \left(
        \left(\bm \eta_{\alpha} - \frac{1}{N} \sum_{i=1} \bm t_{\alpha}(\bm y_{i}) \right)
        \frac{\partial \bm h_{L}^{\alpha}}{\partial \bm \theta^{a}}
        ,
        \left(\bm \eta_{\beta} - \frac{1}{N} \sum_{i=1} \bm t_{\beta}(\bm y_{i}) \right)
        \frac{\partial \bm h_{L}^{\beta}}{\partial \bm \theta^{b}}
        \right) \\
        &=
        \frac{\partial^{2} \bm \theta_{a}}{\partial \bm \xi_{i} \partial \bm \xi_{j}}
        \frac{\partial^{2} \bm \theta_{b}}{\partial \bm \xi_{k} \partial \bm \xi_{l}}
        \frac{\partial \bm h_{L}^{\alpha}}{\partial \bm \theta^{a}}
        \frac{\partial \bm h_{L}^{\beta}}{\partial \bm \theta^{b}}
        \cov \left(
        \left(\bm \eta_{\alpha} - \frac{1}{N} \sum_{i=1} \bm t_{\alpha}(\bm y_{i}) \right)
        ,
        \left(\bm \eta_{\beta} - \frac{1}{N} \sum_{i=1} \bm t_{\beta}(\bm y_{i}) \right)
        \right) \\
        &=
        \frac{\partial^{2} \bm \theta_{a}}{\partial \bm \xi_{i} \partial \bm \xi_{j}}
        \frac{\partial^{2} \bm \theta_{b}}{\partial \bm \xi_{k} \partial \bm \xi_{l}}
        \frac{\partial \bm h_{L}^{\alpha}}{\partial \bm \theta^{a}}
        \frac{\partial \bm h_{L}^{\beta}}{\partial \bm \theta^{b}}
        \frac{1}{N} \sum_{i=1}
        \cov \left(
        \left(\bm \eta_{\alpha} - \bm t_{\alpha}(\bm y_{i}) \right)
        ,
        \left(\bm \eta_{\beta} - \bm t_{\beta}(\bm y_{i}) \right)
        \right) \\
        &=
        \frac{\partial^{2} \bm \theta_{a}}{\partial \bm \xi_{i} \partial \bm \xi_{j}}
        \frac{\partial^{2} \bm \theta_{b}}{\partial \bm \xi_{k} \partial \bm \xi_{l}}
        \frac{\partial \bm h_{L}^{\alpha}}{\partial \bm \theta^{a}}
        \frac{\partial \bm h_{L}^{\beta}}{\partial \bm \theta^{b}}
        \frac{1}{N} \fim_{\alpha \beta}(\bm h_{L}),
    \end{align*}
    where the second last line comes from the independence of samples.

    Thus all there is left is to calculate the middle terms. Without loss of generality, we calculate:
    \begin{align*}
        &\cov \left(
        \frac{\partial \bm \theta_{a}}{\partial \bm \xi_{i}}
        [\efimb(\bm \theta)]^{ab}
        \frac{\partial \bm \theta_{b}}{\partial \bm \xi_{j}}
        ,
        \left(\bm \eta_{\beta} - \frac{1}{N} \sum_{i=1} \bm t_{\beta}(\bm y_{i}) \right)
        \frac{\partial \bm h_{L}^{\beta}}{\partial \bm \theta^{b}} \frac{\partial^{2} \bm \theta_{b}}{\partial \bm \xi_{k} \partial \bm \xi_{l}}
        \right) \\
        &=
        \frac{\partial \bm \theta_{a}}{\partial \bm \xi_{i}}
        \frac{\partial \bm \theta_{b}}{\partial \bm \xi_{j}}
        \frac{\partial^{2} \bm \theta_{c}}{\partial \bm \xi_{k} \partial \bm \xi_{l}}
        \cov \left(
        [\efimb(\bm \theta)]^{ab}
        ,
        \left(\bm \eta_{\beta} - \frac{1}{N} \sum_{i=1} \bm t_{\beta}(\bm y_{i}) \right)
        \frac{\partial \bm h_{L}^{\beta}}{\partial \bm \theta^{c}}
        \right) \\
        &=
        \frac{\partial \bm \theta_{a}}{\partial \bm \xi_{i}}
        \frac{\partial \bm \theta_{b}}{\partial \bm \xi_{j}}
        \frac{\partial^{2} \bm \theta_{c}}{\partial \bm \xi_{k} \partial \bm \xi_{l}} \\
        &\quad\quad
        \cov \left(
        \fim_{ab} (\bm \theta) +
        \left(\bm \eta_{\beta} - \frac{1}{N} \sum_{i=1} \bm t_{\beta}(\bm y_{i}) \right)
        \frac{\partial^{2} \bm h_{L}^{\alpha}}{\partial \bm \theta^{a} \bm \theta^{b}}
        ,
        \left(\bm \eta_{\beta} - \frac{1}{N} \sum_{i=1} \bm t_{\beta}(\bm y_{i}) \right)
        \frac{\partial \bm h_{L}^{\beta}}{\partial \bm \theta^{c}}
        \right) \\
        &=
        \frac{\partial \bm \theta_{a}}{\partial \bm \xi_{i}}
        \frac{\partial \bm \theta_{b}}{\partial \bm \xi_{j}}
        \frac{\partial^{2} \bm \theta_{c}}{\partial \bm \xi_{k} \partial \bm \xi_{l}}
        \frac{\partial^{2} \bm h_{L}^{\alpha}}{\partial \bm \theta^{a} \bm \theta^{b}}
        \frac{\partial \bm h_{L}^{\beta}}{\partial \bm \theta^{c}}
        \cov \left(
        \left(\bm \eta_{\beta} - \frac{1}{N} \sum_{i=1} \bm t_{\beta}(\bm y_{i}) \right)
        ,
        \left(\bm \eta_{\beta} - \frac{1}{N} \sum_{i=1} \bm t_{\beta}(\bm y_{i}) \right)
        \right) \\
        &=
        \frac{\partial \bm \theta_{a}}{\partial \bm \xi_{i}}
        \frac{\partial \bm \theta_{b}}{\partial \bm \xi_{j}}
        \frac{\partial^{2} \bm \theta_{c}}{\partial \bm \xi_{k} \partial \bm \xi_{l}}
        \frac{\partial^{2} \bm h_{L}^{\alpha}}{\partial \bm \theta^{a} \bm \theta^{b}}
        \frac{\partial \bm h_{L}^{\beta}}{\partial \bm \theta^{c}}
        \frac{1}{N} \fim_{\alpha \beta}( \bm h_{L}).
    \end{align*}

    Combining these 3 covariance results gives us the covariance presented in the theorem.
\end{proof}

As per the estimators themselves in \cref{thm:est_xi}, the 4D covariances
obey similar rules under coordinate transformations.
\Cref{eq:varefimbxi} has an additional term \( \frac{1}{N} \mathcal{C}^{\alpha \beta} \fim_{\alpha \beta}(\bm h_{L}) \) which depends on the Hessian of the coordinate transformation.
Notice that each of the coordinate transformed covariance values depend on the same central moments of the exponential family -- even the second estimator with the additional term only depends on the covariance/FIM \wrt \( \bm h_{L} \). Each of these covariances include a weighted sums of the original covariance matrices in \cref{thm:varefima,thm:varefimb}. As such, our initial element-wise considerations of the covariance will still be useful in the computation of the covariance in the new coordinates. In-fact, our upper bounds in \cref{thm:varefima1,thm:varefimb1,thm:varefimboth2} can be simply adapted by adding the appropriate norms of the coordinate transformation (and its Hessian component for the second estimator).

\section{Element-wise Covariance Bounds}
\label{sup:elementwise}

The covariance tensor $\left[ \cov \left( \efima(\bm \theta) \right) \right]^{ijkl}$
has the following element-wise bound.
\begin{lemma}\label{thm:varefima1}
        \begin{align*}
            \left\vert
                \left[
                \cov \left( \efima(\bm \theta) \right)
                \right]^{ijkl}
            \right\vert
                & \le
                \frac{1}{N}\cdot
                \Vert\partial_{i}{\bm h}_{L}(\bm x)\Vert_{2}
                \cdot
                \Vert\partial_{j}{\bm h}_{L}(\bm x)\Vert_{2}
                \cdot
                \Vert\partial_{k}{\bm h}_{L}(\bm x)\Vert_{2}
                \cdot
                \Vert\partial_{l}{\bm h}_{L}(\bm x)\Vert_{2}
                \nonumber \\
                &\quad\quad \cdot
                \Vert \kurt(\bm{t}) - \fim(\bm{h}_L) \otimes \fim(\bm{h}_L) \Vert_{F},
        \end{align*}
where $\Vert\cdot\Vert_{F}$ is the Frobenius norm of a tensor
(square root of the sum of the squares of the elements / the \( L_{2} \)-norm) and $\otimes$ is the tensor-product:
$\left( \fim(\bm{h}_L) \otimes \fim(\bm{h}_L) \right)_{abcd} \defeq \fim_{ab}(\bm{h}_L) \cdot \fim_{cd}(\bm{h}_L)$.
\end{lemma}
\begin{proof}
    Corollary holds immediately from the use of H\"older's inequality / Cauchy-Schwarz.
    Let \( 1 \leq p, q \leq \infty \) such that \( 1 / p  + 1 / q = 1 \).
    Let \( \mathcal{T}_{abcd} = \kurt_{abcd}(\bm{t}) - \fim_{ab}(\bm{h}_L) \cdot \fim_{cd}(\bm{h}_L) \).
    From \cref{thm:varefima} we have:
\begin{align*}
        &\left \vert
        \left[
        \cov
        \left(
            \efima(\bm \theta)
        \right)
        \right]^{ijkl}
        \right \vert \\
        &=
        \frac{1}{N}
        \cdot
        \left \vert
        \partial_{i} {\bm h}_{L}^{a}(\bm x) \partial_{j} {\bm h}_{L}^{b}(\bm x) \partial_{k} {\bm h}_{L}^{c}(\bm x) \partial_{l} {\bm h}_{L}^{d}(\bm x)
        \cdot
        \mathcal{T}_{abcd}
        \right \vert \\
        &\leq
        \frac{1}{N}
        \cdot
        \Vert
        \partial_{i} {\bm h}_{L}(\bm x) \otimes
        \partial_{j} {\bm h}_{L}(\bm x) \otimes
        \partial_{k} {\bm h}_{L}(\bm x) \otimes
        \partial_{l} {\bm h}_{L}(\bm x)
        \Vert_{p}
        \cdot
        \Vert
        \mathcal{T}
        \Vert_{q} \\
        &=
        \frac{1}{N}
        \cdot
        \Vert
        \partial_{i} {\bm h}_{L}(\bm x)
        \Vert_{p}
        \cdot
        \Vert
        \partial_{j} {\bm h}_{L}(\bm x)
        \Vert_{p}
        \cdot
        \Vert
        \partial_{k} {\bm h}_{L}(\bm x)
        \Vert_{p}
        \cdot
        \Vert
        \partial_{l} {\bm h}_{L}(\bm x)
        \Vert_{p}
        \cdot
        \Vert
        \mathcal{T}
        \Vert_{q}.
    \end{align*}
    Thus, by taking the \( p = q = 2 \) the Lemma holds.
\end{proof}

We have similar element-wise and global upper bounds on the covariance of
$\efimb(\bm\theta)$.
\begin{lemma}\label{thm:varefimb1}
    \begin{align*}
\left\vert \left[
        \cov
        \left(
            \efimb(\bm \theta)
        \right)
        \right]^{ijkl} \right\vert
        &
        \le
        \frac{1}{N} \cdot
        \Vert \partial_{ij}^2\bm{h}_L(\bm{x}) \Vert_{2}
        \cdot
        \Vert \partial_{kl}^2\bm{h}_L(\bm{x}) \Vert_{2}
        \cdot
        \Vert \fim(\bm{h}_L) \Vert_{F}.
    \end{align*}
\end{lemma}
\begin{proof}
    Corollary holds immediately from the use of H\"older's inequality / Cauchy-Schwarz.
    Let \( 1 \leq p, q \leq \infty \) such that \( 1 / p  + 1 / q = 1 \).
    From \cref{thm:varefimb} we have:
\begin{align*}
        \left \vert
        \left[
        \cov
        \left(
            \efimb(\bm \theta)
        \right)
        \right]^{ijkl}
        \right \vert
        &=
        \frac{1}{N}
        \cdot
        \left \vert
        \partial^{2}_{ij}{\bm h}_{L}^{\alpha}({\bm x})
        \partial^{2}_{kl}{\bm h}_{L}^{\beta}({\bm x})
        \fim_{\alpha\beta}( \bm{h}_L )
        \right \vert \\
        &\leq
        \frac{1}{N}
        \cdot
        \Vert
        \partial_{ij}^{2}{\bm h}_{L}({\bm x})
        \otimes
        \partial_{kl}^{2}{\bm h}_{L}({\bm x})
        \Vert
        \cdot
        \Vert
        \fim( \bm{h}_L )
        \Vert \\
        &=
        \frac{1}{N}
        \cdot
        \Vert
        \partial_{ij}^{2}{\bm h}_{L}({\bm x})
        \Vert
        \cdot
        \Vert
        \partial_{kl}^{2}{\bm h}_{L}({\bm x})
        \Vert
        \cdot
        \Vert
        \fim( \bm{h}_L )
        \Vert.
    \end{align*}
    Thus, by taking the \( p = q = 2 \) the Lemma holds.
\end{proof}

\section{Alternative Norm Results}
\label{sup:alternativenorm}

An alternative bound to \cref{thm:varfima11} and \cref{thm:varfimb11} can be established by utilizing H\"older's inequality for \( L_{p} \)-norms. \begin{theorem}\label{thm:varefimboth2}
    \begin{align}
        \left \Vert
        \cov
        \left(
            \efima(\bm \theta)
        \right)
        \right \Vert_{\infty}
        &\leq
        \frac{1}{N}
        \cdot
        \Vert \partial {\bm h}_{L}(\bm x) \Vert_{\infty}^{4}
        \cdot
        \Vert \kurt(\bm{t}) - \fim(\bm{h}_L) \otimes \fim(\bm{h}_L) \Vert_{1}
        \label{eq:varefima2} \\
        \left \Vert
        \cov
        \left(
            \efimb(\bm \theta)
        \right)
        \right \Vert_{\infty}
        &\leq
        \frac{1}{N}
        \cdot
        \Vert
        \partial^{2}{\bm h}_{L}({\bm x})
        \Vert_{\infty}^{2}
        \cdot
        \Vert
        \fim( \bm{h}_L )
        \Vert_{1}, \label{eq:varefimb2}
    \end{align}
where \( \Vert \cdot \Vert_{\infty} \) is the \( L_{\infty} \)-norm and \( \Vert \cdot \Vert_{1} \) is the \( L_{1} \)-norm.
\end{theorem}
\begin{proof}
    The Corollary holds directly from the inequalities given in the proof of \cref{thm:varfima11} and \cref{thm:varfimb11}.
    Let \( p = \infty \) and \( q = 1 \).

    Thus we have the inequalities for \( L_{p} \)-norms:
\begin{align*}
        \left \Vert
        \cov
        \left(
            \efima(\bm \theta)
        \right)
        \right \Vert_{\infty}
        &\leq
        \frac{1}{N}
        \cdot
        \Vert \partial {\bm h}_{L}(\bm x) \Vert_{\infty}^{4}
        \cdot
        \Vert \kurt(\bm{t}) - \fim(\bm{h}_L) \otimes \fim(\bm{h}_L) \Vert_{1}
        \\
        \left \Vert
        \cov
        \left(
            \efimb(\bm \theta)
        \right)
        \right \Vert_{\infty}
        &\leq
        \frac{1}{N}
        \cdot
        \Vert
        \partial^{2}{\bm h}_{L}({\bm x})
        \Vert_{\infty}^{2}
        \cdot
        \Vert
        \fim( \bm{h}_L )
        \Vert_{1}.
    \end{align*}

    \begin{remark}
    Note that these are exactly equivalent to certain spectral and nuclear norms. However these have slight differences in their definition \citep{chen2020tensor}.

    The \( p \)-spectral norm for a \( d \)-dimensional tensor is given by:
\begin{equation*}
        \Vert \mathcal{T} \Vert_{p_{\sigma}}
        =
        \max
        \left\{ \langle \mathcal{T}, \bm x^{1} \otimes \ldots \otimes \bm x^{d}
            \rangle
        : \Vert \bm x^{k} \Vert_{p}=1,~\forall{}k\in[d]
        \right\}.
    \end{equation*}
    The standard tensor spectral norm is only equivalent when \( p = 2 \), \ie, \( \Vert \cdot \Vert_{2_\sigma} \).

    The \( p \)-nuclear norm for a \( d \)-dimensional tensor is given by:
\begin{equation*}
        \Vert \mathcal{T} \Vert_{p_{*}}
        =
        \min \left\{
            \sum_{i=1}^{r} \vert \lambda_{i} \vert :
            \mathcal{T} = \sum_{i=1}^{r} \lambda_{i} \bm x_{i}^{1} \otimes \ldots \otimes \bm x_{i}^{d}
        :
        \Vert \bm x_{i}^{k} \Vert_{p}=1,
        ~\forall{}k\in[d]
        \text{~and~} i, r \in \mathbb{N}
        \right\}.
    \end{equation*}
    The standard tensor nuclear norm is only equivalent when \( p = 2 \), \ie, \( \Vert \cdot \Vert_{{2}_{*}} \).

    It follows that for the \( L_{p} \) norms we have established, the \( L_{1} \)-norm is equivalent to the \( 1 \)-nuclear norm and the \( L_{\infty} \)-norm is equal to the \( 1 \)-spectral norm \citep[Proposition 2.6]{chen2020tensor}. However, this equality is not true for the standard tensor nuclear and spectral norms.

    Instead, we can upper bound the \( L_{\infty} \)-norm by the standard tensor spectral norm by the definition of the \( p \)-spectral norm, \ie, \( \Vert \cdot \Vert_{\infty} = \Vert \cdot \Vert_{1_{\sigma}} \leq \Vert \cdot \Vert_{2_{\sigma}} \).
    For the \( L_{1} \)-norm, we can upper bound it by the corresponding \( L_{2} \)-norm through Cauchy-Schwarz, \ie, \( \Vert \cdot \Vert_{1} \leq \Vert \cdot \Vert_{2} \cdot \sqrt{\mathcal{D}} \), where \( \mathcal{D}(\cdot) \) is the product of the dimension size of the tensor. One should note that \( L_{2} \) is the Frobenius norm.
    \end{remark}
\end{proof}

Notably, the \( L_{1} \)-norm can be upper-bounded by
the Frobenius norm trivially through the Cauchy-Schwarz inequality, with \( \Vert \cdot \Vert_{1} \leq \Vert \cdot \Vert_{F} \cdot \sqrt{\mathcal{D}} \), where \( \mathcal{D} \) is the product of the dimension size of the tensor.
Thus the analysis in \cref{thm:elementwisemoment} can be useful to extend the bounds in~\cref{thm:varefimboth2}.
On the other hand, the \( L_{\infty} \)-norm \( \Vert \cdot \Vert_{\infty} \) is upper bounded by the largest singular value of the tensor~\citep{lim2005singular}.
As such, the \( \Vert \cdot \Vert_{\infty} \) quantities on the RHS
can be interpreted as functions of the maximum singular value of
the Jacobian \( \partial \bm h_{L}(\bm x) \) or the Hessian \( \partial^{2} \bm h_{L}(\bm x) \).
Similar corollaries for \cref{thm:varefima1} and \cref{thm:varefimb1} can be established
--- which we omit for brevity.

\section{Combination of Estimators}
\label{sup:combination}

We consider a convex combination of estimators, i.e., \cref{eq:efima,eq:efimb}.

In particular, for \( 0 \leq \alpha \leq 1 \) we have:
\begin{equation}
    \efimc(\bm \theta) = \alpha \efima(\bm \theta) + (1 - \alpha) \efimb(\bm \theta).
\end{equation}

Clearly, this is also a point-wise estimator. Thus, \cref{prop:unbiased_consistent} holds for this estimator.

For the variance, we use the following linear relation:
\begin{equation*}
    \var(\alpha \efima(\bm \theta) + (1-\alpha) \efimb(\bm \theta)) = \alpha^{2} \var(\efima(\bm \theta)) + (1 - \alpha)^{2} \var(\efimb(\bm \theta)) + 2\alpha(1-\alpha) \cov(\efima(\bm \theta), \efimb(\bm \theta)).
\end{equation*}
Or as we have previously discussed we consider:
\begin{equation*}
    \alpha^{2} \var\left(\frac{\partial \ell}{\partial \bm \theta_{i}} \frac{\partial \ell}{\partial \bm \theta_{j}}\right) + (1 - \alpha)^{2} \var\left(- \frac{\partial^{2} \ell}{\partial \bm \theta_{i} \partial \bm \theta_{j}}\right) + 2\alpha(1-\alpha) \cov\left(\frac{\partial \ell}{\partial \bm \theta_{i}} \frac{\partial \ell}{\partial \bm \theta_{j}}, - \frac{\partial^{2} \ell}{\partial \bm \theta_{i} \partial \bm \theta_{j}}\right)
\end{equation*}

We already have the variance values from \cref{thm:varefima,thm:varefimb}. Thus all we have to calculate is the covariance term.

\begin{align*}
    \cov\left(\partial_{i} \ell \cdot \partial_{j} \ell, -\partial^{2}_{ij} \ell \right)
    &= - \expect[\partial_{i} \ell \cdot \partial_{j} \ell \cdot \partial^{2}_{ij} \ell] + \expect[\partial_{i} \ell \cdot \partial_{j} \ell] \expect[\partial^{2}_{ij} \ell].
\end{align*}

The first term of the covariance can be given by:
\begin{align*}
    &- \expect[\partial_{i} \ell \cdot \partial_{j} \ell \cdot \partial^{2}_{ij} \ell] \\
    &= E\left[ \partial_{i} \ell \cdot \partial_{j} \ell \cdot \left( \partial_{i}^{c} \bm h_{L} \partial_{j}^{d} \bm h_{L} \fim_{cd}(\bm h_{L}) - \delta_{c} \partial_{ij}^{c} \bm h_{L} \right) \right] \\
    &= E\left[ \partial_{i} \ell \cdot \partial_{j} \ell \cdot \partial_{i}^{c} \bm h_{L} \partial_{j}^{d} \bm h_{L} \fim_{cd}(\bm h_{L}) \right] - E\left[ \partial_{i} \ell \cdot \partial_{j} \ell \cdot  \delta_{c} \partial_{ij}^{c} \bm h_{L} \right] \\
    &= \partial_{i}^{c} \bm h_{L} \partial_{j}^{d} \bm h_{L} \fim_{cd}(\bm h_{L}) \cdot E\left[ \partial_{i} \ell \cdot \partial_{j} \ell \right] - \partial_{ij}^{c} \bm h_{L} \cdot E\left[ \partial_{i} \ell \cdot \partial_{j} \ell \cdot \delta_{c} \right] \\
    &= \partial_{i}^{c} \bm h_{L} \partial_{j}^{d} \bm h_{L} \fim_{cd}(\bm h_{L}) \cdot E\left[ \partial_{i} \ell \cdot \partial_{j} \ell \right] - \partial_{ij}^{c} \bm h_{L} \cdot E\left[ \partial_{i} \ell \cdot \partial_{j} \ell \cdot \delta_{c} \right] \\
    &= \partial_{i} {\bm h}_{L}^{a} \cdot \partial_{j} {\bm h}_{L}^{b} \cdot \partial_{i}^{c} \bm h_{L} \cdot \partial_{j}^{d} \bm h_{L} \cdot \fim_{cd}(\bm h_{L}) \cdot E\left[ \delta_{a} \cdot \delta_{b} \right] - \partial_{i} {\bm h}_{L}^{a}(\bm x) \cdot \partial_{j} {\bm h}_{L}^{b}(\bm x) \cdot \partial_{ij}^{c} \bm h_{L} \cdot E\left[ \delta_{a} \cdot \delta_{b} \cdot \delta_{c} \right] \\
    &= \partial_{i} {\bm h}_{L}^{a} \cdot \partial_{j} {\bm h}_{L}^{b} \cdot \partial_{i}^{c} \bm h_{L} \cdot \partial_{j}^{d} \bm h_{L} \cdot \fim_{cd}(\bm h_{L}) \cdot \fim_{ab}(\bm h_{L}) - \partial_{i} {\bm h}_{L}^{a}(\bm x) \cdot \partial_{j} {\bm h}_{L}^{b}(\bm x) \cdot \partial_{ij}^{c} \bm h_{L} \cdot E\left[ \delta_{a} \cdot \delta_{b} \cdot \delta_{c} \right] \\
\end{align*}

The second term is given by:
\begin{align*}
    \expect[\partial_{i} \ell \cdot \partial_{j} \ell] \expect[\partial^{2}_{ij} \ell]
    &= - \partial_{i} {\bm h}_{L}^{a} \cdot \partial_{j} {\bm h}_{L}^{b} \cdot \partial_{i}^{c} \bm h_{L} \cdot \partial_{j}^{d} \bm h_{L} \cdot \fim_{cd}(\bm h_{L}) \cdot \fim_{ab}(\bm h_{L})
\end{align*}

Thus, together we have the covariance:
\begin{align*}
    \cov\left(\partial_{i} \ell \cdot \partial_{j} \ell, -\partial^{2}_{ij} \ell \right)
    = - \partial_{i} {\bm h}_{L}^{a} \cdot \partial_{j} {\bm h}_{L}^{b} \cdot \partial_{ij}^{c} \bm h_{L} \cdot E\left[ \delta_{a} \cdot \delta_{b} \cdot \delta_{c} \right].
\end{align*}

This gives us the variance of the combined estimator:
\begin{align}
    &\alpha^{2} \partial_{i} {\bm h}_{L}^{a} \partial_{j} {\bm h}_{L}^{b} \partial_{i} {\bm h}_{L}^{c} \partial_{j} {\bm h}_{L}^{d}
    \cdot
    \left( \kurt_{abcd}(\bm{t}) - \fim_{ab}(\bm{h}_L) \cdot \fim_{cd}(\bm{h}_L) \right)
+
    (1 - \alpha)^{2} \partial^{2}_{ij}{\bm h}_{L}^{\alpha}
    \partial^{2}_{ij}{\bm h}_{L}^{\beta}
    \cdot
    \fim_{\alpha\beta}( \bm{h}_L )
    \nonumber \\
    &\quad \quad-
    2 \alpha (1 - \alpha)
    \partial_{i} {\bm h}_{L}^{a} \partial_{j} {\bm h}_{L}^{b} \partial_{ij}^{c} \bm h_{L} \cdot E\left[ \delta_{a} \cdot \delta_{b} \cdot \delta_{c} \right]. \label{eq:combined}
\end{align}

The covariance (as per \cref{thm:varefima,thm:varefimb}) can similarly be calculated by changing the variables of the partial derivatives of \( \bm h_{L} \).
Notably, the largest differentiating factor from the original estimator is that the combined estimator is dependent on the third central moment of \( \bm t \). This third central moment is equivalent to the third-order cumulant of \( \bm t \). Thus it can be directly calculated via the derivatives of the log-partition function \( F(\bm h_{L}) \).

\section{Experimental Verification of Bounds}
\label{sup:experimental}

The bounds of the variance of the FIM estimators \( \efima(\bm \theta) \) and \(
\efimb(\bm \theta) \) (as presented in \cref{sec:estimators}) can be
experimentally verified. We train a simple convolutional neural network trained
on the standard MNIST dataset. By leveraging the Jacobian
and Hessian in-built functions in PyTorch, we calculate the bounds in
\cref{thm:varfima11,thm:varfimb11} applied to the variance of the estimators
(as described in \cref{subsec:var_in_closed_form}).

\begin{figure}[th]
    \centering
    \begin{subfigure}[b]{0.95\textwidth}
    \includegraphics[width=\textwidth]{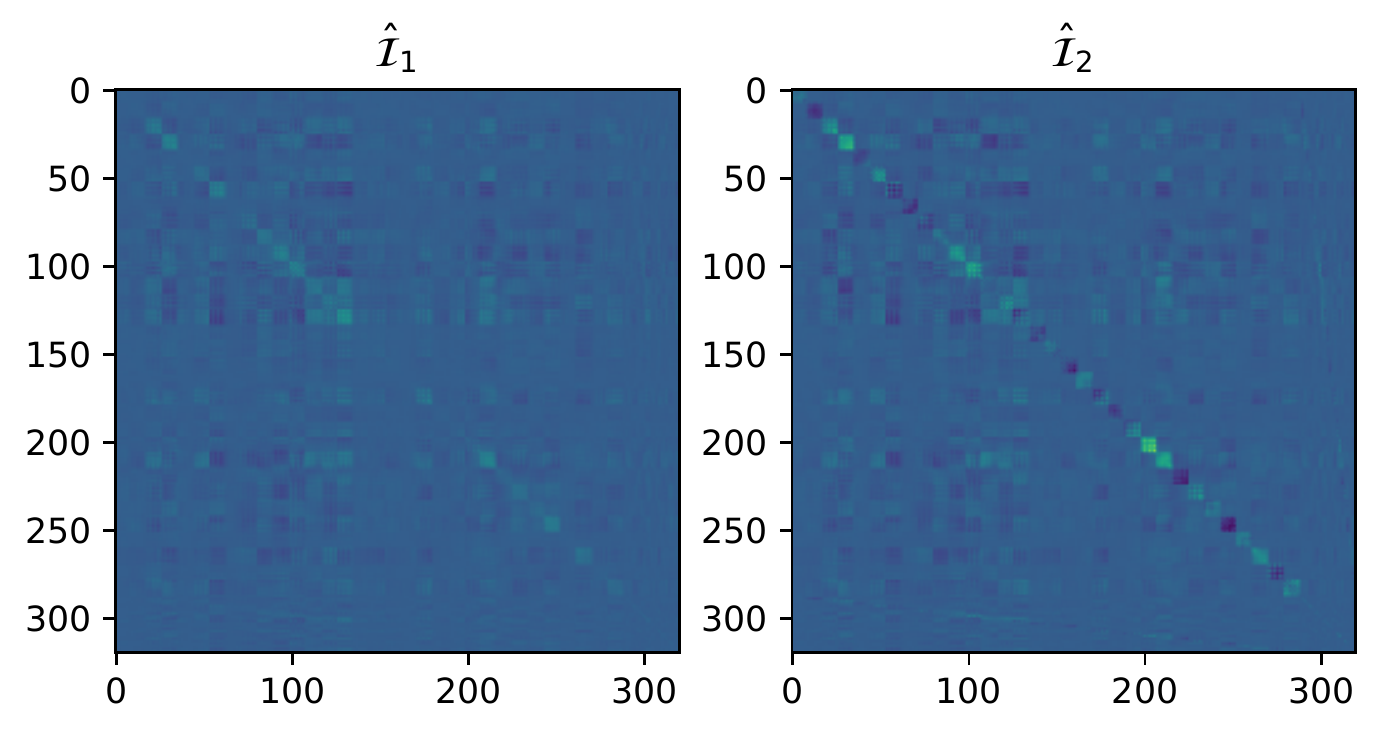}
    \caption{
    \( \efima(\bm \theta) \) and \( \efimb(\bm \theta) \),
      where $\bm\theta$ is a random model.
      Color values are shared.}\label{fig:random_heatmap_dist}
    \end{subfigure}
    \begin{subfigure}[b]{0.95\textwidth}
    \includegraphics[width=\textwidth]{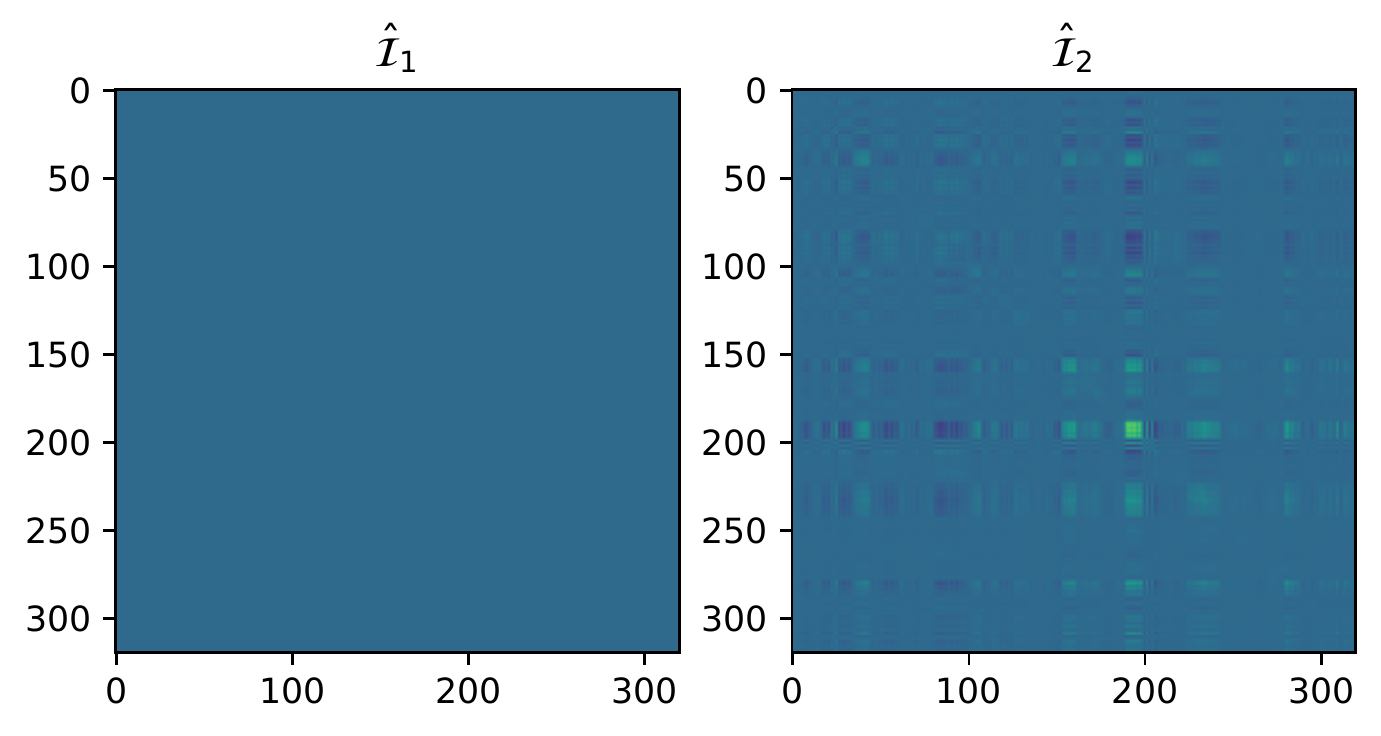}
    \caption{\( \efima(\bm \theta) \) and \( \efimb(\bm \theta) \),
        where $\bm\theta$ is a trained model.
        Color values are shared.}\label{fig:trained_heatmap_dist}
    \end{subfigure}
    \caption{The estimated FIM presented in heatmaps corresponding to the first layer of a CNN .}
\end{figure}

\subsection{Dataset}

We consider the MNIST dataset of \( 28 \times 28 \) pixel grayscale images after normalization.
The training set consists of 60,000 examples and the test set consists of 10,000 examples.

\subsection{Model Setup}

The convolutional neural network considered in our experiments are given by the following layers (in order):

\begin{itemize}
    \renewcommand{\labelitemi}{$\rightarrow$}
    \setlength\itemsep{0em}
    \item Conv(in\_channel=1, out\_channel=32, kernel\_size=(3, 3), stride=(1, 1))
    \item Softplus()
    \item Conv(in\_channel=32, out\_channel=64, kernel\_size=(3, 3), stride=(1, 1))
    \item Softplus()
    \item MaxPool2D()
    \item Dropout(p=0.25)
    \item Linear(in\_features=9216, out\_features=128)
    \item Softplus()
    \item Dropout(p=0.5)
    \item Linear(in\_features=128, out\_features=10)
    \item LogSoftMax()
\end{itemize}

After training, the final model has a 99\% test accuracy.  For most of the
training samples, the predicted probabilities have a low entropy and are close
to a one-hot vector. Consequently, the overall variance of the estimated FIM is
very close to zero.

\subsection{Evaluation}

To compute the FIM, We only consider the weight and bias parameters in the first
layer for simplicity. We randomly choose a fixed \( \bm x \) with multiple
sampled \( \bm y_{i} \) for calculation, as per \cref{eq:efima,eq:efimb}.
Recall that the randomness of our estimators comes from the sampling
of \( \bm y_{i} \sim p(\bm y \mid \bm x) \).
For all related computation, we use double-precision floating point (64 bits).
The FIM estimations for both a trained model and a random model
are given in \cref{fig:trained_heatmap_dist,fig:random_heatmap_dist}.

To calculate the ``true'' variance to compare against the bounds,
we use Monte-Carlo estimation using a large number (1,000 for the presented
results) of samples. Then, the variance of the estimator is approximated by
the sample variance.

\begin{figure}
    \centering
    \begin{subfigure}[b]{0.45\textwidth}
        \centering
        \includegraphics[width=\textwidth]{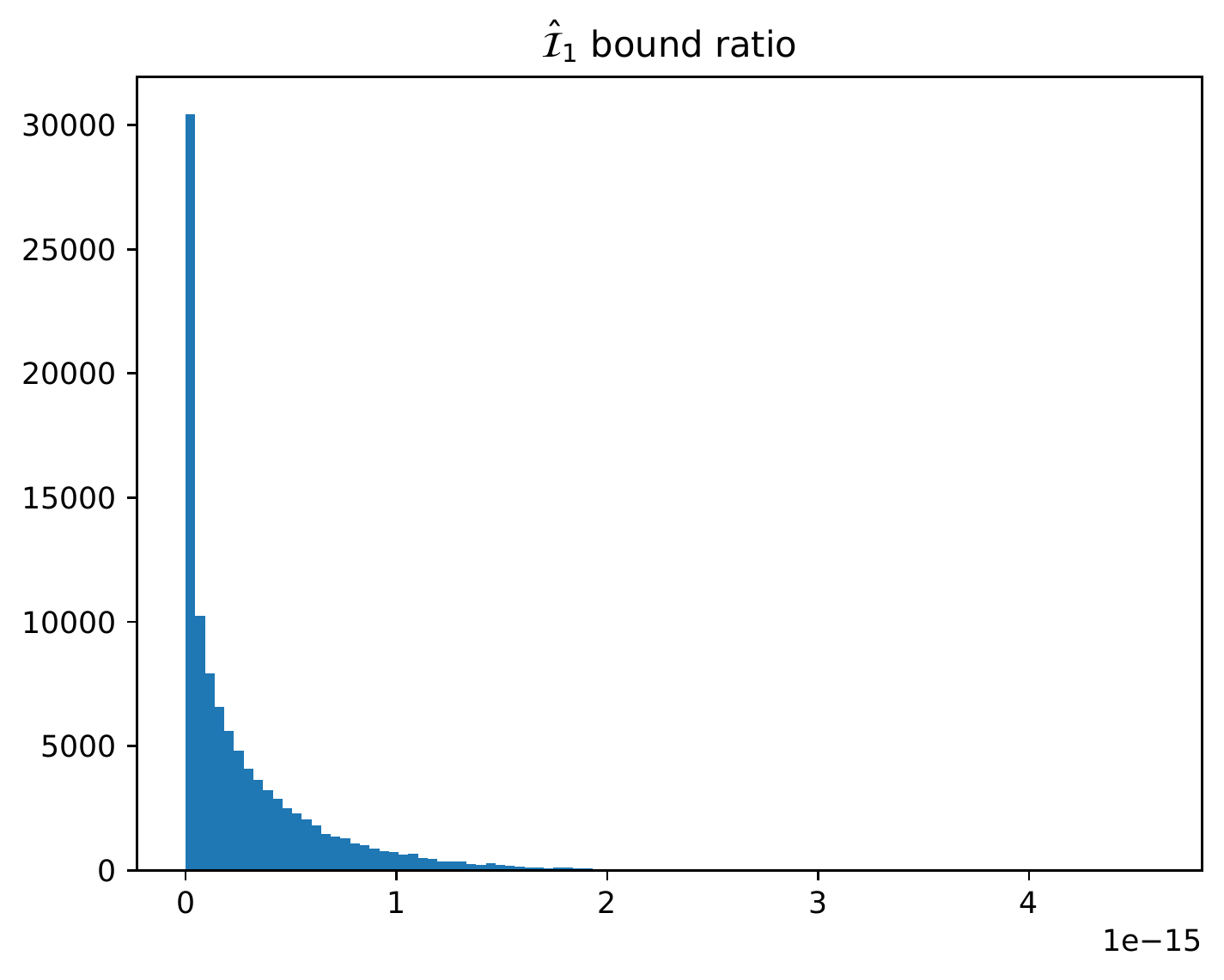}
        \caption{Trained model.}
        \label{fig:trained_efima_ratio}
    \end{subfigure}
    \hfill
    \begin{subfigure}[b]{0.45\textwidth}
        \centering
        \includegraphics[width=\textwidth]{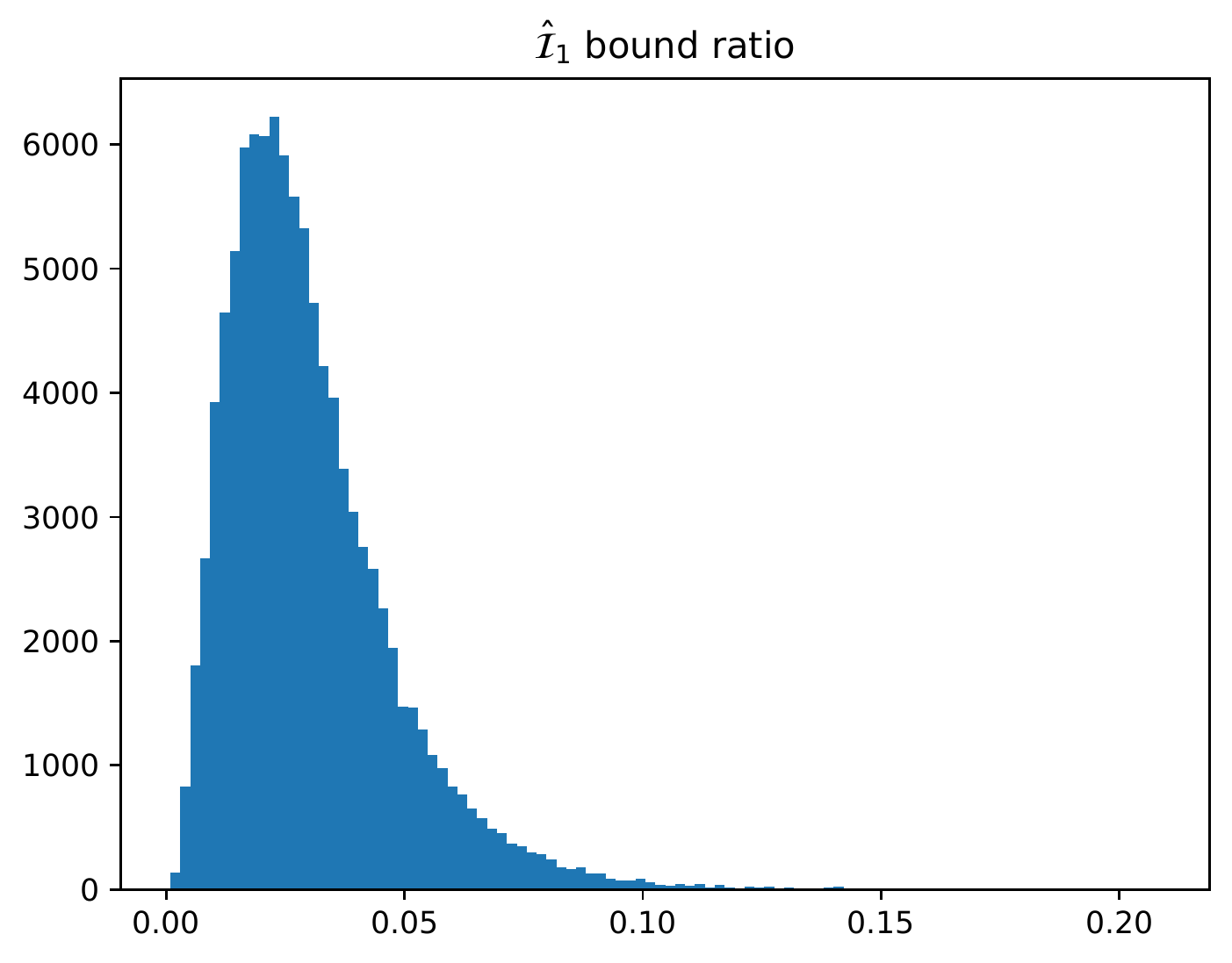}
        \caption{Random model.}
        \label{fig:random_efima_ratio}
    \end{subfigure}
        \caption{Ratio of the bound of the variance in \cref{thm:varfima11} over the true estimated variance.}
        \label{fig:efima_ratio}
\end{figure}

\begin{figure}
    \centering
    \begin{subfigure}[b]{0.45\textwidth}
        \centering
        \includegraphics[width=\textwidth]{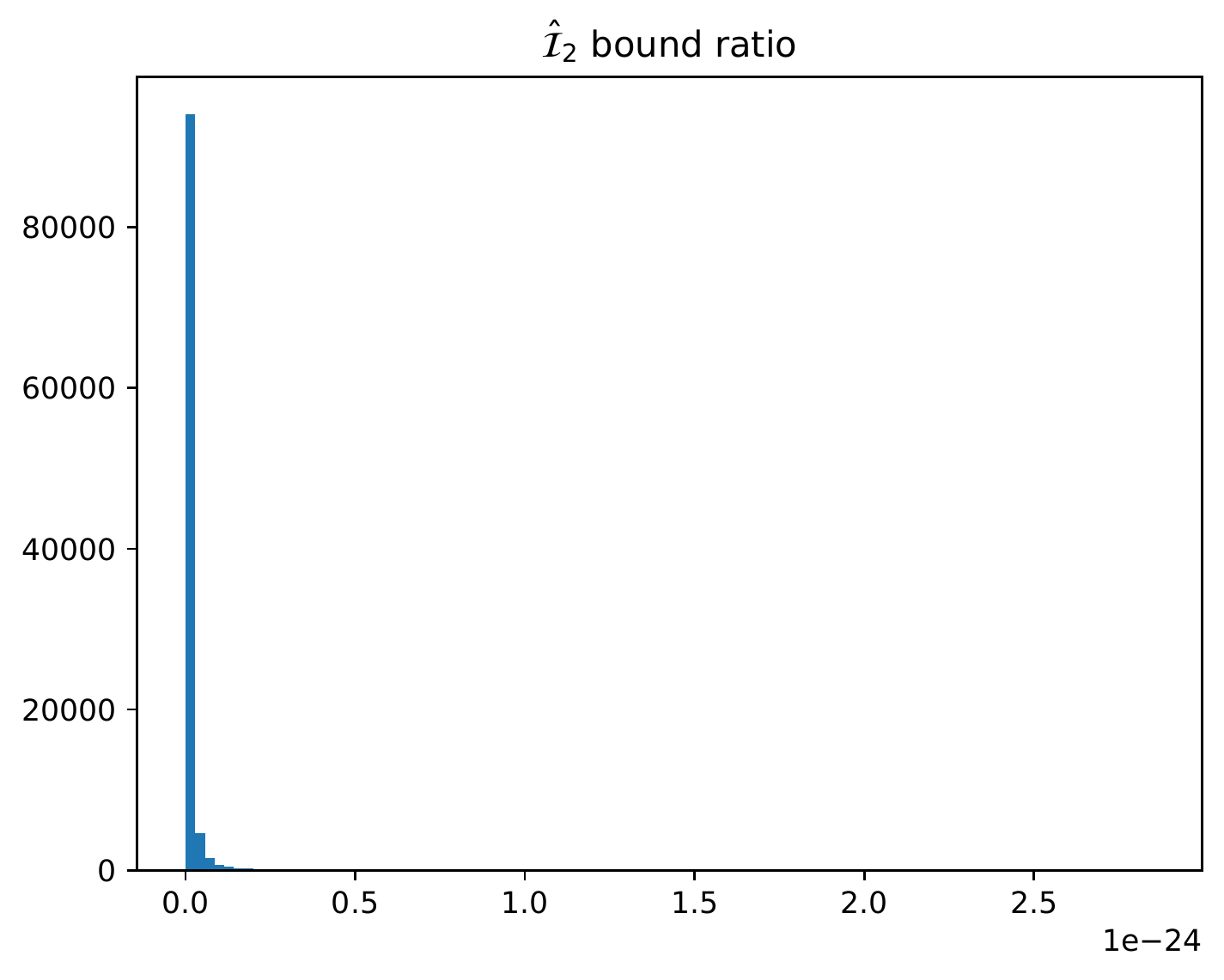}
        \caption{Trained model.}
        \label{fig:trained_efimb_ratio}
    \end{subfigure}
    \hfill
    \begin{subfigure}[b]{0.45\textwidth}
        \centering
        \includegraphics[width=\textwidth]{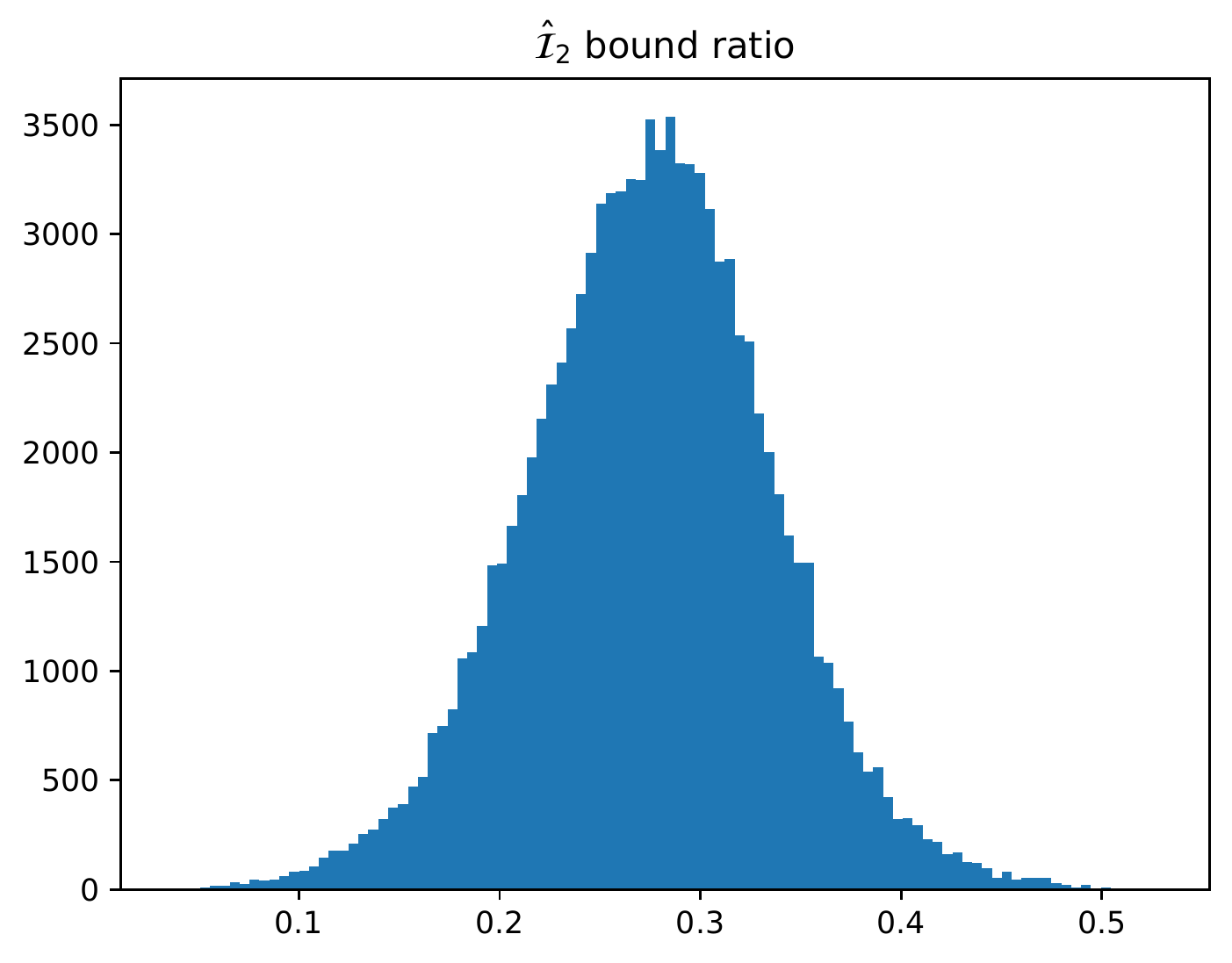}
        \caption{Random model.}
        \label{fig:random_efimb_ratio}
    \end{subfigure}
        \caption{Ratio of the bound of the variance in \cref{thm:varfimb11} over the true estimated variance.}
        \label{fig:efimb_ratio}
\end{figure}

\subsection{Results}

We plot the histograms of the ratio of the estimated variance of
\( \efima(\bm \theta) \) (\( \efimb(\bm \theta) \)) over the
variance bound given by \cref{thm:varfima11} (\cref{thm:varfimb11}).
For the trained network,  the plot is given by \cref{fig:trained_efima_ratio} (\cref{fig:trained_efimb_ratio});
for the random network, the plot is by \cref{fig:random_efima_ratio} (\cref{fig:random_efimb_ratio}).
In both the trained and random models, the bounds are empirically verified
(the ratio is always smaller than 1).
When comparing the ratio histograms, a smaller ratio value corresponds to a
looser bound. We can immediately see that the trained network's bounds are
looser than that of the randomized network.

We also plot of the (Frobenius) distance
\( \Vert \efima(\bm \theta) - \efimb(\bm \theta) \Vert_{F} \)
between the two estimators over the number
of samples for calculating the estimators.
See \cref{fig:trained_line_dist,fig:random_line_dist}
for the cases of trained and random networks, respectively.
As the trained model has a very small variance of $\bm{y}_i$, it is hard to
observe in \cref{fig:trained_line_dist} any change of the distance between \(
\efima(\bm \theta) \) and \( \efimb(\bm \theta) \) as the samples increase. For
the randomized model, we do observe the decrease in estimator distance as the
number of samples increase, as expected, \cref{fig:random_line_dist}.

\begin{figure}
    \centering
    \begin{subfigure}[b]{0.45\textwidth}
        \centering
        \includegraphics[width=\textwidth]{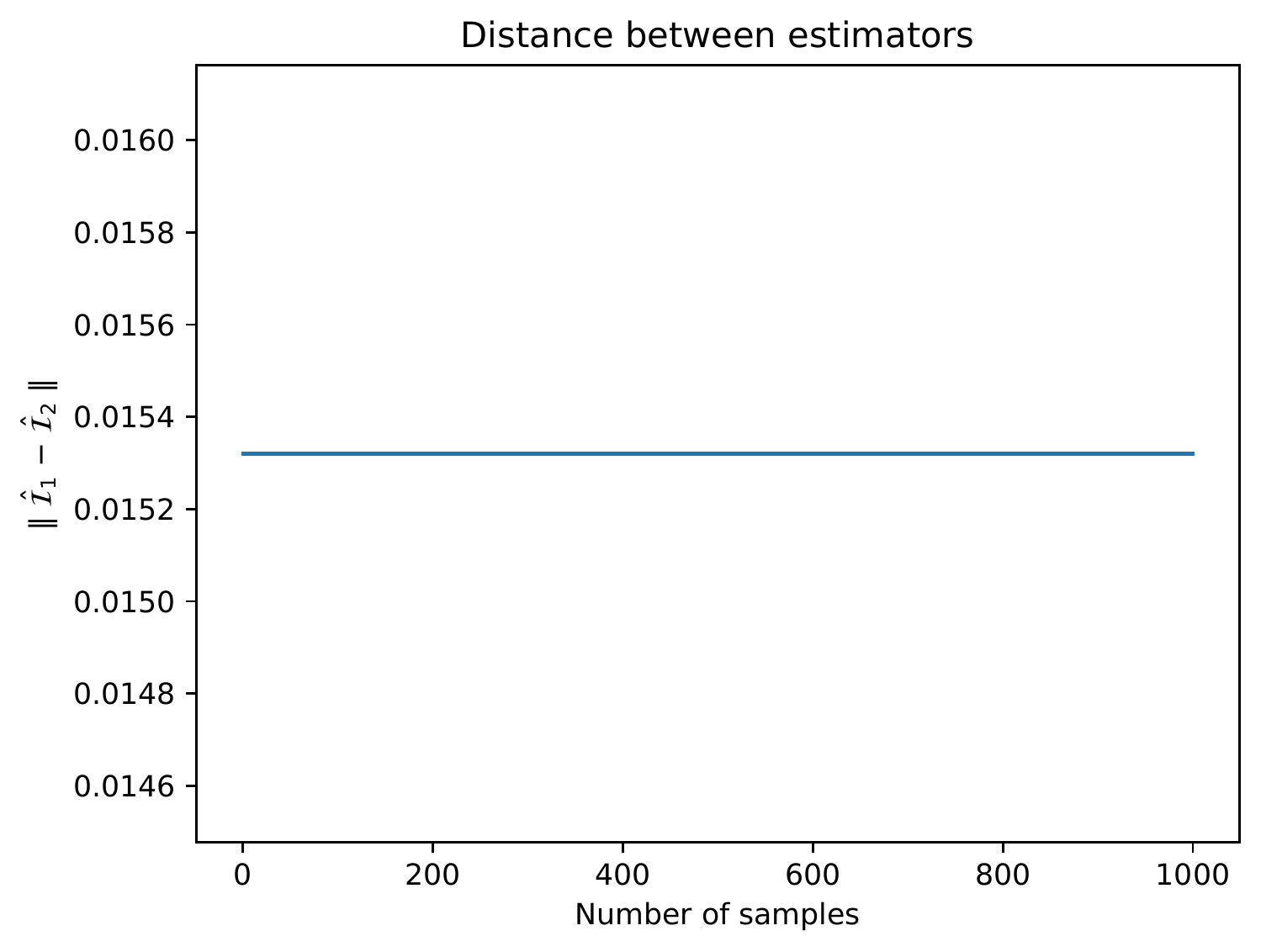}
        \caption{Trained model.}
        \label{fig:trained_line_dist}
    \end{subfigure}
    \hfill
    \begin{subfigure}[b]{0.45\textwidth}
        \centering
        \includegraphics[width=\textwidth]{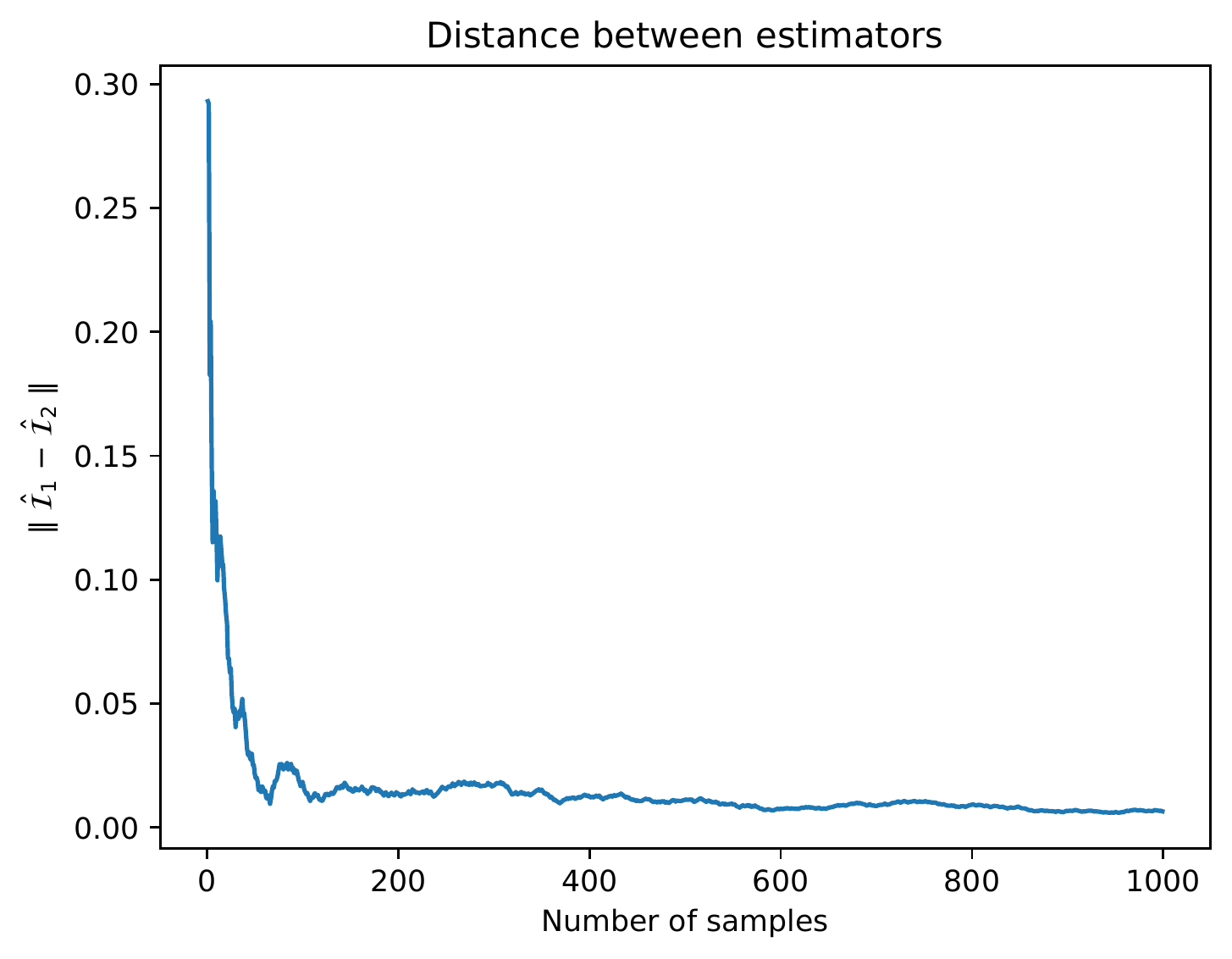}
        \caption{Random model.}
        \label{fig:random_line_dist}
    \end{subfigure}
        \caption{Distance between estimator \( \efima(\bm \theta) \) and \( \efimb(\bm \theta) \).}
        \label{fig:efim_line_dist}
\end{figure}

\section{Univariant Gaussian}

We consider the case where we parameterize a univariant Gaussian distribution and consider the FIM and the corresponding estimators quantities.

Firstly, we specify \cref{eq:exp} for a univariant Gaussian by setting:
\begin{equation*}
    \bm t(y) = (y, y^{2}); \quad F(\bm h) = - \frac{\bm h_{1}^{2}}{4 \bm h_{2}} + \frac{1}{2} \ln\left(\frac{-\pi}{\bm h_{2}}\right),
\end{equation*}
where \( \bm y = y \in \Re \) and \( \bm h = \bm h_{L} \) for readability.

In particular, the 2 dimensional output neural network parametrizes the mean \( \mu \) and standard deviation \( \mathsf{s} \) by:
\begin{equation*}
    \bm h = \left(\frac{\mu}{\mathsf{s}^{2}} , -\frac{1}{2\mathsf{s}^{2}}\right).
\end{equation*}

Furthermore, we have dual coordinates:
\begin{equation*}
    \bm \eta = \left( -\frac{\bm h_{1}}{2 \bm h_{2}}, \frac{\bm h_{1}^{2} - 2\bm h_{2}}{4 \bm h_{2}^{2}} \right) = (\mu, \mu^{2} + \mathsf{s}^{2}).
\end{equation*}

The closed form for the FIM/covariance matrix is given by:
\begin{equation*}
    \fim(\bm h) = \var(\bm t) = \begin{bmatrix}
        -\frac{1}{2\bm h_{2}} & \frac{\bm h_{1}}{2 \bm h_{2}^{2}} \\
        \frac{\bm h_{1}}{2 \bm h_{2}^{2}} & -\frac{\bm h_{1}^2}{2 \bm h_{2}^3} + \frac{1}{2\bm h_{2}^2}
    \end{bmatrix}.
\end{equation*}

Notably, we have that \( \bm h_{1} \in \Re \) and \( \bm h_{2} \in (-\infty, 0) \).

We present the following element-wise plots of FIM in \cref{fig:normal_fim}.
\begin{figure}[h]
    \centering
    \includegraphics[width=\textwidth]{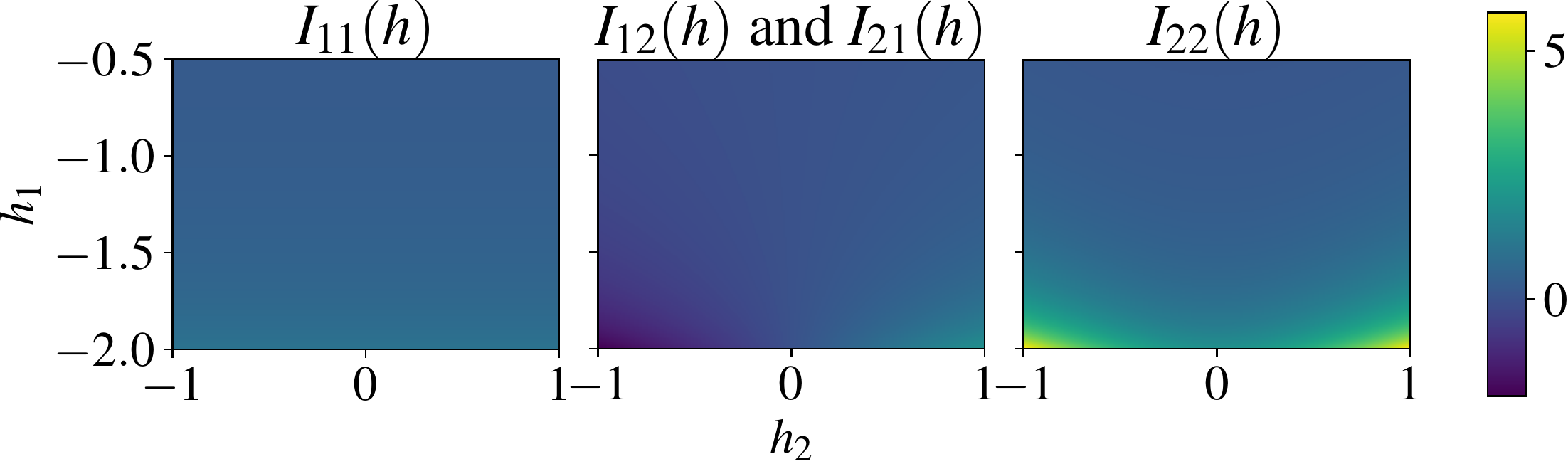}
    \caption{Normal distribution \( \fim(\bm h) \).}
    \label{fig:normal_fim}
\end{figure}

\section{Residual Simple Cases}

We look into the residual term which is present in the proof of \cref{thm:fim2}. Specifically, the quantity:
\begin{equation*}
    \residual(\bm \theta) =
    \int p(\bm x) \frac{\partial^{2}}{\partial \bm\theta \partial \bm\theta^{\T}} p(\bm y \mid \bm x) \dy \dx.
\end{equation*}

Consider the simple case where we only have a single weight and neuron,
\begin{align*}
p(y \mid x) &= \exp\left( t(y)h - F(h) \right),\\
h &= \sigma( wx ).
\end{align*}

First consider the first derivative:
\begin{align*}
    \frac{\partial}{\partial w} p(y \mid x)
    &= p(y \mid x) \cdot \frac{\partial}{\partial w} \left( t(y)h - F(h) \right) \\
    &= p(y \mid x) \cdot \left( t(y) - \eta(h) \right) \cdot \frac{\partial h}{\partial w} \\
    &= p(y \mid x) \cdot \left( t(y) - \eta(h) \right) \cdot \sigma^{\prime}( wx ) \cdot x \\
\end{align*}

\subsection{Assume that activation is identity}

As \( \sigma^{\prime}(z) = 1 \), we have
the second derivative is the following:
\begin{align*}
    \frac{\partial^{2}}{\partial^{2} w} p(y \mid x)
    &= \frac{\partial}{\partial w} \left( p(y \mid x) \cdot \left( t(y) - \eta(h) \right) \cdot x \right) \\
    &= p(y \mid x) \left[  \left( t(y) - \eta(h) \right)^{2} \cdot x^{2} - \grad \eta(h) \cdot x^{2} \right]
\end{align*}

Integrating the first term for the residual we have
\begin{align*}
    \int \int p(x) p(y \mid x) \cdot \left( t(y) - \eta(h) \right)^{2} \cdot x^{2} \dmeas{y} \dmeas{x}
    =\int p(x) \cdot \grad \eta(h) \cdot x^{2} \dmeas{x}.
\end{align*}

Integrating the second term for the residual we have
\begin{align*}
    \int \int p(x) p(y \mid x) \cdot \grad \eta(h) \cdot x^{2} \dmeas{y} \dmeas{x}
    &= \int p(x) \cdot \grad \eta(h) \cdot x^{2} \dmeas{x}.
\end{align*}

Thus by taking the difference, the residual is zero.

\subsection{Assume that activation is not identity}

The second derivative is the following:
\begin{align*}
    \frac{\partial^{2}}{\partial^{2} w} p(y \mid x)
    &= \frac{\partial}{\partial w} \left( p(y \mid x) \cdot \left( t(y) - \eta(h) \right) \cdot \sigma^{\prime}( wx ) \cdot x \right) \\
    &= \frac{\partial}{\partial w} \left( p(y \mid x) \right) \cdot \left( t(y) - \eta(h) \right) \cdot \sigma^{\prime}( wx ) \cdot x \\
    &\quad\quad + p(y \mid x) \cdot \frac{\partial}{\partial w}\left( t(y) - \eta(h) \right) \cdot \sigma^{\prime}( wx ) \cdot x \\
    &\quad\quad + p(y \mid x) \cdot \left( t(y) - \eta(h) \right) \cdot \frac{\partial}{\partial w} \left( \sigma^{\prime}( wx ) \right) \cdot x \\
    &= p(y \mid x) \cdot \left( t(y) - \eta(h) \right)^{2} \cdot \sigma^{\prime}( wx )^{2} \cdot x^{2} \\
    &\quad\quad + p(y \mid x) \cdot \frac{\partial}{\partial w}\left( t(y) - \eta(h) \right) \cdot \sigma^{\prime}( wx ) \cdot x \\
    &\quad\quad + p(y \mid x) \cdot \left( t(y) - \eta(h) \right) \cdot \frac{\partial}{\partial w} \left( \sigma^{\prime}( wx ) \right) \cdot x \\
    &= p(y \mid x) \cdot \left( t(y) - \eta(h) \right)^{2} \cdot \sigma^{\prime}( wx )^{2} \cdot x^{2} \\
    &\quad\quad - p(y \mid x) \cdot \grad \eta(h) \cdot \sigma^{\prime}( wx )^{2} \cdot x^{2} \\
    &\quad\quad + p(y \mid x) \cdot \left( t(y) - \eta(h) \right) \cdot \frac{\partial}{\partial w} \left( \sigma^{\prime}( wx ) \right) \cdot x.
\end{align*}

By the linearity of the integral, we calculate each term of the residual. For the first term:
\begin{align*}
    &\int \int p(x) \cdot p(y \mid x) \cdot \left( t(y) - \eta(h) \right)^{2} \cdot \sigma^{\prime}( wx )^{2} \cdot x^{2} \dmeas{y} \dmeas{x} \\
    &= \int p(x) \cdot \sigma^{\prime}( wx )^{2} \cdot x^{2} \int p(y \mid x) \cdot \left( t(y) - \eta(h) \right)^{2} \dmeas{y} \dmeas{x} \\
    &= \int p(x) \cdot \sigma^{\prime}( wx )^{2} \cdot x^{2} \cdot \grad\eta(h) \dmeas{x}.
\end{align*}

Given that the second term only has \( p(y \mid x) \) which is dependent on \( y \), the first and second term of the residual cancel out. Thus we only have:
\begin{align*}
    \residual(w)
    &= \int \int p(x) p(y \mid x) \cdot \left( t(y) - \eta(h) \right) \cdot \frac{\partial}{\partial w} \left( \sigma^{\prime}( wx ) \right) \cdot x \dmeas{y} \dmeas{x} \\
    &= \int \int p(x) p(y \mid x) \cdot t(y) \cdot \frac{\partial}{\partial w} \left( \sigma^{\prime}( wx ) \right) \cdot x \dmeas{y} \dmeas{x} \\
    &\quad\quad - \int \int p(x) p(y \mid x) \cdot \eta(h) \cdot \frac{\partial}{\partial w} \left( \sigma^{\prime}( wx ) \right) \cdot x \dmeas{y} \dmeas{x} \\
    &= \int p(x) \cdot \frac{\partial}{\partial w} \left( \sigma^{\prime}( wx ) \right) \cdot x \int p(y \mid x) \cdot t(y) \dmeas{y} \dmeas{x} - \int p(x) \cdot \eta(h) \cdot \frac{\partial}{\partial w} \left( \sigma^{\prime}( wx ) \right) \cdot x \dmeas{x}.
\end{align*}
Given that
\begin{equation*}
    \frac{\partial}{\partial w} \left( \sigma^{\prime}( wx ) \right)
    = \delta(wx)
\end{equation*}
and
\begin{equation}
    \int f(x) \delta(x) \dmeas{x} = f(0),
\end{equation}
the residual will result in zero for this case.


\begin{thebibliography}{32}
\providecommand{\natexlab}[1]{#1}
\providecommand{\url}[1]{\texttt{#1}}
\expandafter\ifx\csname urlstyle\endcsname\relax
  \providecommand{\doi}[1]{doi: #1}\else
  \providecommand{\doi}{doi: \begingroup \urlstyle{rm}\Url}\fi

\bibitem[Amari(2016)]{aIGI}
S.~Amari.
\newblock \emph{Information Geometry and Its Applications}, volume 194 of
  \emph{Applied Mathematical Sciences}.
\newblock Springer-Verlag, Berlin, 2016.

\bibitem[Amari et~al.(2019)Amari, Karakida, and Oizumi]{amari2019fisher}
S.~Amari, R.~Karakida, and M.~Oizumi.
\newblock {F}isher information and natural gradient learning in random deep
  networks.
\newblock In \emph{International Conference on Artificial Intelligence and
  Statistics}, pages 694--702. PMLR, 2019.

\bibitem[Ay(2020)]{ay}
N.~Ay.
\newblock On the locality of the natural gradient for learning in deep
  {B}ayesian networks.
\newblock \emph{Information Geometry}, pages 1--49, 2020.

\bibitem[Ba et~al.(2016)Ba, Kiros, and Hinton]{ba2016layer}
J.~L. Ba, J.~R. Kiros, and G.~E. Hinton.
\newblock Layer normalization.
\newblock \emph{arXiv preprint arXiv:1607.06450}, 2016.

\bibitem[Bowman and Shenton(1988)]{gamma}
K.~O. Bowman and L.~R. Shenton.
\newblock \emph{Properties of Estimators for the Gamma Distribution}.
\newblock Marcel Dekker, New York, 1988.

\bibitem[Chen and Li(2020)]{chen2020tensor}
B.~Chen and Z.~Li.
\newblock On the tensor spectral p-norm and its dual norm via partitions.
\newblock \emph{Computational Optimization and Applications}, 75\penalty0
  (3):\penalty0 609--628, 2020.

\bibitem[Chen and Li(2009)]{InfiniteFIMMixture-2009}
J.~Chen and P.~Li.
\newblock Hypothesis test for normal mixture models: The {EM} approach.
\newblock \emph{The Annals of Statistics}, 37\penalty0 (5A):\penalty0
  2523--2542, 2009.

\bibitem[Chen(2007)]{chen2007new}
X.~Chen.
\newblock A new generalization of {C}hebyshev inequality for random vectors.
\newblock \emph{arXiv preprint arXiv:0707.0805}, 2007.

\bibitem[Delattre and Kuhn(2019)]{delattre2019estimating}
M.~Delattre and E.~Kuhn.
\newblock Estimating {F}isher information matrix in latent variable models
  based on the score function.
\newblock In \emph{European Meeting of Statisticians (EMS)}, 2019.

\bibitem[Efron and Hinkley(1978)]{efron}
B.~Efron and D.~V. Hinkley.
\newblock Assessing the accuracy of the maximum likelihood estimator: Observed
  versus expected {F}isher information.
\newblock \emph{Biometrika}, 65\penalty0 (3):\penalty0 457--482, 1978.

\bibitem[Guo and Spall(2019)]{spall}
S.~Guo and J.~C. Spall.
\newblock Relative accuracy of two methods for approximating observed {F}isher
  information.
\newblock In \emph{Data-Driven Modeling, Filtering and Control: Methods and
  applications}, pages 189--211. IET Press, London, 2019.

\bibitem[Hinton et~al.(2015)Hinton, Vinyals, and Dean]{hinton2015distilling}
G.~Hinton, O.~Vinyals, and J.~Dean.
\newblock Distilling the knowledge in a neural network.
\newblock \emph{arXiv preprint arXiv:1503.02531}, 2015.

\bibitem[Karakida et~al.(2019)Karakida, Akaho, and
  Amari]{karakida2019universal}
R.~Karakida, S.~Akaho, and S.~Amari.
\newblock Universal statistics of {F}isher information in deep neural networks:
  Mean field approach.
\newblock In \emph{International Conference on Artificial Intelligence and
  Statistics}, pages 1032--1041. PMLR, 2019.

\bibitem[Kingma and Ba(2015)]{adam}
D.~P. Kingma and J.~Ba.
\newblock Adam: {A} method for stochastic optimization.
\newblock In \emph{International Conference on Learning Representations}, 2015.

\bibitem[Kunstner et~al.(2020)Kunstner, Balles, and
  Hennig]{kunstner2020limitations}
F.~Kunstner, L.~Balles, and P.~Hennig.
\newblock Limitations of the empirical {F}isher approximation for natural
  gradient descent.
\newblock In \emph{Advances in Neural Information Processing Systems}, pages
  4133--4144. Curran Associates, Inc., 2020.

\bibitem[Lehmann and Casella(1998)]{thpe}
E.~L. Lehmann and G.~Casella.
\newblock \emph{Theory of Point Estimation}.
\newblock Springer-Verlag New York, second edition, 1998.

\bibitem[Lim(2005)]{lim2005singular}
L.-H. Lim.
\newblock Singular values and eigenvalues of tensors: a variational approach.
\newblock In \emph{1st IEEE International Workshop on Computational Advances in
  Multi-Sensor Adaptive Processing, 2005.}, pages 129--132. IEEE, 2005.

\bibitem[Martens(2020)]{martens}
J.~Martens.
\newblock New insights and perspectives on the natural gradient method.
\newblock \emph{Journal of Machine Learning Research}, 21\penalty0
  (146):\penalty0 1--76, 2020.

\bibitem[Martens and Grosse(2015)]{martens2015optimizing}
J.~Martens and R.~Grosse.
\newblock Optimizing neural networks with {K}ronecker-factored approximate
  curvature.
\newblock In \emph{International Conference on Machine Learning}, pages
  2408--2417. PMLR, 2015.

\bibitem[McCullagh(2018)]{mccullagh2018tensor}
P.~McCullagh.
\newblock \emph{Tensor methods in statistics}.
\newblock Courier Dover Publications, 2018.

\bibitem[Nielsen(2017)]{alphaFIM}
F.~Nielsen.
\newblock The $\alpha$-representations of the {F}isher information matrix,
  2017.
\newblock URL \url{https://franknielsen.github.io/blog/alpha-FIM/index.html}.

\bibitem[Nielsen(2020)]{elementig}
F.~Nielsen.
\newblock An elementary introduction to information geometry.
\newblock \emph{Entropy}, 22\penalty0 (10), 2020.

\bibitem[Ollivier(2015)]{ollivier}
Y.~Ollivier.
\newblock Riemannian metrics for neural networks {I}: feedforward networks.
\newblock \emph{Information and Inference: A Journal of the IMA}, 4\penalty0
  (2):\penalty0 108--153, 2015.

\bibitem[Park et~al.(2000)Park, Amari, and Fukumizu]{park}
H.~Park, S.~Amari, and K.~Fukumizu.
\newblock Adaptive natural gradient learning algorithms for various stochastic
  models.
\newblock \emph{Neural Networks}, 13\penalty0 (7):\penalty0 755--764, 2000.

\bibitem[Pascanu and Bengio(2014)]{pbRNG}
R.~Pascanu and Y.~Bengio.
\newblock Revisiting natural gradient for deep networks.
\newblock In \emph{International Conference on Learning Representations}, 2014.

\bibitem[Paszke et~al.(2019)Paszke, Gross, Massa, Lerer, Bradbury, Chanan,
  Killeen, Lin, Gimelshein, Antiga, Desmaison, Kopf, Yang, DeVito, Raison,
  Tejani, Chilamkurthy, Steiner, Fang, Bai, and Chintala]{torch}
A.~Paszke, S.~Gross, F.~Massa, A.~Lerer, J.~Bradbury, G.~Chanan, T.~Killeen,
  Z.~Lin, N.~Gimelshein, L.~Antiga, A.~Desmaison, A.~Kopf, E.~Yang, Z.~DeVito,
  M.~Raison, A.~Tejani, S.~Chilamkurthy, B.~Steiner, L.~Fang, J.~Bai, and
  S.~Chintala.
\newblock Pytorch: An imperative style, high-performance deep learning library.
\newblock In \emph{Advances in Neural Information Processing Systems}, pages
  8024--8035. Curran Associates, Inc., 2019.

\bibitem[Pennington and Worah(2018)]{pennington2018spectrum}
J.~Pennington and P.~Worah.
\newblock The spectrum of the {F}isher information matrix of a
  single-hidden-layer neural network.
\newblock In \emph{Advances in Neural Information Processing Systems}, pages
  5415--5424, 2018.

\bibitem[Salimans and Kingma(2016)]{weightnorm}
T.~Salimans and D.~P. Kingma.
\newblock Weight normalization: A simple reparameterization to accelerate
  training of deep neural networks.
\newblock In \emph{Advances in Neural Information Processing Systems},
  volume~29. Curran Associates, Inc., 2016.

\bibitem[Schraudolph(2002)]{schraudolph2002fast}
N.~N. Schraudolph.
\newblock Fast curvature matrix-vector products for second-order gradient
  descent.
\newblock \emph{Neural computation}, 14\penalty0 (7):\penalty0 1723--1738,
  2002.

\bibitem[Soen and Sun(2021)]{varfim}
A.~Soen and K.~Sun.
\newblock On the variance of the {F}isher information for deep learning.
\newblock In \emph{Advances in Neural Information Processing Systems}, 2021.

\bibitem[Sun(2020)]{pull}
K.~Sun.
\newblock Information geometry for data geometry through pullbacks.
\newblock In \emph{Deep Learning through Information Geometry (Workshop at
  NeurIPS 2020)}, 2020.

\bibitem[Sun and Nielsen(2017)]{rfim}
K.~Sun and F.~Nielsen.
\newblock Relative {F}isher information and natural gradient for learning large
  modular models.
\newblock In \emph{International Conference on Machine Learning}, pages
  3289--3298. PMLR, 2017.

\end{thebibliography}
\end{document}